\documentclass{article} 
\usepackage{iclr2027_conference_added_preprint_flag,times}

\usepackage{amsmath,amsfonts,bm}

\def\eqref#1{equation~\ref{#1}}

\def\ceil#1{\lceil #1 \rceil}
\def\floor#1{\lfloor #1 \rfloor}
\def\1{\bm{1}}

\def\eps{{\epsilon}}

\DeclareMathAlphabet{\mathsfit}{\encodingdefault}{\sfdefault}{m}{sl}
\SetMathAlphabet{\mathsfit}{bold}{\encodingdefault}{\sfdefault}{bx}{n}

\newcommand{\E}{\mathbb{E}}

\newcommand{\R}{\mathbb{R}}

\newcommand{\Var}{\mathrm{Var}}

\DeclareMathOperator*{\argmax}{arg\,max}

\newcommand{\dd}{\mathrm{d}}
\DeclareMathOperator{\Unif}{Unif}

\DeclareMathOperator{\Bin}{Bin}

\def\safedef#1{%
   \ifx#1\undefined
      \expandafter\def\expandafter#1%
   \else
      \errmessage{The \string#1 is defined already}%
      \expandafter\def\expandafter\tmp
   \fi
}

\usepackage{amsthm,amsmath,bbm,amsfonts,amssymb}
\usepackage{dsfont} 
\usepackage{mathtools}
\usepackage{commath}
\mathtoolsset{showonlyrefs=false}
\allowdisplaybreaks 
\usepackage{placeins}

\usepackage{xspace}
\usepackage[T1]{fontenc}
\usepackage[utf8]{inputenc}
\usepackage[dvipsnames]
{xcolor} 
\usepackage{xurl}
\usepackage{microtype}
\usepackage{hyperref,url}
\usepackage{cleveref} 
\definecolor{darkblue}{RGB}{0,75,130}
\definecolor{darkred}{RGB}{150,50,0}
\definecolor{darkgreen}{RGB}{0,95,70}
\hypersetup{
    colorlinks = true,
    citecolor  = darkblue,
    linkcolor  = darkred,
    filecolor  = darkblue,
    urlcolor   = darkgreen,
}

\usepackage{wrapfig}
\usepackage{pbox} 
\usepackage{algorithm,algorithmic}

\usepackage{enumitem}
\usepackage[normalem]{ulem}

\newcommand{\seiyun}[1]{{\color[rgb]{0.10,0.30,0.80} \textbf{SY:} #1}}

\newcommand{\kj}[1]{{\color{RedOrange}[#1]}}

{\color{kjsaved}}%

{\color{kjsavedenvkj}}%

\newenvironment{revised}
{\colorlet{kjsavedrevised}{.}\color{MidnightBlue}}%
{\color{kjsavedrevised}}%
\usepackage{framed}
\usepackage[most]{tcolorbox}
\definecolor{kjgray}{rgb}{.7,.7,.7}

\newtheoremstyle{kjstyle}
{1ex} 
{\topsep} 
{\itshape} 
{} 
{\bfseries} 
{.} 
{.5em} 
{} 

\newtheoremstyle{kjstyle2}
{.0em} 
{.0em} 
{\itshape} 
{} 
{\bfseries} 
{.} 
{.5em} 
{} 

\newtheoremstyle{kjstylenoitalic}
{1ex} 
{\topsep} 
{} 
{} 
{\bfseries} 
{.} 
{.5em} 
{} 

\tcbset{kjboxstyle/.style={title={},breakable,colback=white,enhanced jigsaw,boxrule=1.3pt,sharp corners,colframe=kjgray,boxsep=0pt,coltitle={black},attach title to upper={},left=.8ex,bottom=.4em}}    

\tcbset{kjboxstylec/.style={title={},breakable,colback=white,enhanced jigsaw,boxrule=1.3pt,sharp corners,colframe=kjgray,boxsep=0pt,coltitle={black},attach title to upper={},before skip=1.5ex,left=.8ex,bottom=.4em,top=0.4em,enlarge top by=-0.3em,enlarge bottom by=-0.3em}}    

\newtheorem{theorem}{Theorem}
\newtheorem{lemma}{Lemma}
\newtheorem{corollary}{Corollary}
\newtheorem{proposition}{Proposition}
\newtheorem{remark}{Remark}

\usepackage{mdframed}
\usepackage{lipsum}
\definecolor{kjgray}{rgb}{.7,.7,.7}
\makeatletter



\makeatletter
\renewcommand{\paragraph}{%
  \@startsection{paragraph}{4}%
  {\z@}{0.50ex \@plus 1ex \@minus .2ex}{-1em}%
  {\normalfont\normalsize\bfseries}%
}
\makeatother

\usepackage{tabularx} 
\newcolumntype{P}[1]{>{\centering\arraybackslash}p{#1}}
\newcolumntype{M}[1]{>{\centering\arraybackslash}m{#1}}

\def\ddefloop#1{\ifx\ddefloop#1\else\ddef{#1}\expandafter\ddefloop\fi}

\def\ddef#1{\expandafter\def\csname #1#1\endcsname{\ensuremath{\mathbb{#1}}}}
\ddefloop ABCDFGHIJKLMNORSTUWXYZ\ddefloop 

\def\ddef#1{\expandafter\def\csname c#1\endcsname{\ensuremath{\mathcal{#1}}}}
\ddefloop ABCDEFGHIJKLMNOPQRSTUVWXYZ\ddefloop

\def\ddef#1{\expandafter\def\csname b#1\endcsname{\ensuremath{{\mathbf{#1}}}}}
\ddefloop ABCDEFGHIJKLMNOPQRSTUVWXYZ\ddefloop  
\def\ddef#1{\expandafter\def\csname b#1\endcsname{\ensuremath{{\boldsymbol{#1}}}}}
\ddefloop abcdeghijklmnopqrtsuvwxyz\ddefloop  

\def\ddef#1{\expandafter\def\csname h#1\endcsname{\ensuremath{\hat{#1}}}}
\ddefloop ABCDEFGHIJKLMNOPQRSTUVWXYZabcdefghijklmnopqrsuvwxyz\ddefloop 
\def\ddef#1{\expandafter\def\csname hc#1\endcsname{\ensuremath{\hat{\mathcal{#1}}}}}
\ddefloop ABCDEFGHIJKLMNOPQRSTUVWXYZ\ddefloop
\def\ddef#1{\expandafter\def\csname hb#1\endcsname{\ensuremath{\hat{\mathbf{#1}}}}}
\ddefloop ABCDEFGHIJKLMNOPQRSTUVWXYZ\ddefloop %
\def\ddef#1{\expandafter\def\csname hb#1\endcsname{\ensuremath{\hat{\boldsymbol{#1}}}}}
\ddefloop abcdefghijklmnopqrstuvwxyz\ddefloop %

\def\ddef#1{\expandafter\def\csname t#1\endcsname{\ensuremath{\tilde{#1}}}}
\ddefloop ABCDEFGHIJKLMNOPQRSTUVWXYZabcdefgijklmnpqtsuvwxyz\ddefloop 
\def\ddef#1{\expandafter\def\csname tc#1\endcsname{\ensuremath{\tilde{\mathcal{#1}}}}}
\ddefloop ABCDEFGHIJKLMNOPQRSTUVWXYZ\ddefloop
\def\ddef#1{\expandafter\def\csname tb#1\endcsname{\ensuremath{\tilde{\mathbf{#1}}}}}
\ddefloop ABCDEFGHIJKLMNOPQRSTUVWXYZ\ddefloop
\def\ddef#1{\expandafter\def\csname tb#1\endcsname{\ensuremath{\tilde{\boldsymbol{#1}}}}}
\ddefloop abcdefghijklmnopqrstuvwxyz\ddefloop %

\def\ddef#1{\expandafter\def\csname bar#1\endcsname{\ensuremath{\bar{#1}}}}
\ddefloop ABCDEFGHIJKLMNOPQRSTUVWXYZabcdefghijklmnopqrtsuvwxyz\ddefloop
\def\ddef#1{\expandafter\def\csname barc#1\endcsname{\ensuremath{\bar{\mathcal{#1}}}}}
\ddefloop ABCDEFGHIJKLMNOPQRSTUVWXYZ\ddefloop
\def\ddef#1{\expandafter\def\csname barb#1\endcsname{\ensuremath{\bar{\mathbf{#1}}}}}
\ddefloop ABCDEFGHIJKLMNOPQRSTUVWXYZ\ddefloop
\def\ddef#1{\expandafter\def\csname barb#1\endcsname{\ensuremath{\bar{\boldsymbol{#1}}}}}
\ddefloop abcdefghijklmnopqrstuvwxyz\ddefloop %

\def\ddef#1{\expandafter\def\csname war#1\endcsname{\ensuremath{\overline{#1}}}}
\ddefloop ABCDEFGHIJKLMNOPQRSTUVWXYZabcdefghijklmnopqrtsuvwxyz\ddefloop
\def\ddef#1{\expandafter\def\csname warc#1\endcsname{\ensuremath{\overline{\mathcal{#1}}}}}
\ddefloop ABCDEFGHIJKLMNOPQRSTUVWXYZ\ddefloop
\def\ddef#1{\expandafter\def\csname warb#1\endcsname{\ensuremath{\overline{\mathbf{#1}}}}}
\ddefloop ABCDEFGHIJKLMNOPQRSTUVWXYZ\ddefloop
\def\ddef#1{\expandafter\def\csname warb#1\endcsname{\ensuremath{\overline{\boldsymbol{#1}}}}}
\ddefloop abcdefghijklmnopqrstuvwxyz\ddefloop %

\def\eps{\varepsilon}
\def\epsilon{\varepsilon}

\usepackage{pgffor}
\def\greeksymbols{alpha,beta,gamma,gam,delta,dt,eps,epsilon,zeta,eta,theta,th,iota,kappa,kap,lambda,lam,mu,nu,xi,pi,rho,sigma,sig,tau,phi,chi,psi,omega,om,Gamma,Gam,Delta,Dt,Theta,Th,Lambda,Lam,Pi,Sigma,Sig,Phi,Psi,Omega,Om}
\def\greeksymbolsnoeta{alpha,beta,gamma,gam,delta,dt,eps,epsilon,zeta,theta,th,iota,kappa,kap,lambda,lam,mu,nu,xi,pi,rho,sigma,sig,tau,phi,chi,psi,omega,om,Gamma,Gam,Delta,Dt,Theta,Th,Lambda,Lam,Pi,Sigma,Sig,Phi,Psi,Omega,Om} 

\foreach \x in \greeksymbolsnoeta{\expandafter\xdef\csname b\x\endcsname{\noexpand\ensuremath{\noexpand\boldsymbol{\csname \x\endcsname}}}}

\foreach \x in \greeksymbols{\expandafter\xdef\csname h\x\endcsname{\noexpand\ensuremath{\noexpand\hat{\csname \x\endcsname}}}}
\foreach \x in \greeksymbolsnoeta{\expandafter\xdef\csname hb\x\endcsname{\noexpand\ensuremath{\noexpand\hat{\noexpand\boldsymbol{ \csname \x\endcsname}}}}}

\foreach \x in \greeksymbols{\expandafter\xdef\csname bar\x\endcsname{\noexpand\ensuremath{\noexpand\bar{\csname \x\endcsname}}}}
\foreach \x in \greeksymbolsnoeta{%
\expandafter\xdef\csname barb\x\endcsname{\noexpand\ensuremath{\noexpand\bar{\noexpand\boldsymbol{ \csname \x\endcsname}}}}
}

\foreach \x in \greeksymbols{\expandafter\xdef\csname t\x\endcsname{\noexpand\ensuremath{\noexpand\tilde{\csname \x\endcsname}}}}
\foreach \x in \greeksymbolsnoeta{\expandafter\xdef\csname tb\x\endcsname{\noexpand\ensuremath{\noexpand\tilde{\noexpand\boldsymbol{ \csname \x\endcsname}}}}}

\providecommand{\normz}[2][-1]{
\ensuremath{\mathinner{
\ifthenelse{\equal{#1}{-1}}{ 
\!\left\|#2\right\|}{}
\ifthenelse{\equal{#1}{0}}{ 
\|#2\|}{}
\ifthenelse{\equal{#1}{1}}{ 
\bigl\|#2\bigr\|}{}
\ifthenelse{\equal{#1}{2}}{ 
\Bigl\|#2\Bigr\|}{}
\ifthenelse{\equal{#1}{3}}{ 
\biggl\|#2\biggr\|}{}
\ifthenelse{\equal{#1}{4}}{ 
\Biggl\|#2\Biggr\|}{}
}} 
}  

\providecommand{\floor}[2][-1]{
\ensuremath{\mathinner{
\ifthenelse{\equal{#1}{-1}}{ 
\!\left\lfloor#2\right\rfloor}{}
\ifthenelse{\equal{#1}{0}}{ 
\lfloor#2\rfloor}{}
\ifthenelse{\equal{#1}{1}}{ 
\!\bigl\lfloor#2\bigr\rfloor}{}
\ifthenelse{\equal{#1}{2}}{ 
\!\Bigl\lfloor#2\Bigr\rfloor}{}
\ifthenelse{\equal{#1}{3}}{ 
\!\biggl\lfloor#2\biggr\rfloor}{}
\ifthenelse{\equal{#1}{4}}{ 
\!\Biggl\lfloor#2\Biggr\rfloor}{}
}} 
}

\providecommand{\ceil}[2][-1]{
\ensuremath{\mathinner{
\ifthenelse{\equal{#1}{-1}}{ 
\!\left\lceil#2\right\rceil}{}
\ifthenelse{\equal{#1}{0}}{ 
\lceil#2\rceil}{}
\ifthenelse{\equal{#1}{1}}{ 
\!\bigl\lceil#2\bigr\rceil}{}
\ifthenelse{\equal{#1}{2}}{ 
\!\Bigl\lceil#2\Bigr\rceil}{}
\ifthenelse{\equal{#1}{3}}{ 
\!\biggl\lceil#2\biggr\rceil}{}
\ifthenelse{\equal{#1}{4}}{ 
\!\Biggl\lceil#2\Biggr\rceil}{}
}} 
}

\definecolor{mygrn}{rgb}{0,.8,0}
\definecolor{myred}{rgb}{.8,0,0}

\DeclareMathOperator{\supp}{{\mathrm{supp}}}

\def\1{\mathds{1}}

\DeclarePairedDelimiterX{\inp}[2]{\langle}{\rangle}{#1, #2}

\newcommand\declareop[3]{%
  \newcommand#1{%
    \mskip\muexpr\medmuskip*#2\relax
    {#3}%
    \mskip\muexpr\medmuskip*#2\relax
}}
\declareop\capprox{1}{{\sr{\const}{\approx}}} 
\declareop\logapprox{1}{{\sr{\mathrm{log}}{\approx}}} 

\def\Bin{\mathrm{Bin}}

\def\const{\mathsf{const}}

\usepackage{pifont}
\newcommand{\sr}{\stackrel}

\makeatletter
\newcommand{\vast}{\bBigg@{3}}
\newcommand{\Vast}{\bBigg@{4}}
\makeatother

\def\calE{{\mathcal{E}}}

\newcommand{\DTV}{D_{\mathrm{TV}}}

\newenvironment{talign*}
 {\let\displaystyle\textstyle\csname align*\endcsname}
 {\endalign}

\def\chrulefill{\leavevmode\leaders\hrule height 0.7ex depth \dimexpr0.4pt-0.7ex\hfill\kern0pt}

\renewcommand{\cite}{\citep}

\usepackage{hyperref}
\usepackage{url}
\usepackage{booktabs}
\usepackage{flafter}
\usepackage{graphicx}
\usepackage{subcaption}
\usepackage{svg}
\usepackage{nicefrac}
\usepackage{microtype}
\usepackage{multirow}
\usepackage{placeins}
\usepackage{bm}

\newif\ifFINAL
\FINALtrue 

\ifFINAL
   
  \def\introguide#1{}

  \renewcommand{\kj}[1]{{}}
  \renewcommand{\seiyun}[1]{{}}
\else
  \usepackage{transparent}
  \usepackage[inline]{showlabels}
  \usepackage{rotating}
  \renewcommand{\showlabelfont}%
  {\transparent{0.8}\scriptsize\bf\slshape\color{Lavender}}
\fi

\title{Uniform Race: Parameter-Free Approximate \\ Rejection Sampling}

\author{Seiyun Shin \\
Graduate School of Artificial Intelligence \\
Pohang University of Science and Technology \\
Pohang, 37673, South Korea \\
\texttt{seiyunshin@postech.ac.kr}
\And
Juhyeong Pang \\
Department of Computer Science \\
University of Wisconsin--Madison \\
Madison, WI 53706, USA \\
\texttt{pang35@wisc.edu} \\
\AND
Kwang-Sung Jun \\
Graduate School of Artificial Intelligence \\
Department of Computer Science and Engineering \\
Pohang University of Science and Technology \\
Pohang, 37673, South Korea \\
\texttt{kwangsungjun@postech.ac.kr}
}

\iclrpreprint

\begin{document}

\maketitle

\begin{abstract}
Consider the approximate sampling problem:
given $N$ independent samples from a proposal distribution $\mu$, the goal is to select one whose distribution is close to a target $\pi$ specified only up to a normalizing constant.
The seminal work of~\citet{block2023} provides a total variation (TV) error bound for approximate rejection sampling (RS) as a function of the algorithm's acceptance threshold $M$.
The ideal threshold $M$ that minimizes the bound, however, depends on properties of $(\pi, \mu)$ that are typically unavailable from the observed sample.
This raises a natural question:
\emph{Can one attain the best RS guarantee without taking the best $M$ as input?}
We answer this question affirmatively by proposing a parameter-free sampling algorithm called \emph{uniform race} (UR).
UR computes a score for each sample $y$ that is the importance weight $\pi(y)/\mu(y)$ divided by an independent uniform random variable and then returns the one with the highest score. 
We prove that UR attains the best $M$-tuned TV error bound for RS.
That is, UR, without knowing the best $M$, achieves the same performance guarantee as RS that takes the best $M$ as input.
%
We also prove that UR has no larger TV error than
the budget-calibrated version of the rejection sampler of~\citet{rohatgi2025} and sampling importance resampling (SIR).
In particular, we identify instances where UR's error is exponentially smaller in $N$ than that of these two baselines.
Furthermore, we establish conditions under which attaining this RS guarantee for every $(\pi, \mu)$ uniquely determines the selection probabilities as those of UR. 
Finally, test-time scaling experiments on LLM reasoning tasks corroborate the theoretical comparisons and demonstrate that UR remains competitive without requiring threshold selection.
\end{abstract}

\section{Introduction}
\label{sec:intro}


Inference-time computation often involves generating a finite number of samples from a proposal distribution and selecting one to approximate a target distribution.
Inspired by this, we study the following approximate sampling problem:
given $N$ independent and identically distributed (i.i.d.) samples $Y_1, \ldots, Y_N$ drawn from a proposal distribution $\mu$, select one whose marginal distribution is close to a target $\pi$ that is specified only up to a normalizing constant.
To describe how the target differs from the proposal, define the importance weight as $w := \frac{\dd\pi}{\dd\mu}$.
For discrete distributions, this is simply $w(y) = \pi(y)/\mu(y)$, so $w(y)$ measures how much the proposal probability at $y$ must be reweighted to obtain the target.
We assume that the sampler does not evaluate this weight directly.
Instead, at each sampled point $Y_i$ for $i = 1, \ldots, N$, the sampler only has access to the \emph{unnormalized} weight $w_u(Y_i) := Zw(Y_i)$, where $Z > 0$ is an \emph{unknown} normalizing constant.
Equivalently, $w(Y_i) = w_u(Y_i)/Z$, but neither $Z$ nor $w(Y_i)$ is available to the sampler.
We assume that the proposal covers the target, written $\pi \ll \mu$: every set with zero proposal probability also has zero target probability.
Throughout the main text, $0 < w(Y) < \infty$ almost surely under $\mu$;
we consider zero weights separately in Appendix~\ref{app:zero-weights}.

We refer to the sampled points as \emph{candidates}.
Let $\widehat{Y}$ be the returned candidate and $P_N := \mathcal{L}(\widehat{Y})$ its marginal distribution, where $\mathcal{L}(X)$ denotes the law of $X$.
Our goal is to make $P_N$ close to $\pi$, measured by total variation (TV) distance,
\begin{equation}
    \DTV(P_N, \pi) := \sup_D |P_N(D) - \pi(D)|,
    \label{eq:objective}
\end{equation}
where $D$ ranges over measurable sets.
We note that the distribution $P_N$ averages over both the candidates and the sampler's additional randomness.
A useful baseline is classical rejection sampling (RS), which produces an exact sample from $\pi$ if a valid acceptance threshold $M$ is available.
For the threshold to be valid, it must satisfy $w(y) \leq M$ over all possible proposals.
Given such a threshold, RS repeatedly draws $Y \sim \mu$ and accepts with probability $w(Y)/M$ until the first acceptance.
The returned sample follows $\pi$ since the probability of proposing $y$ and accepting it is proportional to $\mu(y)w(y) = \pi(y)$.
Notably, this does not require prior knowledge of the normalizing constant $Z$.
If only the unnormalized weights $w_u = Zw$ are available, the same acceptance test can be implemented using the corresponding upper bound $ZM$:
\(
    w_u(Y)/ZM = w(Y)/M.
\)
Hence RS need not compute $Z$, but it still relies on a valid upper bound, called an \emph{envelope}, on the importance weight, and may continue sampling until an acceptance occurs.

This exposes two desiderata in our setting:
the sampler should respect a prescribed proposal budget $N$, and should not require prior knowledge of a problem-dependent envelope.
One way to accommodate both is to choose an arbitrary threshold $M > 0$, accept each proposal with probability (in normalized units)
\(
    \min\{w(Y)/M, 1\},
\)
and return an observed candidate if all $N$ proposals are rejected.
This removes the need for a valid envelope, but still requires supplying the unnormalized threshold $ZM$.


The threshold now controls two sources of error.
First, if $M$ falls below some importance weights, the cap at one changes the distribution of accepted samples:
weights above $M$ are effectively replaced by $M$.
This is the \emph{clipping} effect.
To quantify it, for $M > 0$, define
\begin{equation}
    \calE_M(\pi, \mu) := \E_\mu[(w(Y) - M)_+], \quad
    A_M(\pi, \mu) := 1 - \calE_M(\pi, \mu) = \E_\mu[\min\{w(Y), M\}].
    \label{eq:truncation}
\end{equation}
Here $(x)_+ := \max\{x, 0\}$;
$\calE_M$ is the excess weight removed by clipping at $M$, while $A_M$ is the retained clipped mass.
Denote $\calE_M(\pi,\mu)$ by $\calE_M$ when $(\pi, \mu)$ is fixed.
For $M \geq 1$, $\calE_M$ is the standard $E_M$-divergence~\citep{liu2016e_};
see also~\citet[Example~4]{block2023}.
We use the same clipping definition for $0 < M < 1$.
Notice that $\calE_M$ bounds the TV error of the accepted distribution; it is nonincreasing in $M$ and tends to zero as $M \rightarrow \infty$ for every fixed $(\pi, \mu)$.
Second, increasing $M$ reduces the acceptance probability and makes it more likely that all $N$ proposals are rejected.

This tradeoff is central to approximate RS under a finite budget.
Building on the analysis from~\citet{block2023},
\citet[Lemma~D.4]{huang2025} analyze the clipped-threshold rule above and show that, for a fixed threshold $M$,
\begin{equation}
    \DTV(P_{N, M}^{\mathrm{RS}}, \pi)
    \lesssim \underbrace{2\calE_M(\pi, \mu)}_{\text{clipping}} + \underbrace{\exp(-N/M)}_{\text{budget exhaustion}},
    \label{eq:block23}
\vspace{-0.3em}
\end{equation}
where $P_{N, M}^{\mathrm{RS}}$ denotes the marginal output distribution of this finite budget procedure.
The two terms favor opposite choices of $M$:
a larger threshold reduces clipping, but makes exhausting the proposal budget more likely.
Consequently, the threshold giving the best guarantee hinges on the unknown target--proposal pair, including importance weights that may never appear among the $N$ observed candidates.
This motivates the central question of this paper:
\vspace{-0.5em}
\begin{center}
    \emph{Can one attain the best fixed threshold RS guarantee without taking $M$ as input?}
\end{center}
\vspace{-0.5em}

\paragraph{Uniform race.}
To answer the question, we propose a parameter-free sampling algorithm, called \emph{uniform race} (UR).
Let $[N]:=\{1,\ldots,N\}$.
For each candidate, the algorithm divides its observed weight by an independent uniform random variable on
$(0,1)$ and returns the candidate with the largest resulting score:
\begin{equation}
    \widehat{I}_U := \argmax_{i \in [N]} \frac{w_u(Y_i)}{U_i}, \qquad
    \widehat{Y}_U := Y_{\widehat{I}_U}, \qquad
    U_i \overset{\mathrm{i.i.d.}}{\sim} \Unif(0,1),
    \label{eq:algorithm}
\end{equation}

\begin{table}[!htbp]
    \centering
    \small
    \setlength{\tabcolsep}{3pt}
    \renewcommand{\arraystretch}{1.17}
    \begin{tabular*}{\linewidth}{@{\extracolsep{\fill}}llll@{}}
        \toprule
        Sampling rule
        & \shortstack{Marginal TV upper bound}
        & \shortstack{Extra input} \\
        \midrule
        SIR
        & $\displaystyle \frac{C^\pi}{N}$~\citep{agapiou2017}
        & --- \\
        RS with $M = C_\infty^\pi$
        & $\displaystyle F(C_\infty^\pi, N)$ (\Cref{lem:rs-components})
        & Exact envelope \\
        Budget-calibrated RS
        & $\displaystyle \frac{C^\pi}{N}\1\left\{N \lesssim C_\infty^\pi\log(eC_\infty^\pi)\right\} + F(m_N, N)$ (\Cref{thm:pilot-rs-fdom})
        & Calibration parameter $\delta_{\rm cal}$ \\
        \textbf{UR}
        & $\displaystyle \boldsymbol{\inf_{M > 0}F(M, N)}$ (\Cref{thm:main})
        & --- \\
        \bottomrule
    \end{tabular*}
    \caption{Comparison of sampling rules where $F(M, N):=2E_M(\pi, \mu) + \exp(-N/M)$ is the bound from~\eqref{eq:oracle-informal},
    $C^\pi := \E_{y\sim\mu}[(\pi(y)/\mu(y))^2]$, and $C_\infty^\pi := \operatorname*{ess\,sup}_\mu \pi(y)/\mu(y)$.
    For finite $C^\pi$, optimizing $F$ gives an $O((C^\pi/N)[1+\log(1+N/C^\pi)])$ bound, while Corollary~\ref{cor:rates} establishes the stronger log-free bound $2C^\pi/(eN)$ for UR.  
    For budget-calibrated RS, the bound holds for $N \geq 3$ with $m_N := N/[32\log(4\lceil N/2\rceil)]$ and uses $\delta_{\rm cal} := 1/\lceil N/2\rceil$.   
    Terms requiring an infinite moment or weight bound, multiplicative constants, and clipping by $1$ are omitted, and envelope RS assumes $C_\infty^\pi < \infty$.
}
    \label{tab:comparison}
\end{table}
Since $Z$ is common to all candidates, it cancels from the ranking.
By \emph{parameter-free}, we mean that the sampler does not take an acceptance threshold as input, yet its guarantee adapts automatically to the best threshold in hindsight.
It also does not require an estimate of the normalization constant $Z$; the proposal budget and any parameters that define the target remain inputs.
Specifically, our main contributions are:


\textbf{Main result (informal; Theorem~\ref{thm:main}).}
Let $P_N^U := \mathcal{L}(\widehat{Y}_U)$ be UR's marginal output distribution.
For every fixed $N \geq 1$,
\begin{equation}
    \DTV(P_N^U, \pi) \leq \inf_{M > 0}\left\{2\calE_M(\pi, \mu) + \exp(-N/M)\right\}.
    \label{eq:oracle-informal}
\end{equation}
The key difference from~\eqref{eq:block23} is the \emph{infimum over $M$}:
UR satisfies the bound for every threshold simultaneously, so choosing a threshold is unnecessary.
Theorem~\ref{thm:main} gives a sharper version of this guarantee.
Beyond this guarantee, we establish the following structural and comparative results.
\begin{itemize}[leftmargin=1em, nosep]
    \item \textbf{Minimax-optimal bounded weight guarantee.}
    Let $C_\infty^\pi := \operatorname*{ess\,sup}_\mu w$ denote the smallest almost-sure upper bound on the importance weight.
    When it is finite, we show that UR's TV error is at most $(1 - 1/C_\infty^\pi)^N$, even though the sampler is not given this bound.
    This guarantee matches the smallest possible worst-case error over each bounded weight class;
    see Corollary~\ref{cor:rates} in Section~\ref{sec:oracle}.
    
    \item \textbf{Conditional output distribution.}
    With normalized scores $S_i := w(Y_i)/U_i$, conditioning on $\max_i S_i = s$ gives the output law $\pi(\cdot \mid w < s)$.
    In particular, this law equals $\pi$ when $s > C_\infty^\pi$.
    Specifically, this identity underlies Theorem~\ref{thm:main}.
    
    \item \textbf{Comparison with budget-calibrated RS.}
    We next ask whether or not estimating the normalizer and choosing a threshold from the available budget can recover UR's accuracy.
    In particular, a direct budget-calibrated version of~\citet[Algorithm~5]{rohatgi2025} fixes a calibration input $\delta_{\mathrm{cal}}$ before sampling, estimates the normalizer from pilot proposals, and applies RS to fresh candidates.
    Theorem~\ref{thm:pilot-rs-fdom} shows that UR has no greater TV distance from $\pi$ at the same budget, and even when UR uses only the fresh sampling portion of that budget.
    Moreover, on a fixed binary pair, the baseline-to-UR TV error ratio grows at least exponentially with $N$, even after optimizing $\delta_{\rm cal}$ prior to sampling.
    See Section~\ref{sec:rohatgi}.
    
    \item \textbf{Comparison with SIR.}
    Rather than setting an acceptance threshold, sampling importance resampling (SIR)~\citep{rubin1987comment, smith1992bayesian} selects candidate $i$ with probability $w_u(Y_i)/\sum_j w_u(Y_j)$.
    We show that UR has no greater TV error for every $(\pi, \mu)$ and fixed $N$, and this also extends to every convex $f$-divergence.
    In particular, on a fixed binary pair, UR's TV error decays exponentially in $N$, whereas SIR's decays as $1/N$.
    See Section~\ref{sec:sir}.

    \item \textbf{Structural characterization and a sequential realization.}
    We further investigate whether selection rules other than UR can achieve the bounded weight guarantee of Corollary~\ref{cor:rates}.
    We show that UR's selection probabilities are uniquely determined by this guarantee under three natural requirements:
    the rule uses only the relative importance weights of the observed candidates, applies uniformly across all target--proposal pairs, and is permutation-equivariant.
    See Section~\ref{sec:uniqueness}.
    We also establish a surprising connection to independent Metropolis--Hastings (IMH).
    When initialized from the proposal distribution, IMH has exactly the same marginal output distribution as UR for every proposal budget $N$.
    Hence, UR and IMH achieve the same finite budget output distribution: UR is a symmetric one-shot rule, whereas IMH processes proposals sequentially.
    See Appendix~\ref{app:imh}.
\end{itemize}
Table~\ref{tab:comparison} summarizes explicit consequences of the error bounds:
UR satisfies the best fixed-threshold
RS upper bound $\inf_{M>0}F(M,N)$ without selecting
$M$ or requiring calibration.
When $C_\infty^\pi < \infty$, this includes the envelope RS
bound $F(C_\infty^\pi, N)$ without knowing $C_\infty^\pi$.
Moreover, Theorems~\ref{thm:pilot-rs-fdom} and~\ref{thm:SIR} show that UR has no larger marginal TV error than budget-calibrated RS or SIR at the same
proposal budget.
By contrast, budget-calibrated RS requires specifying $\delta_{\mathrm{cal}}$ prior to sampling.

\section{RS guarantees without a threshold}
\label{sec:oracle}
We now state the sharper guarantee underlying~\eqref{eq:oracle-informal}.
With $\calE_M$ and $A_M$ defined in~\eqref{eq:truncation}, the quantity $A_M/M$ is the acceptance probability of one proposal under the fixed threshold RS rule.
Thus $(1 - A_M/M)^N$ is the probability that all $N$ proposals are rejected.

\begin{theorem}[Best fixed threshold RS upper bound]
\label{thm:main}
Under the assumptions of Section~\ref{sec:intro}, for every fixed $N \geq 1$,
\vspace{-0.6em}
\begin{equation}
    \DTV(P_N^U, \pi)
    \leq \inf_{M > 0}
    \left\{\calE_M + \left(1 - \frac{A_M}{M}\right)^N\right\}.
    \label{eq:oracle}
\vspace{-0.4em}
\end{equation}
\end{theorem}

To prove Theorem~\ref{thm:main}, we compare uniform race (UR) with rejection sampling (RS) at a fixed threshold, and then optimize that threshold only in the analysis.
To this end, fix $N \geq 1$ under the assumptions of Section~\ref{sec:intro}, and use the same candidates and uniforms for both procedures.
Recall the normalized scores $S_i := \tfrac{w(Y_i)}{U_i}$ and their maximum $S_{(N)} := \max_{i \in [N]} S_i$, and define the threshold-crossing event $H_M := \{S_{(N)} \geq M\}$.
By construction, $S_i \geq M$ exactly when $U_i \leq \min\{w(Y_i)/M, 1\}$.
Hence $H_M$ is exactly the event that RS accepts at least one proposal when the same uniforms are used; on $H_M$, RS returns the first accepted candidate.
UR, by contrast, always returns the candidate with the largest score.
We denote their conditional distributions on $H_M$ by
\begin{equation}
    Q_M^{\mathrm{RS}}(\dd y) := \frac{\min\{w(y), M\}}{A_M}\mu(\dd y), \qquad
    Q_{N, M}^U := \mathcal{L}(\widehat{Y}_U \mid H_M).
    \label{eq:clipped-law-definition}
\vspace{-0.5em} 
\end{equation}
\begin{proof}[Proof of Theorem~\ref{thm:main}]
The proof has two key ingredients.
First, fix $M > 0$ and let $\delta_M := \Pr(H_M^c)$.
Lemma~\ref{lem:rs-components} yields $\delta_M = (1 - A_M/M)^N$ and $\DTV(Q_M^{\mathrm{RS}}, \pi) \leq \calE_M$.
Second, Lemma~\ref{lem:accepted-selection} then shows that selecting the largest accepted score has no greater conditional error than returning the first accepted candidate.
When $\delta_M > 0$, write
\(
    F_{N, M}^U := \mathcal{L}(\widehat{Y}_U\mid H_M^c).
\)
By the law of total probability,
\(
    P_N^U = (1 - \delta_M)Q_{N, M}^U + \delta_M F_{N, M}^U.
\)
With this decomposition, we obtain:
\begin{align}
    \DTV(P_N^U, \pi) &\stackrel{\mathrm{(i)}}{\leq}
    (1 - \delta_M)\DTV(Q_{N, M}^U, \pi) + \delta_M\DTV(F_{N, M}^U, \pi) \notag \\
    &\stackrel{\mathrm{(ii)}}{\leq}
    (1 - \delta_M)\DTV(Q_M^{\mathrm{RS}}, \pi) + \delta_M \notag \\
    &\stackrel{\mathrm{(iii)}}{\leq} \calE_M + \left(1 - \frac{A_M}{M}\right)^N,
    \label{eq:oracle-assembly}
\end{align}
where (i) follows from the convexity of total variation;
(ii) follows from Lemma~\ref{lem:accepted-selection} and $\DTV\leq 1$;
and (iii) follows from Lemma~\ref{lem:rs-components} and the definition of $\delta_M$.
If $\delta_M = 0$, then $P_N^U = Q_{N, M}^U$, and the same bound follows directly from Lemmas~\ref{lem:rs-components} and~\ref{lem:accepted-selection}.
Since $P_N^U$ does not depend on the analysis threshold $M$, taking the infimum over $M>0$ completes the proof.
We note that the lemmas are stated in Section~\ref{subsec:two-lemmas}, with full proofs in Appendix~\ref{app:oracle}.
\end{proof}

\subsection{Output distribution at a fixed largest score}
We first characterize the output conditional on the largest score.

\begin{proposition}[Output distribution at a fixed largest score]
\label{prop:score-law}
For every fixed $N \geq 1$, except on a set of score values having probability zero under the distribution of $S_{(N)}$,
\begin{equation}
    \mathcal{L}(\widehat{Y}_U \mid S_{(N)} = s) = \pi(\cdot \mid w < s),
    \label{eq:score-target-restriction}
\end{equation}
whenever $\pi(w < s) > 0$.
Here $\{w < s\} := \{y : w(y) < s\}$.
In particular, if $C_\infty^\pi < \infty$, then
\begin{equation}
    \mathcal{L}(\widehat{Y}_U \mid S_{(N)} = s) = \pi
    \label{eq:winning-score-target}
\end{equation}
for every $s > C_\infty^\pi$, except possibly on the same probability zero set.
\end{proposition}
\begin{proof}
See Appendix~\ref{app:conditional_law}.
The conditioning identity says that, at a fixed winning score $s$, UR samples from the target restricted to states whose importance weight is below $s$. 
If $C_\infty^\pi < \infty$, then conditional on $S_{(N)} > C_\infty^\pi$, the returned candidate has distribution $\pi$.
Finally, we note that the budget affects only the distribution of the largest score, not this conditional target rule.
\end{proof}

\subsection{Acceptance probability and conditional error}
\label{subsec:two-lemmas}
We now return to the two lemmas used to prove Theorem~\ref{thm:main}.
The first describes RS at a fixed threshold, and the second compares its conditional output distribution with that of UR.
\begin{lemma}
\label{lem:rs-components}
For every $M > 0$, $0 < A_M/M \leq 1$ and
\begin{equation}
    \Pr(H_M^c) = \left(1 - \frac{A_M}{M}\right)^N, \qquad
    \DTV(Q_M^{\mathrm{RS}}, \pi) \leq \calE_M.
    \label{eq:clipped}
\end{equation}
\vspace{-1.7em}
\end{lemma}
\begin{proof}
One proposal is accepted with probability $\E_\mu[\min\{w(Y)/M, 1\}] = A_M/M$, so independence gives the probability of no acceptance.
Conditioning a proposal on acceptance yields $Q_M^{\mathrm{RS}}$; Appendix~\ref{app:oracle} verifies that the first accepted candidate has this same distribution.
To bound its error, observe that
\(
    A_M Q_M^{\mathrm{RS}}(\dd y) = \min\{w(y), M\}\mu(\dd y) \leq \pi(\dd y).
\)
The residual measure has mass $1 - A_M$, so $\pi$ is a mixture containing $Q_M^{\mathrm{RS}}$ with weight $A_M$.
Consequently $\DTV(Q_M^{\mathrm{RS}}, \pi) \leq 1 - A_M = \calE_M$; when $A_M = 1$, the two distributions agree.
\end{proof}

\begin{lemma}
\label{lem:accepted-selection}
For every $N \geq 1$ and $M > 0$,
\(
    \DTV(Q_{N, M}^U, \pi) \leq \DTV(Q_M^{\mathrm{RS}}, \pi).
\vspace{-0.5em}
\)
\end{lemma}
Intuitively, clipping makes RS put too little probability on large weight states.
By selecting the largest score, UR favors these states more strongly and compensates for this loss.
\vspace{-0.8em}
\begin{proof}[Proof sketch]
Using the joint distribution from Proposition~\ref{prop:score-law}, we obtain two monotonicity properties.
First, the density of $Q_{N, M}^U$ relative to $\pi$ is nonincreasing in the weight.
Second, the density ratio $dQ_{N, M}^U/dQ_M^{\mathrm{RS}}$ is nondecreasing in the weight.
Together, these monotonicity properties show that selecting the largest accepted score does not increase the total variation error relative to the target.
The formal proof is deferred to Appendix~\ref{app:oracle}, where Lemma~\ref{lem:refinement} converts these two properties into the desired result.
\end{proof}
\vspace{-1.2em}
\subsection{Bounded weights and minimax error}
At $M = C_\infty^\pi < \infty$, both conditional distributions are exactly $\pi$, so the inequality in Lemma~\ref{lem:accepted-selection} is tight.
Note that there is then no clipping, so $Q_{C_\infty^\pi}^{\mathrm{RS}} = \pi$.
Averaging~\eqref{eq:winning-score-target} over $H_{C_\infty^\pi}$ and applying Lemma~\ref{lem:rs-components} yield:
\vspace{-1em}
\begin{equation}
    \mathcal{L}(\widehat{Y}_U \mid H_{C_\infty^\pi}) = \pi, \qquad \Pr(H_{C_\infty^\pi}^c) = \left(1 - \frac{1}{C_\infty^\pi}\right)^N.
    \label{eq:exact-threshold}
\end{equation}
Hence RS with $M = C_\infty^\pi$ has the same conditional output distribution $\pi$ on $H_{C_\infty^\pi}$.
Note that this does not imply equality of the full output distributions, since their conditional laws on $H_{C_\infty^\pi}^c$ may differ.
Let $\delta_N := \Pr(H_{C_\infty^\pi}^c) = (1 - \tfrac{1}{C_\infty^\pi})^N$.
When $\delta_N > 0$, let
\(
    F_{N, C_\infty^\pi}^U := \mathcal{L}(\widehat{Y}_U\mid H_{C_\infty^\pi}^c).
\)
It follows that
\begin{equation}
    \begin{aligned}
        P_N^U &= (1 - \delta_N)\pi + \delta_N F_{N, C_\infty^\pi}^U, \\
        \DTV(P_N^U, \pi) &= \delta_N\DTV(F_{N, C_\infty^\pi}^U, \pi) \leq \delta_N.
    \end{aligned}
    \label{eq:bounded-mixture}
\end{equation}
Here $F_{N, C_\infty^\pi}^U$ is simply UR's output conditioned on $H_{C_\infty^\pi}^c$; it is not an additional fallback step.
If $\delta_N = 0$, then $P_N^U = \pi$ directly.
In particular, Section~\ref{sec:sir} gives an example where UR and RS with $M = C_\infty^\pi$ also agree on the event of no acceptance, so their full laws coincide.
\Cref{eq:bounded-mixture} gives exponential convergence whenever $C_\infty^\pi < \infty$.
The next corollary places this bound alongside a second moment bound and identifies the exact minimax error over the class $C_\infty^\pi\leq C$.
\begin{corollary}[Error bounds and minimax optimality]
\label{cor:rates}
Let $C^\pi := \E_\mu[w(Y)^2]$.
For every $N \geq 1$,
\begin{equation}
    \DTV(P_N^U, \pi) \leq \min\left\{1,\frac{2C^\pi}{eN},\left(1 - \frac{1}{C_\infty^\pi}\right)^N\right\}.
    \label{eq:rates}
\end{equation}
If either $C^\pi$ or $C_\infty^\pi$ is infinite, its corresponding term is interpreted as the trivial bound $1$.
For every $C \geq 1$,
\vspace{-1em}
\begin{equation}
    \inf_{\mathsf{A}}\sup_{(\pi,\mu): C_\infty^\pi \leq C}\DTV(P_N^{\mathsf{A}}, \pi) = \left(1 - \frac{1}{C}\right)^N.
    \label{eq:minimax}
\end{equation}
Here the infimum is over procedures $\mathsf{A}$ that return one of the $N$ observed proposals.
The supremum is over pairs satisfying the stated weight bound.
\end{corollary}
Appendix~\ref{app:rates} proves both bounds.
The lower bound has a simple interpretation:
for a worst-case pair, the target places mass on a state that appears
in one proposal with probability $1/C$, and no algorithm can return that state if it is absent from all $N$ proposals.

\section{Comparison with budget-calibrated RS}
\label{sec:rohatgi}
Section~\ref{sec:oracle} compares UR with fixed threshold RS guarantees.
We now ask whether an RS procedure that estimates the unknown normalizer from pilot samples and calibrates its threshold to the available proposal budget can match UR's actual sampling accuracy.
A natural construction comes from the budget relation in~\citet[Algorithm~5]{rohatgi2025}; we call this specific reparameterization \emph{budget-calibrated rejection sampling}.

The algorithm starts by estimating $Z$ by averaging the unnormalized weights of $n$ pilot proposals and then applies RS to fresh candidates.
For a maximum proposal budget $N = 2n + 1 \geq 3$ and calibration parameter $\delta_{\rm cal} \in (0,1)$, invert the relation $n = 4M\log(4/\delta_{\rm cal})$ to set
\begin{equation}
    M_{N, \delta_{\rm cal}} := \frac{n}{4\log(4/\delta_{\rm cal})}, \qquad
    \widehat{Z} := \frac{1}{n}\sum_{i = 1}^n w_u(Y_i).
    \label{eq:budget-rs-calibration}
\end{equation}
Here $\delta_{\rm cal}$ is a calibration input before sampling; Appendix~\ref{app:budget-rs-law} records the sufficient condition under which it upper-bounds the TV error.
The procedure then tests $Y_{n+1}, \ldots, Y_{2n}$ in order, accepting $Y_i$ with probability $\min\{w_u(Y_i)/(M_{N, \delta_{\rm cal}}\widehat{Z}), 1\}$.
Budget-calibrated RS then returns the first accepted candidate, or the fresh proposal $Y_{2n+1}$ if all $n$ proposals are rejected.
To quantify how far the algorithm's output distribution is from $\pi$, let $P^R_{N, \delta_{\rm cal}}$ denote its output distribution.
For a general integer budget $N \geq 3$, take $n = \lfloor(N - 1)/2\rfloor$ and leave at most one proposal unused.

\subsection{Exponential separation on fixed target--proposal pair}
\label{subsec:fixed-pair-separation}

We first show an exponential separation on a fixed target--proposal pair.
The lower bound holds for every calibration input chosen before observing the pilot.

\begin{proposition}
\label{prop:budget-rs-separation}
Let $\mu = (1/2,1/2)$ and $\pi = (1/4,3/4)$.
For every odd total proposal budget $N \ge 3$,
\begin{equation}
    \inf_{\delta_{\rm cal} \in (0,1)}\DTV(P^R_{N, \delta_{\rm cal}},\pi)
    \geq \frac{1}{12} 2^{-(N-1)/2}, \qquad
    \DTV(P^U_N,\pi) = \frac{3}{4} 3^{-N}.
    \label{eq:budget-rs-separation}
\end{equation}
\end{proposition}

\paragraph{Why calibration leaves error.}
Write $n := (N-1)/2$, $M := M_{N, \delta_{\rm cal}}$,
and $T := M\overline{W}_n$, where $\overline{W}_n := n^{-1}\sum_{i=1}^n w(Y_i)$.
The weights are $w(0) = 1/2$ and $w(1) = 3/2$.
At threshold $T$, let $a_j(T) := \min\{w(j)/T, 1\}$ for $j \in \{0,1\}$.
Clipping cannot increase the ratio $a_1(T)/a_0(T)$ above $3$, so the accepted law assigns state~$1$ probability at most $3/4$.
The fallback assigns it probability $1/2$.
Therefore, pilot-conditioned errors have the same sign and cannot cancel.

If $M \leq 2$, all pilot weights equal $1/2$ with probability $2^{-n}$, giving $T \leq 1$.
On this event, $a_1(T) = 1$ and $a_0(T) \geq 1/2$, so the accepted law assigns state~$1$ probability at most $1/(1+a_0(T)) \leq 2/3$.
The fallback cannot increase it, yielding a marginal TV error of at least $2^{-n}(3/4-2/3) = 2^{-n}/12$.

If $M > 2$, symmetry gives $\Pr(\overline{W}_n \geq 1)\geq 1/2$.
On this event $T > 2$, so clipping is absent and each fresh proposal is accepted with probability $1/T < 1/2$.
Hence all $n$ proposals reject with probability $(1-1/T)^n \geq 2^{-n}$.
Since the accepted law is then $\pi$, the fallback contributes at least $(1/2)2^{-n}(3/4-1/2) = 2^{-n}/8$
to the marginal TV error.
See Appendix~\ref{app:budget-rs-binary} for the full proof.

For UR, a score is below $3/2$ exactly when $Y_i = 0$ and $U_i > 1/3$, an event of probability $1/3$.
All scores are below $3/2$ with probability $3^{-N}$, on which UR returns $0$; on the complement its law is $\pi$
by Proposition~\ref{prop:score-law}.
Its TV error is therefore $(3/4)3^{-N}$.

\paragraph{Beyond a difference in exponential rates.}
Despite an error ratio of at least $\frac{1}{3}(9/2)^{(N-1)/2}$, both methods have $\Theta(\log(1/\xi))$ sample complexity for TV tolerance $\xi \downarrow 0$ with optimal tuning; see Appendix~\ref{app:budget-rs-binary}.
We now fix $N$ and ask whether calibration can recover the bounded weight guarantee of Corollary~\ref{cor:rates} within a pair-independent factor.
\begin{corollary}[No uniform constant factor guarantee]
\label{cor:no-uniform-calibrated-rs}
Fix an odd $N \geq 13$.
There is no finite constant $K_N$, independent of $(\pi, \mu)$, satisfying
\[
    \inf_{\delta_{\rm cal} \in (0,1)}\DTV(P^R_{N,\delta_{\rm cal}},\pi)
    \leq K_N\left(1 - \frac{1}{C_\infty^\pi}\right)^N
\]
for all strictly positive target--proposal pairs on finite state spaces.
\end{corollary}
\begin{proof}
Proposition~\ref{prop:pilot-rs-local} in Appendix~\ref{app:pilot-rs-local} gives a binary family $(\pi_\eta,\mu_\eta)$ with lower weight proposal probability $\eta$ and $C_\infty^{\pi_\eta} = (1 - \eta/2)^{-1}$.
Therefore, $1 - 1/C_\infty^{\pi_\eta} = \eta/2$, and the optimized error $\Theta_N(\eta^2)$ divided by the benchmark $(\eta/2)^N$ is $\Theta_N(\eta^{2-N}) \to \infty$ as $\eta \downarrow 0$ at fixed $N$.
\end{proof}
From this, we note that at a fixed budget, no single constant factor allows calibrated RS to match UR's bounded weight guarantee uniformly over all target--proposal pairs.

\subsection{General comparison}
The comparison is not limited to these binary examples:
UR has no greater TV error than budget-calibrated RS for every $(\pi, \mu)$.

\begin{theorem}[Marginal error relative to budget-calibrated RS]
\label{thm:pilot-rs-fdom}
For every pair satisfying the positive-weight assumptions of
Section~\ref{sec:intro}, let $N \geq 3$ and $n := \lfloor(N-1)/2\rfloor$.
For every input $\delta_{\rm{cal}} \in (0,1)$,
\[
    \DTV(P_{2n + 1}^U, \pi)
    \leq \DTV(P_{n + 1}^U, \pi)
    \leq \DTV(P_{2n + 1, \delta_{\rm cal}}^{R}, \pi).
\]
More generally,
\begin{equation}
    D_f(P_{2n + 1}^U \Vert \pi)
    \leq D_f(P_{n + 1}^U \Vert \pi)
    \leq D_f(P_{2n + 1, \delta_{\rm cal}}^{R} \Vert \pi)
    \label{eq:pilot-rs-fdom}
\end{equation}
for every convex $f : (0, \infty) \to \mathbb{R}$ with $f(1) = 0$.
Here $D_f(P \Vert \pi) := \int f(\dd P/\dd\pi)\dd\pi$ denotes $f$-divergence~\citep{sason2016f}. 
Specifically, the following upper bounds hold for specified pre-pilot inputs.

\smallskip
\noindent\textup{(i) Finite second moment.}
If $C^\pi < \infty$, choose $\delta_{\rm cal} := 1/(n+1)$ and set $m_N := N/[32\log(4(n+1))]$.
Then
\begin{equation}
    \DTV(P^R_{N,1/(n+1)}, \pi)
    \leq \frac{16(C^\pi -1)}{N} + F(m_N, N/12),
    \label{eq:bc-fixed-schedule-F}
\end{equation}
and
\begin{equation}
    \DTV(P^R_{N,1/(n+1)}, \pi)
    \leq \frac{32C^\pi\log(eN)}{N}.
    \label{eq:bc-fixed-schedule-moment}
\end{equation}
\smallskip
\noindent\textup{(ii) Bounded weights.}
Let $C_\infty^\pi < \infty$ and suppose $n > 16C_\infty^\pi\log 4$.
With the input $\delta_{\rm cal} := 4\exp(-n/(16C_\infty^\pi)) \in (0,1)$,
\begin{equation}
    \DTV(P^R_{N,4\exp(-n/(16C_\infty^\pi))}, \pi)
    \leq 4\exp\left(-\frac{N-2}{32C_\infty^\pi}\right).
    \label{eq:bc-fixed-envelope-calibration}
\end{equation}
\end{theorem}
\begin{proof}
    See Appendix~\ref{app:pilot-rs-domination}.
    In particular, part~\textup{(i)} uses $\delta_{\rm cal} = 1/(n+1)$, which depends only on the budget.
    Part~\textup{(ii)} uses $\delta_{\rm cal} = 4\exp(-n/(16C_\infty^\pi))$, which also requires a supplied bound $C_\infty^\pi$.
    Notice that the divergence ordering holds for every $\delta_{\rm cal} \in (0,1)$, not only these two choices.
\end{proof}
Notice that budget-calibrated RS spends $n$ proposals estimating $Z$ before its fresh sampling stage.
UR needs no pilot stage: with only $n+1$ proposals, the maximum number of fresh proposals available to the baseline, it already has no larger divergence.
One can readily see that giving UR the full $2n+1$ budget can only improve the comparison.
\section{Comparison with sampling importance resampling (SIR)}
\label{sec:sir}
A standard threshold-free alternative is SIR, which selects candidate $i$ with probability $w_u(Y_i)/\sum_j w_u(Y_j)$.
Let $\widehat{Y}_{\mathrm{SIR}}$ denote the returned candidate and $P_N^{\mathrm{SIR}} := \mathcal{L}(\widehat{Y}_{\mathrm{SIR}})$ its marginal output distribution.
We first illustrate the difference on a binary example with $N = 2$, and then establish the general comparison at the same proposal budget.


\subsection{Example where UR and SIR differ}
Consider the binary family where $\mu = (\tfrac{1}{2}, \tfrac{1}{2})$ and $\pi = (\eps,1 - \eps)$.
Here $0 < \eps < \tfrac{1}{2}$.
At $\eps = 1/4$, this is the fixed pair in Proposition~\ref{prop:budget-rs-separation}.
State $1$ has the larger target probability, and the two importance weights are $(w_L, w_H) = (2\eps, 2(1 - \eps))$, where $C_\infty^\pi = 2(1 - \eps)$.
Fix $N = 2$ and write $B := (Y_1, Y_2)$.
Each of $00, 01, 10, 11$ has probability $\frac{1}{4}$.
Every rule returning an observed proposal must return state $0$ on $00$ and state $1$ on $11$.
The methods can therefore differ only when both states are present.

On either mixed batch, SIR chooses state $1$ with probability $w_H/(w_L + w_H) = 1 - \eps$.
For UR, let $U_L, U_H$ be the independent uniform variables attached to the two candidates.
The candidate in state $0$ is selected exactly when $\tfrac{w_L}{U_L} > \tfrac{w_H}{U_H}$, equivalent to $U_L < \tfrac{\eps}{1 - \eps}U_H$.
Conditioning on $U_H = u$ yields:
\begin{align}
    \Pr(\widehat{Y}_U = 0 \mid B = 01) &= \int_0^1 \Pr\left(U_L < \frac{\eps}{1 - \eps}u\right)\dd u
    = \frac{\eps}{2(1 - \eps)}.
    \label{eq:two-proposal-direct}
\end{align}
Hence UR chooses state $1$ with probability $1 - \tfrac{\eps}{2(1 - \eps)}$ on either mixed batch.
RS with $M = w_H = C_\infty^\pi$ accepts state $0$ with probability $\tfrac{\eps}{1 - \eps}$ and state $1$ with probability one.
It returns the first accepted proposal, or an observed proposal if both are rejected.
We note that this reference exploits the true unnormalized envelope $Zw_H$.
Averaging over the four equally likely batches in turn gives
\begin{equation}
    \Pr(\widehat{Y}_U = 1) = \frac{3}{4} - \frac{\eps}{4(1 - \eps)}, \qquad \Pr(\widehat{Y}_{\mathrm{SIR}} = 1) = \frac{3}{4} - \frac{\eps}{2}.
    \label{eq:two-proposal-ur}
\end{equation}
For RS with $M = w_H$, averaging the probabilities of returning state $1$ over the two mixed batches gives $1 - \tfrac{\eps}{2(1 - \eps)}$, matching UR's probability.
Hence, the two output distributions agree after averaging. 
\paragraph{Where the improvement comes from.}
On a mixed batch, SIR chooses state $1$ with its target probability $1-\eps$.
This does not make the marginal output exact, because the homogeneous batches $00$ and $11$ contribute unequal errors.
By contrast, UR corrects part of this finite batch bias by choosing state $1$ more often when both states are available.
Consequently,
\begin{equation}
    \underbrace{\frac{1 - 2\eps}{4}}_{\text{SIR error}} = \underbrace{\frac{\eps(1 - 2\eps)}{4(1 - \eps)}}_{\text{reduction under UR}} + \underbrace{\frac{(1 - 2\eps)^2}{4(1 - \eps)}}_{\text{UR error}}.
    \label{eq:two-proposal-accounting}
\end{equation}
When $\eps = \tfrac{1}{4}$, the identity gives $\tfrac{1}{8} = \tfrac{1}{24} + \tfrac{1}{12}$.


\paragraph{Why RS with $w_H = C_\infty^\pi$ agrees with UR.}
The calculation above shows that UR and RS with the exact envelope have the same output distribution.
From the threshold view, on $H_{w_H}$ both conditional output distributions are $\pi$.
Moreover, $H_{w_H}^c$ can occur only when both candidates are in state $0$, where both methods return state $0$.
Appendix~\ref{app:two-level-rs} extends this equality to general two-level weights.

\subsection{Comparison for general distributions}
\begin{theorem}[Marginal error relative to SIR]
\label{thm:SIR}
Under the positive weight assumptions in Section~\ref{sec:intro}, for every fixed $N \geq 1$,
\begin{equation}
    \DTV(P_N^U,\pi) \leq \DTV(P_N^{\mathrm{SIR}},\pi).
    \label{eq:tvdom}
\end{equation}
More generally,
\begin{equation}
    D_f(P_N^U\Vert\pi) \leq D_f(P_N^{\mathrm{SIR}}\Vert\pi)
    \label{eq:fdom}
\end{equation}
for every convex $f: (0, \infty) \to \R$ with $f(1) = 0$, allowing infinite divergences.
\end{theorem}
\paragraph{Why the comparison holds.}
Appendix~\ref{app:comparison} proves that, for every $t > 0$,
\(
    \Pr(w(\widehat{Y}_{\mathrm{SIR}}) > t) \leq \Pr(w(\widehat{Y}_U) > t) \leq \pi(w > t).
\)
Hence, relative to SIR, UR shifts probability toward larger importance weights without overshooting the target tail probabilities.
We highlight that this tail comparison alone is not enough for TV, since TV maximizes over all measurable events, not only events of the form $\{w > t\}$.
A second ingredient is that the density of the UR output relative to $\pi$ is nonincreasing in the weight.
Its excess probability is therefore at smaller weights and its shortfall at larger weights, so the threshold comparison controls the set attaining its TV error.
Appendix~\ref{app:comparison} gives the formal argument and extends it to every convex $f$-divergence.

We note that the standard second-moment bound $\DTV(P_N^{\mathrm{SIR}},\pi) \leq \min\{1, 6C^\pi/N\}$ follows from~\citet[Theorem 2.1]{agapiou2017};
Appendix~\ref{app:rates} recalls the argument.

\paragraph{Beyond two proposals.}
We now return to the binary example and let the proposal budget grow.
\begin{proposition}
\label{prop:binary-sir-separation}
Fix $\eps \in (0, 1/2)$, and let $\mu = (1/2, 1/2)$ and $\pi = (\eps, 1 - \eps)$.
For every $N \ge 1$,
\begin{align}
    \DTV(P_N^U,\pi) &= (1 - \eps)\left(\frac{1 - 2\eps}{2(1 - \eps)}\right)^N,\qquad
    \label{eq:binary-ur-numeric-rate} \\
    \DTV(P_N^{\mathrm{SIR}},\pi) &= \frac{2\eps(1 - \eps)(1 - 2\eps)}{N} + O_{\eps}(N^{-3/2}).
    \label{eq:binary-sir-numeric-rate}
\end{align}
\end{proposition}
\begin{proof}
See Appendix~\ref{app:binary-general}.
\end{proof}
\vspace{-1em}
The separation persists even when both states appear in nearly every batch.
SIR retains the leading bias of $O(1/N)$ caused by random candidate counts and normalization, whereas UR's remaining error is exponentially small.
Appendix~\ref{app:sir-bias} shows that this phenomenon extends to every fixed nontrivial pair with bounded positive weights.
\section{Characterizing the selection probabilities}
\label{sec:uniqueness}
Sections~\ref{sec:rohatgi} and~\ref{sec:sir} compare UR's output law with budget-calibrated RS and SIR, respectively.
These comparisons show advantages over particular alternatives, but leave a structural question:
\emph{can different selection probabilities satisfy the bounded weight guarantee of Corollary~\ref{cor:rates} uniformly?}
Within the class of sampling algorithms restricted as below, we prove that the answer is no: the guarantee determines the probabilities on every fixed positive weight vector, and they must coincide with UR's.

\paragraph{What can the sampler use?}
Fix $N \geq 2$ and let the algorithm $\mathsf{A}$ return one of the $N$ candidates.
We restrict our attention to a class of selection rules satisfying two structural properties:
First, its selection probabilities depend only on the observed weights up to a common positive factor, not on candidate states or other information about $(\pi, \mu)$.
The same rule must be used for every pair.
Second, changing the order of the candidates must only reorder their selection probabilities.
For instance, weights $(1,2)$ and $(10,20)$ must give the same probabilities, while $(2,1)$ must exchange them.
We note that both UR and SIR satisfy these restrictions.
For a fixed weight vector $w_{1:N} := (w_1, \ldots, w_N)$, let $p_i^{\mathsf{A}}(w_{1:N})$ be the probability of selecting candidate $i$ over the sampler's random choices.
The restrictions require, for every $c > 0$ and permutation $\sigma$,
\begin{equation}
    p_i^{\mathsf{A}}(cw_{1:N}) = p_i^{\mathsf{A}}(w_{1:N}), \qquad
    p_{\sigma(i)}^{\mathsf{A}}(\sigma w_{1:N}) = p_i^{\mathsf{A}}(w_{1:N}),
    \label{eq:invariances}
\end{equation}
where $(\sigma w_{1:N})_{\sigma(i)} = w_i$.
Write $p_i^U(w_{1:N})$ for UR's probability on the same vector and $P_N^{\mathsf{A}}$ for the output distribution after averaging over the proposal draws.
The following theorem requires the bounded weight guarantee from Corollary~\ref{cor:rates}.

\begin{theorem}
\label{thm:uniqueness}
Fix $N \geq 2$ and suppose $\mathsf{A}$ satisfies the information restriction above and~\eqref{eq:invariances}.
If a finite constant $L_N$, independent of $(\pi,\mu)$, satisfies
\begin{equation}
    \DTV(P_N^{\mathsf{A}},\pi) \leq L_N\left(1 - \frac{1}{C_\infty^\pi}\right)^N
    \label{eq:global-characterization}
\end{equation}
for every target and proposal on a finite state space with strictly positive probabilities, then
\[
    p_i^{\mathsf{A}}(w_{1:N}) = p_i^U(w_{1:N}) \qquad \text{for every } w_{1:N} \in (0,\infty)^N \text{ and } i \in [N].
\]
Conversely, UR satisfies the bound with $L_N = 1$.
\end{theorem}
\begin{proof}
    See Appendix~\ref{app:uniqueness}.
    We note that the requirement ranges over all $(\pi, \mu)$ at fixed $N$, so a different rule may have smaller error on a particular pair without satisfying this requirement.
\end{proof}

\paragraph{A sequential realization.}
Theorem~\ref{thm:uniqueness} determines UR's selection probabilities, but does not prescribe a unique way to implement them.
Interestingly, independent Metropolis--Hastings (IMH), initialized from the proposal distribution, produces exactly the same marginal output distribution as UR for every proposal budget $N$.
More precisely, for any fixed batch of proposals, IMH's selection probabilities agree with those of UR after averaging over all possible processing orders.
Note that UR selects from the entire batch at once, whereas IMH processes proposals sequentially, and their selections need not coincide on an individual run.
See Appendix~\ref{app:imh} for details.


\begin{figure}[!htbp] 
    \centering
    \includegraphics[width=0.815\linewidth, page=1]{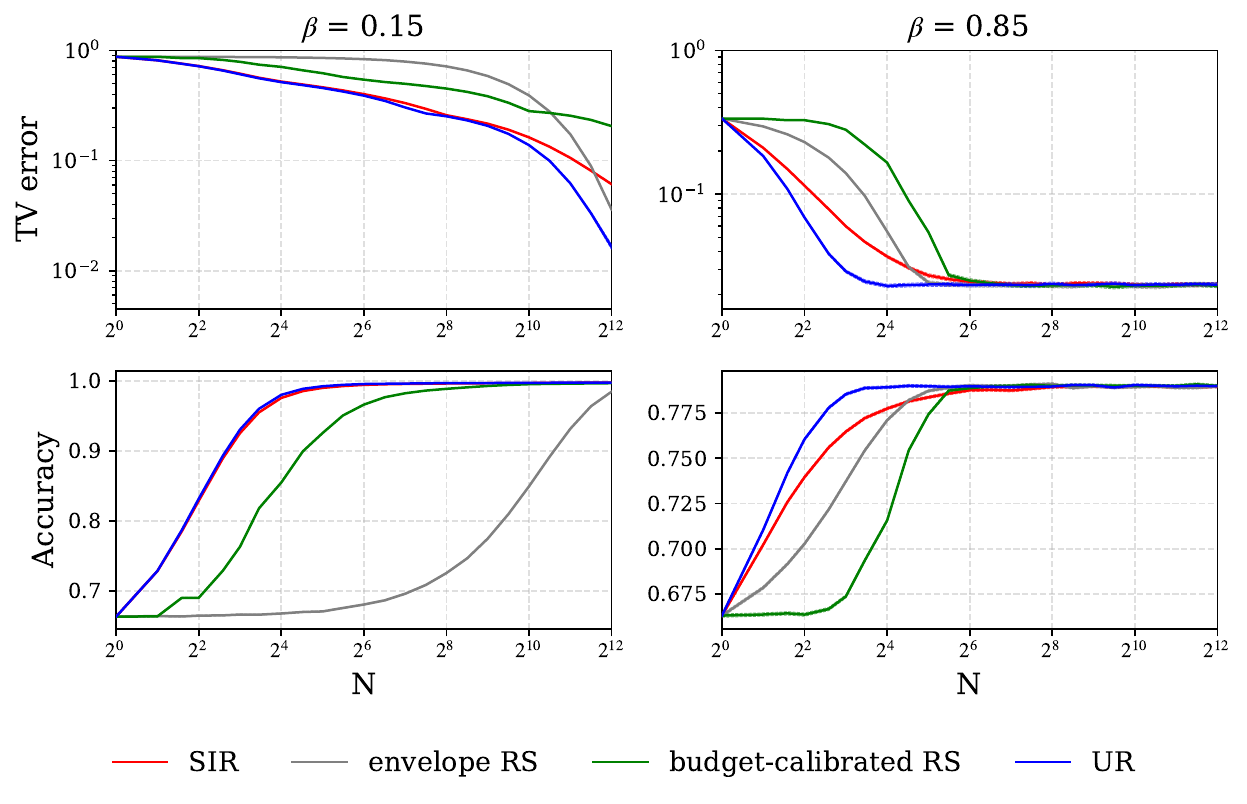}
    \vspace{-0.5em}
    \caption{TV error (top) and ground-truth accuracy (bottom) w.r.t.\ sampling budget $N$ on GSM8K Problem~\#258, for $\beta = 0.15$ (left) and $\beta = 0.85$ (right).
    The bands indicate 95\% confidence intervals.
    }
    \label{fig:tv_error_and_acc}
\end{figure}

\section{Test-time scaling experiments}
\label{sec:experiments}
We complement our theoretical findings with controlled LLM experiments on mathematical reasoning tasks where the samplers approximate a reward-tilted target policy, also known as soft BoN~\citep{aminian26bestofn,verdun25soft}.
We evaluate both sampling fidelity, measured by TV distance, and ground-truth answer accuracy.

\paragraph{Experimental setup and evaluation.}
For each prompt $x$, we pre-generate a fixed response pool $\mathcal{A}_x^\star$ of $N^\star$ responses from \path{Llama-3.2-3B-Instruct}~\citep{grattafiori2024llama}.
Let $\widehat{\mu}(y \mid x)$ denote the empirical distribution on $\mathcal{A}_x^\star$.
For each sampling budget $N$, we draw $N$ candidates i.i.d.\ from $\widehat{\mu}(\cdot \mid x)$ and apply each sampling procedure.
Each response $y$ is scored by the OASST reward model~\citep{kopf2023openassistant}, yielding $\widehat{r}(x,y)$.
Motivated by the standard KL-regularized RLHF objective~\citep{jaques2017sequence, jaques2020human, rafailov2023direct}, we define the reward-tilted target policy
\(
    \pi_\beta(y \mid x) \propto \widehat{\mu}(y \mid x)\exp(\tfrac{\widehat{r}(x,y)}{\beta}),\ y \in \mathcal{A}_x^\star,
\)
so that the unnormalized importance weight for each sampler is $w_u(x,y) = \exp(\tfrac{\widehat{r}(x,y)}{\beta})$.
We compare four methods:
\textsc{UR (ours)}, \textsc{envelope RS} with threshold equal to the maximum weight in $\mathcal{A}_x^\star$, \textsc{budget-calibrated RS}, adapted from~\citet[Algorithm~5]{rohatgi2025}, and \textsc{SIR}.
For budget-calibrated RS, we select the calibration
parameter retrospectively to optimize each reported
metric; see Appendix~\ref{app:exp_details} for details.
We evaluate on GSM8K~\citep{cobbe2021training}, varying the sampling budget $N$ and temperature parameter $\beta$.
Estimates of the TV error $\DTV(P_N, \pi_\beta)$ and ground-truth accuracy are based on $10^6$ Monte Carlo trials, with ground-truth accuracy defined as the probability that the selected response has the correct final answer.
We also evaluate on MATH500~\citep{hendrycks2021measuring, lightman2024let}, where we observe the same trends; additional results and implementation details are given in Appendix~\ref{app:exp_details}.

\paragraph{Results and analysis.}
Figure~\ref{fig:tv_error_and_acc} reports TV error (top) and ground-truth accuracy (bottom) for the same GSM8K problem under two values of $\beta$.
Across both settings, UR exhibits lower estimated TV error than budget-calibrated RS and SIR, consistent with Theorems~\ref{thm:pilot-rs-fdom} and~\ref{thm:SIR},
while remaining comparable to envelope RS without requiring threshold tuning.
Before reaching the Monte Carlo floor, UR exhibits substantially faster decay in TV error than budget-calibrated RS and SIR.
The accuracy curves show how these distributional differences are reflected in ground-truth accuracy.
For $\beta = 0.15$, the accuracies of UR, SIR, and budget-calibrated RS are already close to the target accuracy at larger budgets, despite visible differences in TV error.
This indicates that the TV upper bound on the accuracy gap is loose in this regime.
Specifically, accuracy probes only the probability of the correct answer event, so
\(
    \left|\operatorname{Acc}(P_N; x) - \operatorname{Acc}(\pi_\beta; x)\right|
    \leq \DTV(P_N, \pi_\beta).
\)
For $\beta = 0.85$, the accuracy gaps track the TV errors more closely, with UR approaching the common target accuracy at a smaller budget.
Additional results across prompts and values of $\beta$, including MATH500, are given in Appendix~\ref{app:exp_results_analysis}.

\clearpage

\bibliography{iclr2027_conference}
\bibliographystyle{iclr2027_conference}

\appendix

\section{Related work}
\label{app:related_work}

\paragraph{Approximate rejection sampling.}
Our problem is closely related to the approximate rejection sampling (RS) framework of
\citet{block2023}: given $N$ independent samples from a proposal distribution $\mu$, the goal is to select one whose marginal distribution is close in total variation to a target distribution $\pi$.
They characterize the minimax sample complexity of this problem under $f$-divergence constraints between target and proposal.
Their upper bounds are obtained by truncating the target according to a likelihood ratio threshold and applying RS to the resulting truncated distribution.
Our focus is complementary.
Rather than optimizing sample complexity over a divergence class, we study a single parameter-free selection rule that uses only the observed unnormalized importance weights.
In particular, UR requires neither the normalizing constant of the target nor an envelope bound or clipping threshold as an algorithmic input.
The clipped RS formulation has also been used in recent inference-time sampling work~\citep{huang2025}.
\vspace{-0.5em}
\paragraph{Priority sampling.}
UR uses the same selection rule as a special case of classical priority sampling~\citep{duffield2007}.
Given a fixed collection of weighted items, priority sampling assigns each item an independent priority $w_i/U_i$ and retains the items with the largest priorities.
When only one item is retained, this reduces exactly to the UR selection rule, $I = \arg\max_i \frac{w_i}{U_i}$.
Classical priority sampling, however, is designed for a different statistical
task:
retain the top-$k$ items with the largest priorities.
Together with the $(k+1)$st largest priority as a threshold, the retained items are assigned adjusted weights that yield unbiased estimators of arbitrary subset sums (i.e., $\sum_{i\in A} w_i$ for subsets $A$ of the population).
Consequently, the classical analysis is concerned with inclusion, unbiased
weight estimation, and the variance of subset-sum estimators.
In our setting, the items themselves are random,
$Y_i \stackrel{\mathrm{iid}}{\sim} \mu$, and we instead study the marginal distribution of the selected candidate $Y_I$ as an approximation to the target distribution $\pi$.
Therefore, while UR is a special case of the classical priority selection rule, our objective and analysis concern the induced sampling distribution rather than subset-sum estimation.
Our contribution is not a new top-priority selection rule.
Rather, we analyze how closely the algorithm's output distribution approximates $\pi$, establish finite budget error bounds and comparisons with existing methods, and provide a structural characterization of the selection probabilities.

\vspace{-0.5em}
\paragraph{Sampling importance resampling and importance sampling.}
Sampling importance resampling (SIR) also generates a candidate pool from $\mu$ and selects a candidate using importance weights.
Conditional on $Y_{1:N}$, SIR selects index $i$ with probability
\(
    \frac{w_u(Y_i)}{\sum_{j=1}^N w_u(Y_j)},
\)
and therefore coincides with sampling from the self-normalized empirical importance distribution.
For finite $N$, the resulting marginal distribution is generally different
from $\pi$, although it converges to the target as the candidate pool grows.
\citet{skare2003improved} propose adjusted resampling probabilities to reduce
this finite sample bias and compare the resulting method with independent Metropolis--Hastings.
Related work on self-normalized importance sampling studies the bias and
stability of estimators based on the same normalized importance weights;
see, e.g.,~\citet{agapiou2017} and~\citet{deligiannidis2026}.
These works primarily concern estimation of expectations under $\pi$, whereas our primary object is the marginal distribution of a single selected candidate and its total variation distance from $\pi$.
\vspace{-0.5em}
\paragraph{Independent Metropolis--Hastings.}
Independent Metropolis--Hastings (IMH) also combines proposals from a fixed
distribution $\mu$ with the importance ratio between $\pi$ and $\mu$.
Unlike the one-shot selection rules considered here, IMH constructs a Markov
chain with stationary distribution $\pi$.
Its convergence properties and relationship with rejection and importance sampling have been studied extensively.
\citet{liu1996} compares metropolized independent sampling with rejection
sampling and importance sampling, while~\citet{mengersen1996} establish classical convergence results for independence samplers.

More recent work gives sharp convergence analyses~\citep{wang2022} and studies IMH under unbounded weight functions~\citep{deligiannidis2026}.
Specifically,~\citet{deligiannidis2026} formulate the comparison through the marginal TV error of a selected candidate and compare the rates of SIR and IMH.
Here we establish a fixed budget convex divergence ordering between UR and SIR and identify UR with proposal-initialized IMH at every fixed proposal budget.
The independence of its proposals also permits parallel computation;
\citet{jacob2011} exploit batches of independent proposals to improve IMH-based estimation.
Our perspective is complementary: we study the full marginal law of a single returned candidate. 
We show that, at every fixed proposal budget, proposal-initialized IMH has exactly the same marginal output distribution as the one-shot UR rule.
Consequently, the best fixed threshold RS guarantee, bounded weight minimax
guarantee, and convex $f$-divergence comparison with SIR established for UR
also apply to the IMH terminal law.


\section{Proof of the RS bound}
\label{app:oracle}
Section~\ref{sec:oracle} assembles Theorem~\ref{thm:main} from the acceptance calculation and the comparison of conditional errors.
We first prove the comparison in Lemma~\ref{lem:accepted-selection}.
The remaining subsections justify the conditioning, give the exact clipping error, and extend the results to zero weights.
Write $W := w(Y)$ for $Y \sim \mu$, with $0 < W < \infty$ almost surely until Appendix~\ref{app:zero-weights}.

\subsection{Proof of Lemma~\ref{lem:accepted-selection}}
Fix $N \geq 1$ and $M > 0$.
Let $r_M$ and $q_M$ be functions of the scalar weight such that
\[
    \frac{\dd Q_{N, M}^U}{\dd \pi}(y) = r_M(w(y)), \qquad \frac{\dd Q_M^{\mathrm{RS}}}{\dd \pi}(y) = q_M(w(y)).
\]
We suppress the fixed $N$ in $r_M$.
Integrating the joint law~\eqref{eq:app-winning-score} over $s \geq M$ gives, for $x > 0$,
\[
    r_M(x) = \frac{N}{\Pr(H_M)} \int_{\max\{x, M\}}^\infty \frac{G(s)^{N - 1}}{s^2} \dd s, \quad
    q_M(x) = \frac{\min\{x, M\}}{A_M x},
\]
where $G(s) := \Pr(W/U \leq s)$ for an independent $U \sim \Unif(0, 1)$.
Both functions are positive and finite, and their compositions with $w$ integrate to one under $\pi$.
The comparison lemma below requires two properties: $r_M$ is nonincreasing, and $r_M/q_M$ is nondecreasing.
The first follows because the lower integration limit cannot decrease with $x$.
For the second, both functions are constant on $x \leq M$.
For $x > M$, their ratio is a positive constant times
\begin{equation}
    x \int_x^\infty \frac{G(s)^{N - 1}}{s^2} \dd s = \int_0^1 G(x/u)^{N - 1} \dd u,
    \label{eq:app-change-variable-monotonicity}
\end{equation}
using $u = x/s$.
The right side is nondecreasing in $x$ because $G$ is nondecreasing.
The two expressions for the ratio agree at $x = M$, proving the second property.

The two monotonicity statements play different roles.
The fact that $r_M$ is nonincreasing says that UR's conditional output, relative to $\pi$, has its excess mass on lower weight states.
The fact that $r_M/q_M$ is nondecreasing says that moving from the RS-accepted law to UR shifts mass toward larger weights.
Combining these statements to compare their TV errors, Lemma~\ref{lem:refinement}, with $P = \pi$, $Q = Q_M^{\mathrm{RS}}$, and $R = Q_{N, M}^U$, now proves Lemma~\ref{lem:accepted-selection}.

\begin{lemma}[A comparison of densities]
\label{lem:refinement}
Let $R \ll Q \ll P$ be probability measures whose densities have the form $\dd Q/\dd P = q(W) > 0$ and $\dd R/\dd P = r(W)$ for a measurable real function $W$.
If $r$ is nonincreasing and $r/q$ is nondecreasing on the range of $W$, then
\[
    \DTV(P, R) \leq \DTV(P, Q).
\]
\end{lemma}
\begin{proof}
Let $D := \{y : r(W(y)) < 1\}$, the set where the density ratio of $R$ to $P$ is below one.
Hence $\DTV(P, R) = P(D) - R(D)$, and it is enough to show $R(D) \geq Q(D)$.
Put $v := r/q$, so $\dd R/\dd Q = v(W)$ and $\int v(W) \dd Q = 1$.
If $y \in D$ and $z \notin D$, then $r(W(y)) < 1 \leq r(W(z))$.
Since $r$ is nonincreasing, $W(y) > W(z)$, and the monotonicity of $v$ gives $v(W(y)) \geq v(W(z))$.
Thus, when $Q$ is multiplied by the density ratio $v(W)$ to obtain $R$, the factors on $D$ are at least as large as those outside $D$.
Because both measures have total mass one, this implies
\begin{align*}
    R(D) - Q(D)
    &\stackrel{\mathrm{(i)}}{=} R(D) Q(D^c) - Q(D) R(D^c) \\
    &\stackrel{\mathrm{(ii)}}{=} \int_D \int_{D^c} [v(W(y)) - v(W(z))] Q(\dd z) Q(\dd y) \\
    &\stackrel{\mathrm{(iii)}}{\geq} 0.
\end{align*}
Here (i) uses total mass one, (ii) substitutes $\dd R = v(W) \dd Q$, and (iii) uses the pointwise comparison above.
The absolute integrand is bounded by $v(W(y)) + v(W(z))$, so it is integrable.
It follows that
\[
    \DTV(P, R) = P(D) - R(D) \leq P(D) - Q(D) \leq \DTV(P, Q),
\]
where the last inequality is the definition of TV.
\end{proof}

\subsection{Proof of Proposition~\ref{prop:score-law}}
\label{app:conditional_law}

We first consider the state and score of a single candidate, then account for the other candidates.
Given $Y_i = y$, the change of variables $u = w(y)/s$ has absolute derivative $w(y)/s^2$ and requires $s > w(y)$.
Since $w(y)\mu(\dd y) = \pi(\dd y)$, the joint distribution is
\begin{equation}
    \Pr(Y_i \in \dd y,S_i \in \dd s) = \frac{\1\{w(y) < s\}}{s^2}\pi(\dd y)\dd s.
    \label{eq:app-single-score}
\end{equation}
To be selected, this candidate must also have a larger score than the other $N - 1$ candidates.
Let $G(s) := \Pr(S_i \leq s)$ be the distribution function of one score.
Independence yields the factor $G(s)^{N - 1}$.
Crucially, this factor depends on $s$ but not on $y$.
The scores are atomless, so ties have probability zero, and summing over the $N$ possible selected indices yields:
\begin{equation}
    \Pr(\widehat{Y}_U \in \dd y,S_{(N)} \in \dd s) = \frac{NG(s)^{N - 1}}{s^2}\1\{w(y) < s\}\pi(\dd y)\dd s.
    \label{eq:app-winning-score}
\end{equation}
Integrating over $y$ gives the density of $S_{(N)}$ as $\tfrac{N G(s)^{N-1}}{s^2}\pi(w < s)$.
Therefore, for every measurable $D$ and almost every such $s$,
\[
    \Pr(\widehat{Y}_U \in D \mid S_{(N)} = s)
    = \frac{\pi(D \cap \{w < s\})}{\pi(w < s)}
    = \pi(D \mid w < s),
\]
proving~\eqref{eq:score-target-restriction};
if $s > C_\infty^\pi$, $\pi(w < s) = 1$, so~\eqref{eq:winning-score-target} holds.
This completes the proof.

\subsection{Conditioning on the largest score}
The following identities are consequences of the joint law in Proposition~\ref{prop:score-law}.
They also make precise the conditioning on a continuously distributed score.
\begin{lemma}[Integrating the joint distribution]
\label{lem:joint-score}
Integrating~\eqref{eq:app-winning-score} over $s \geq M$ and dividing by $\Pr(H_M)$ gives $r_M$ above.
Integrating over all positive scores gives
\[
    \frac{\dd P_N^U}{\dd \pi}(y) = N \int_{w(y)}^\infty \frac{G(s)^{N - 1}}{s^2} \dd s.
\]
If $C_\infty^\pi < \infty$, then for every measurable state set $D$ and Borel score set $B \subset (C_\infty^\pi, \infty)$,
\begin{equation}
    \Pr(\widehat{Y}_U \in D, S_{(N)} \in B) = \pi(D) \Pr(S_{(N)} \in B).
    \label{eq:app-score-factorization}
\end{equation}
\end{lemma}
\begin{proof}
The first two identities follow by integrating the nonnegative joint density.
For $s > C_\infty^\pi$, the indicator $\1\{w(y) < s\}$ is one for $\pi$-almost every $y$.
The joint density therefore factors into $\pi(\dd y)$ and the density of $S_{(N)}$, which gives~\eqref{eq:app-score-factorization} after integration over $D \times B$.
\end{proof}
Conditioning on $S_{(N)} = s$ is not division by the probability of this event, which is zero.
Instead, the joint density verifies the conditional probability formula
\[
    \Pr(\widehat{Y}_U \in D \mid S_{(N)} = s) = \frac{\pi(D \cap \{w < s\})}{\pi(w < s)}
\]
where $\pi(w < s) > 0$.
On the remaining score values we may define it to be $\pi(D)$.
Those values carry no score probability, since the score density contains the factor $\pi(w < s)$.
\Cref{eq:app-score-factorization} also shows that, conditional on $S_{(N)} > C_\infty^\pi$, the output state has law $\pi$ independently of the largest score.

\subsection{Accepted samples and the event of no acceptance}
Put $a_M := A_M/M$ and let $J_M$ be the first accepted index on $H_M$.
For $J_M = i$, the first $i - 1$ candidates must be rejected and candidate $i$ must be accepted.
Thus
\begin{align*}
    \Pr(Y_{J_M} \in D, H_M)
    &\stackrel{\mathrm{(i)}}{=} \sum_{i = 1}^N (1 - a_M)^{i - 1} a_M Q_M^{\mathrm{RS}}(D) \\
    &\stackrel{\mathrm{(ii)}}{=} [1 - (1 - a_M)^N] Q_M^{\mathrm{RS}}(D),
\end{align*}
where (i) follows from the independence and the accepted law of one candidate; and (ii) follows from evaluating the geometric sum.
Dividing by $\Pr(H_M) = 1 - (1 - a_M)^N$ proves the claimed conditional law.
This calculation also explains why that law does not depend on $N$.
The clipping bound in Lemma~\ref{lem:rs-components} can be sharpened to an exact identity:
\begin{align}
    \DTV(\pi, Q_M^{\mathrm{RS}})
    &\stackrel{\mathrm{(i)}}{=} \int \left(w(y) - \frac{\min\{w(y), M\}}{A_M}\right)_+ \mu(\dd y) \notag \\
    &\stackrel{\mathrm{(ii)}}{=} \int \left(w(y) - \frac{M}{A_M}\right)_+ \mu(\dd y) \notag \\
    &\stackrel{\mathrm{(iii)}}{=} \calE_{M/A_M} \leq \calE_M,
    \label{eq:app-clipped-error}
\end{align}
where (i) holds since TV is the integrated positive difference of the densities.
For (ii), both positive parts vanish when $w(y) \leq M$, since $A_M \leq 1$, and their contents agree when $w(y) > M$.
Step (iii) follows due to the definition of $\calE_M$ evaluated at the threshold $M/A_M$.
The last inequality uses $M/A_M \geq M$ and monotonicity in the threshold.
Consequently, ensuring the factor $1 - \delta_M$ in~\eqref{eq:oracle-assembly} gives $(1 - \delta_M) \calE_{M/A_M} + \delta_M$, where $\delta_M = (1 - A_M/M)^N$.

\paragraph{A direct acceptance interpretation for bounded weights.}
Suppose $C_\infty^\pi < \infty$.
For candidate $i$, define
\[
    H_i := \left\{U_i \leq \frac{w(Y_i)}{C_\infty^\pi}\right\}, \quad V_i := \frac{C_\infty^\pi U_i}{w(Y_i)} \quad \text{on } H_i.
\]
For every measurable $D$ and $0 \leq v \leq 1$,
\begin{equation}
    \Pr(Y_i \in D, H_i, V_i \leq v) = \int_D \frac{v w(y)}{C_\infty^\pi} \mu(\dd y) = \frac{v \pi(D)}{C_\infty^\pi}.
    \label{eq:accepted-independence}
\end{equation}
After division by $\Pr(H_i) = 1/C_\infty^\pi$, this factors as $\pi(D) v$.
Hence, conditional on acceptance, $Y_i$ has law $\pi$ and is independent of $V_i \sim \Unif(0, 1)$.
Conditioning on which candidates are accepted preserves independence of these pairs, because each acceptance uses only its own candidate and uniform.
Among the accepted candidates, UR selects the smallest $V_i$, equivalently the largest score $C_\infty^\pi/V_i$.
This choice uses variables independent of the accepted states, so the selected state still has law $\pi$.
Together with the independent probability $1/C_\infty^\pi$ of each acceptance, this proves~\eqref{eq:exact-threshold} directly.

\subsection{Zero importance weights}
\label{app:zero-weights}
Now allow $w \geq 0$ with $\E_\mu w = 1$.
If all observed weights are zero, let both UR and SIR choose each of the $N$ candidate indices with probability $1/N$.
Otherwise neither rule selects a zero weight.
Write $p_0 := \mu(w = 0) < 1$.
The all zero event has probability $p_0^N$, and, if $p_0 > 0$, both conditional output laws on that event are $\mu(\cdot \mid w = 0)$.
The joint density formulas still hold for positive scores, with an additional mass $p_0^N$ at $S_{(N)} = 0$.
For every $M > 0$, the conditional laws $Q_{N, M}^U$ and $Q_M^{\mathrm{RS}}$ put no mass on $\{w = 0\}$.
Their density formulas and monotonicity properties on $\{w > 0\}$ are unchanged, so Lemma~\ref{lem:accepted-selection} still applies.
The mixture argument in Theorem~\ref{thm:main} then applies with the same bound on $H_M^c$.
The bounded weight conclusion is unchanged, and Appendix~\ref{app:rates} includes the extra zero score event in its moment argument.
For the SIR comparison, the proof on a fixed batch in Appendix~\ref{app:comparison} applies after zero entries are removed from a batch with a positive weight.
On a batch where all weights are zero, the two rules agree.
On $\{w > 0\}$, the density of the UR law relative to $\pi$ is still $h_N^U(w(y))$ from~\eqref{eq:marginal-density}, but its integral is now $1 - p_0^N$.
All remaining mass is on a set of target probability zero.
Therefore the full probability gap is attained on
\[
    D := \{y : w(y) > 0, h_N^U(w(y)) < 1\}.
\]
This is an upper weight set within $\{w > 0\}$, possibly the whole positive weight set.
The selected-weight comparison, including the endpoint $0$ where the positive masses agree, gives $P_N^U(D) \geq P_N^{\mathrm{SIR}}(D)$.
Consequently,
\[
    \DTV(P_N^U, \pi) = \pi(D) - P_N^U(D) \leq \pi(D) - P_N^{\mathrm{SIR}}(D) \leq \DTV(P_N^{\mathrm{SIR}}, \pi).
\]
We retain the positive weight assumption for the general $f$-divergence statement, so no convention for singular mass is needed there.

\subsection{Exponential relaxation: Derivation of~\eqref{eq:oracle-informal}} 
\label{app:exp_relaxation}
The output law $P_{N, M}^{\mathrm{RS}}$ uses the clipped acceptance rule and returns an observed candidate if all $N$ proposals are rejected.
Conditional on at least one acceptance, its law is $Q_M^{\mathrm{RS}}$; on the complementary event, its TV error is at most one regardless of the specified fallback.
Thus the same mixture argument as in~\eqref{eq:oracle-assembly} gives
\[
    \DTV(P_{N, M}^{\mathrm{RS}}, \pi)
    \leq \calE_M + \left(1 - \frac{A_M}{M}\right)^N.
\]
To obtain the simpler form used in the Introduction, note that $0 < A_M \leq 1$, $0 \leq \calE_M < 1$, and $A_M = 1 - \calE_M$.
Then
\begin{align*}
    \calE_M + \left(1 - \frac{A_M}{M}\right)^N
    & \stackrel{\mathrm{(i)}}{\leq}
    \calE_M + \exp\left(-\frac{N A_M}{M}\right) \\
    & \stackrel{\mathrm{(ii)}}{\leq}
    2\calE_M + A_M\exp(-N/M) \\
    & \stackrel{\mathrm{(iii)}}{\leq}
    2\calE_M + \exp(-N/M).
\end{align*}
Here (i) uses $1 - x \leq \exp(-x)$.
For (ii), convexity of the exponential gives
\[
    \exp(-A_M x) = \exp(\calE_M \cdot 0 + A_M(-x))
    \leq \calE_M + A_M \exp(-x), \qquad
    x := N/M.
\]
(iii) uses $A_M \leq 1$.
This proves~\eqref{eq:block23}; applying it inside Theorem~\ref{thm:main} proves~\eqref{eq:oracle-informal}.
The term $\exp(-N/M)$ is not asserted to bound the rejection probability by itself when clipping occurs.
The relaxation reallocates part of that probability bound to the clipping term.



\section{Comparison with budget-calibrated RS}
\label{app:budget-rs}

We analyze the budget reparameterization in~\eqref{eq:budget-rs-calibration}.
Throughout this appendix, weights are positive and finite almost surely, and $N = 2n + 1 \geq 3$ is the maximum budget counting all pilot, rejection, and fallback proposals.
For other integer budgets $N \geq 3$, use $n = \lfloor(N-1)/2\rfloor$, leaving at most one proposal unused.
Appendix~\ref{app:budget-rs-law} records the conditional output law, and Appendix~\ref{app:pilot-rs-domination} proves the divergence ordering and the bounds for specified calibrations in Theorem~\ref{thm:pilot-rs-fdom}.
Appendices~\ref{app:budget-rs-binary} and~\ref{app:pilot-rs-local} prove the fixed pair and fixed budget local separations, respectively.

\subsection{Baseline and its conditional law}
\label{app:budget-rs-law}
Algorithm~5 of~\citet{rohatgi2025} exploits $n = 4M\log(4/\delta_{\rm cal})$ pilot proposals, at most $n$ fresh rejection proposals, and one fresh fallback upon complete rejection.
Solving for $M$ gives~\eqref{eq:budget-rs-calibration}.
The effective threshold is therefore $M_{N, \delta_{\rm cal}}\widehat{Z}$.
Proposition~F.2 of~\citet{rohatgi2025} applies when
\begin{equation}
    M_{N, \delta_{\rm cal}} \geq 4C_\infty^\pi \quad \Longleftrightarrow \quad
    n \geq 16C_\infty^\pi\log(4/\delta_{\rm cal}).
    \label{eq:budget-rs-source-condition}
\end{equation}
Under this sufficient condition, $\DTV(P_{N, \delta_{\rm cal}}^R, \pi) \leq \delta_{\rm cal}$.
Therefore, the source provides an $O(C_\infty^\pi\log(1/\eps))$ sufficient budget when $\delta_{\rm cal} = \eps$ is chosen as the target accuracy.
Without this condition, however, we note that $\delta_{\rm cal}$ is not necessarily a certified TV bound.
The divergence ordering concerns the actual output law and does not require this sufficient condition.

\paragraph{Condition on the pilot first.}
Write $M := M_{N, \delta_{\rm cal}}$ and define
\[
    \overline{W}_n := \frac{1}{n}\sum_{i = 1}^n w(Y_i), \qquad
    T := M\overline{W}_n = \frac{M\widehat{Z}}{Z}, \qquad
    r_t := 1 - \frac{A_t}{t}.
\]
The normalized threshold $T$ is an analysis variable; the algorithm does not know $Z$.
Conditioning on the pilot freezes $T$ but not the fresh candidates.
With probability $1-r_T^n$ at least one fresh trial
accepts, in which case the output law is $Q_T^{\mathrm{RS}}$;
otherwise the independent fallback has law $\mu$.
Averaging this conditional mixture over the random pilot yields~\eqref{eq:budget-rs-output-law} below.
\begin{equation}
    P_{N, \delta_{\rm cal}}^R
    = \E\left[(1 - r_T^n)Q_T^{\mathrm{RS}} + r_T^n\mu\right].
    \label{eq:budget-rs-output-law}
\end{equation}
We note that the expectation is over the pilot.
If $T \geq C_\infty^\pi$, the accepted law is exactly $\pi$; hence normalizer estimation affects the error through clipping and complete rejection.

\subsection{Proof of Theorem~\ref{thm:pilot-rs-fdom}}
\label{app:pilot-rs-domination}

We prove Theorem~\ref{thm:pilot-rs-fdom} in two parts.
We first compare the marginal divergences for every calibration input, then establish quantitative bounds for specified calibrations.


\subsubsection{Divergence ordering}
\label{app:budget-rs-ordering}
We first prove the TV comparison using only $k := n + 1$ UR proposals, extend it to convex $f$-divergences, and compare two UR budgets.
We note that these ordering arguments do not use the sufficient condition in Appendix~\ref{app:budget-rs-law} (\eqref{eq:budget-rs-source-condition}).

\paragraph{Conditioning and randomizing the fresh-candidate order.}
Condition on the pilot and any independent threshold randomization, and denote the resulting fixed normalized threshold by $T$.
Generate all $k=n+1$ fresh candidates in advance, including the fallback, and uniformly permute them.
Uniformly permuting these i.i.d. candidates preserves their joint law.
Assign each candidate an independent acceptance indicator with probability $\min\{w_i/T, 1\}$, including a virtual indicator for the last candidate.
That indicator does not change the output:
the last candidate is returned whenever all preceding candidates are rejected, regardless of its indicator.
If an accepted index exists, the output is consequently the first accepted index in the permuted order; if no index is accepted, it is the last index.
Conditional on the candidate list and indicators, the rule therefore selects uniformly among accepted indices, or uniformly among all indices when none is accepted.
Conditional on $T$, the following lemma applies at the fixed threshold $\tau = T$.

\begin{lemma}[A fixed batch comparison]
\label{lem:fixed-batch-rs-order}
Fix positive weights $w_{1:k}$ and $\tau>0$.
Retain index $i$ independently with probability $\min\{w_i/\tau, 1\}$, select uniformly among retained indices, and use a uniform index if none is retained.
Let $q_i(\tau)$ be its selection probabilities.
For every $t > 0$,
\[
    \sum_{i : w_i > t} q_i(\tau)
    \leq \sum_{i : w_i > t} p_i^U(w_{1:k}).
\]
The same inequality holds with $w_i \geq t$.
\end{lemma}

\begin{proof}
For $k = 1$, the rules coincide.
Otherwise, put $w_{\max} := \max_i w_i$ and $I_t := \{i : w_i > t\}$.
We show that moving the threshold to $w_{\max}$ cannot decrease the probability of selecting an index in $I_t$.
At this threshold, the retained set is nonempty and the rule is UR's Bernoulli representation from Appendix~\ref{app:bernoulli}.

For $k \geq 2$, suppose $\tau \leq w_{\max}$ and put $v_i := \min\{w_i, \tau\}$.
The rule is UR on the clipped vector $v_{1:k}$.
Couple this race with UR on $w_{1:k}$ using the same uniforms.
The factor $w_i/v_i = \max\{1, w_i/\tau\}$ is nondecreasing in $w_i$.
If the clipped race selects $i \in I_t$, then for every $j \notin I_t$,
\[
    \frac{w_i}{U_i}
    = \frac{w_i}{v_i}\frac{v_i}{U_i}
    > \frac{w_j}{v_j}\frac{v_j}{U_j}
    = \frac{w_j}{U_j}.
\]
Ties have probability zero.
Hence UR also selects within $I_t$, proving the comparison in this case.

Now consider the case where $\tau \geq w_{\max}$.
Set $z_i := w_i/w_{\max}$, and write $\alpha := w_{\max}/\tau \in (0, 1]$.
Integrating the reciprocal size of the retained set as in Appendix~\ref{app:bernoulli} yields
\[
    q_i(w_{\max}/\alpha) = z_i\int_0^{\alpha}\prod_{j \neq i}(1 - z_j u)\dd u + \frac{1}{k}\prod_{j = 1}^k(1 - \alpha z_j).
\]
The first term accounts for selection from a nonempty retained set; the second accounts for the fallback when no index is retained.
Consequently,
\begin{align*}
    \frac{\dd}{\dd\alpha}\sum_{i \in I_t}q_i(w_{\max}/\alpha)
    &\stackrel{\mathrm{(i)}}{=}
    \frac{1}{k}\sum_{i \in I_t}\sum_{j \notin I_t}
    \left[z_i\prod_{\ell \neq i}(1 - \alpha z_\ell)
    - z_j\prod_{\ell \neq j}(1 - \alpha z_\ell)\right] \\
    &\stackrel{\mathrm{(ii)}}{=}
    \frac{1}{k}\sum_{i \in I_t}\sum_{j \notin I_t}
    (z_i - z_j)\prod_{\ell \notin \{i, j\}}(1 - \alpha z_\ell) \\
    &\geq 0.
\end{align*}
Here (i) differentiates and groups terms across $I_t$ and its complement.
(ii) cancels the two terms containing $\alpha z_i z_j$.
The last inequality uses $z_i \geq z_j$ for the indicated indices.
Hence the probability increases up to $\alpha = 1$, where the rule is UR.
Empty and full sets cause no difficulty, and the argument is unchanged for $I_t := \{i : w_i \geq t\}$.
\end{proof}

\paragraph{From weight ordering to divergence ordering.}
Average Lemma~\ref{lem:fixed-batch-rs-order} over the fresh candidates and the pilot.
Writing $\widehat{Y}_R$ for the baseline output and using $k$ candidates for $\widehat{Y}_U$ gives
\[
    \Pr(w(\widehat{Y}_R) > t)
    \leq \Pr(w(\widehat{Y}_U) > t),
\]
with the same inequality for weights at least $t$.
This establishes the required selected weight stochastic ordering.
For the extension, notice first that both output laws have positive finite densities relative to $\pi$ almost surely.
For the baseline, positivity follows from the possibility of accepting its first fresh candidate, and finiteness follows from $P_{N, \delta_{\rm cal}}^{R}(D) \leq k\mu(D)$.
Write $h^R(y) := (\dd P_{N, \delta_{\rm cal}}^{R}/\dd\pi)(y)$.
By Lemma~\ref{lem:ur-density}, $\dd P_k^U/\dd\pi=h_k^U\circ w$, where $h_k^U$ is nonincreasing.
For $c > 0$, let $D_c := \{y : h_k^U(w(y)) > c\}$.
This is a lower weight set, so $P_k^U(D_c) \leq P_{N, \delta_{\rm cal}}^{R}(D_c)$.
For $X \sim \pi$,
\begin{align*}
    \E[(h_k^U(w(X)) - c)_+]
    &\stackrel{\mathrm{(i)}}{=} P_k^U(D_c) - c\pi(D_c) \\
    &\stackrel{\mathrm{(ii)}}{\leq} P_{N, \delta_{\rm cal}}^{R}(D_c) - c\pi(D_c) \\
    &\stackrel{\mathrm{(iii)}}{\leq} \E[(h^R(X) - c)_+].
\end{align*}
Here (i) integrates over the positive set;
(ii) uses the probability ordering;
and (iii) integrates the positive part over the whole space.
At $c = 1$, this also recovers the TV comparison.
Both density ratios have mean one, so affine terms have equal expectations.
Positive sums of these hinge inequalities therefore compare every convex piecewise-linear function with finitely many breakpoints.
For an arbitrary allowed convex $f$, subtract a supporting affine function at one, approximate the resulting nonnegative function increasingly by maxima of finitely many supporting affine functions and zero, and apply monotone convergence as in Appendix~\ref{app:comparison}.
This proves the selection-stage inequality in~\eqref{eq:pilot-rs-fdom}, including infinite divergences.
Nothing here required a sample-mean pilot: any positive finite threshold determined from an independent pilot and held fixed during the fresh stage is covered.
The fresh-proposal fallback is essential to this argument; arbitrary other fallbacks and thresholds adapted to the fresh candidates are not covered.

\paragraph{Comparison under the same budget.}
It remains to compare two UR budgets.
A decreasing upper bound would not establish that the actual error decreases.
\begin{lemma}[UR's marginal divergences decrease with the budget]
\label{lem:ur-budget-monotonicity}
For every positive weight pair and integers $1 \leq k \leq m$,
\[
    \DTV(P_m^U, \pi) \leq \DTV(P_k^U, \pi).
\]
More generally, $D_f(P_m^U \Vert \pi) \leq D_f(P_k^U \Vert \pi)$ for every convex $f$ covered by Theorem~\ref{thm:pilot-rs-fdom}.
\end{lemma}
\begin{proof}
We first compare selected weights, then apply the same positive part argument.
For a score value $s$ with $\pi(w < s) > 0$,
\[
    \pi(w > t \mid w < s) =
    \begin{cases}
        0, & s \leq t, \\
        1 - \frac{\pi(w \leq t)}{\pi(w < s)}, & s > t.
    \end{cases}
\]
This probability is nondecreasing in $s$ on the score values that can occur.
Using the first $k$ candidates among the same $m$ candidate--uniform pairs couples their maxima so that $S_{(k)} \leq S_{(m)}$.
Proposition~\ref{prop:score-law} then gives $P_k^U(w > t) \leq P_m^U(w > t)$.
The same reasoning applies to $w \geq t$, replacing $\pi(w \leq t)$ by $\pi(w < t)$ in the second branch.
For $c > 0$, the set $D_c := \{y : h_m^U(w(y)) > c\}$ is a lower weight set, and hence $P_m^U(D_c) \leq P_k^U(D_c)$.
Therefore, for $X \sim \pi$,
\begin{align*}
    \E[(h_m^U(w(X)) - c)_+]
    &= P_m^U(D_c) - c\pi(D_c) \\
    &\leq P_k^U(D_c) - c\pi(D_c) \\
    &\leq \E[(h_k^U(w(X)) - c)_+].
\end{align*}
Setting $c = 1$ proves the TV assertion.
Both density ratios have mean one, so the convex approximation argument above proves the general assertion, including infinite divergences.
\end{proof}
Combining the selection stage comparison with Lemma~\ref{lem:ur-budget-monotonicity} therefore proves the divergence ordering part of Theorem~\ref{thm:pilot-rs-fdom}.
For an even maximum budget, the baseline leaves one proposal unused and the same budget monotonicity applies.

\subsubsection{Quantitative bounds for specified calibrations}
\label{app:budget-rs-upper}
We now prove parts~\textup{(i)} and~\textup{(ii)} of Theorem~\ref{thm:pilot-rs-fdom}.
Both bounds follow by separating pilot failure from the RS error when the pilot is accurate.
The finite second-moment bound uses Chebyshev's inequality; the bounded weight bound uses exponential concentration.
A final refinement gives the bound for the same budget-only calibration displayed in Table~\ref{tab:comparison}.

\begin{proof}
Fix $N \geq 3$, put $n := \lfloor (N - 1)/2 \rfloor$, and choose $\delta_{\rm cal} \in (0,1)$ before the pilot.
As in Appendix~\ref{app:budget-rs-law}, write
\[
    M := M_{N,\delta_{\rm cal}} = \frac{n}{4\log(4/\delta_{\rm cal})}, \qquad
    \overline{W}_n := \frac{\widehat Z}{Z} = \frac{1}{n}\sum_{i=1}^n w(Y_i), \qquad
    T := M\overline{W}_n.
\]
The algorithm uses $n$ pilot proposals, at most $n$ fresh rejection proposals, and a fresh fallback only if all trials reject.
Its maximum proposal count is $2n + 1 \leq N$.
Neither $T$ nor $\overline{W}_n$ needs to be evaluated by the algorithm.

\paragraph{Error decomposition.}
Conditioning on the pilot fixes $T = t$ without changing the fresh proposal law.
By \eqref{eq:budget-rs-output-law}, the conditional output law is $(1 - r_t^n)Q_t^{\mathrm{RS}} + r_t^n\mu$, where $r_t = 1 - A_t/t$.
Lemma~\ref{lem:rs-components} and the relaxation in Appendix~\ref{app:exp_relaxation} yield
\begin{align*}
    \DTV\left((1 - r_t^n)Q_t^{\mathrm{RS}} + r_t^n\mu, \pi\right)
    &\leq E_t + \left(1 - \frac{A_t}{t}\right)^n \\
    &\leq 2E_t + \exp(-n/t).
\end{align*}
This conditional error is also at most one.
Define the pilot event
\[
    \cG := \left\{\left|\overline{W}_n - 1\right| \leq \frac{1}{2}\right\}.
\]
On $\cG$, the threshold satisfies $M/2 \leq T \leq 3M/2$.
The lower endpoint controls clipping; the upper endpoint controls the exponential term.
On $\cG^c$, use the TV bound one.
Averaging the conditional laws therefore yields:
\begin{equation}
    \DTV(P^R_{N,\delta_{\rm cal}}, \pi)
    \leq \underbrace{\Pr(\cG^c)}_{\text{pilot failure}} + \underbrace{2E_{M/2} + \exp\left(-\frac{2n}{3M}\right)}_{\text{conditional RS error on }\cG}.
    \label{eq:bc-pilot-decomposition}
\end{equation}
At fixed $n$, the pilot event does not depend on $\delta_{\rm cal}$, whereas the conditional RS bound depends on it through $M$.
The event is used only in the analysis; the sampler does not test it.

\paragraph{Proof of Theorem~\ref{thm:pilot-rs-fdom} (i): finite second moment.}
Suppose $C^\pi < \infty$.
Since $\E_\mu[w] = 1$ and $\Var_\mu(w) = C^\pi - 1$, Chebyshev's inequality gives
\[
    \Pr(\cG^c) \leq \frac{4(C^\pi - 1)}{n}.
\]
A finite second moment thus controls pilot failure, not a deterministic interval containing every realized $\widehat Z$.
Substituting into \eqref{eq:bc-pilot-decomposition} yields
\begin{equation}
    \DTV(P^R_{N,\delta_{\rm cal}}, \pi) \leq \frac{4(C^\pi - 1)}{n} + 2E_{M/2} + \exp\left(-\frac{2n}{3M}\right).
    \label{eq:budget-rs-tail-upper}
\end{equation}
For $x \geq 0$ and $t > 0$, $(x - t)_+ \leq x^2/(4t)$, so $2E_{M/2} \leq C^\pi/M$.
The exponential term is handled separately by substituting the algorithmic multiplier:
\[
    \exp\left(-\frac{2n}{3M}\right)
    = \exp\left(-\frac83\log\frac4{\delta_{\rm cal}}\right)
    = \left(\frac{\delta_{\rm cal}}{4}\right)^{8/3}.
\]
Consequently, for every pre-pilot calibration,
\begin{equation}
    \DTV(P^R_{N,\delta_{\rm cal}}, \pi)
    \leq \frac{4(C^\pi - 1) + 4C^\pi\log(4/\delta_{\rm cal})}{n} + \left(\frac{\delta_{\rm cal}}{4}\right)^{8/3}.
    \label{eq:budget-rs-moment-upper}
\end{equation}
Now choose $\delta_{\rm cal} = 1/(n + 1)$, as in part~\textup{(i)}.
Then
\begin{equation}
    \DTV(P^R_{N,1/(n + 1)}, \pi)
    \leq \frac{4(C^\pi - 1) + 4C^\pi\log(4(n + 1))}{n} + [4(n + 1)]^{-8/3}.
    \label{eq:budget-rs-budget-only-upper}
\end{equation}
Since $C^\pi \geq 1$ and $[4(n + 1)]^{-8/3} \leq C^\pi/n$, the right-hand side is at most
\begin{align*}
\frac{C^\pi}{n}\left(5 + 4\log(4(n + 1))\right)
    &\leq \frac{8C^\pi\log(e(n + 1))}{n} \\
    &\leq \frac{32C^\pi\log(eN)}{N}.
\end{align*}
The first inequality uses $n + 1 \geq 2$ and $4\log 2 < 3$; the second uses $n + 1 \leq N$ and $N \leq 4n$.

For the threshold form in part~\textup{(i)}, put $L_N := \log(4(n + 1))$, so $m_N = N/(32L_N)$.
Then $M = n/(4L_N)$ and $m_N \leq M/2$, while
\[
    \frac{N}{12m_N} = \frac{2n}{3M} = \frac{8L_N}{3}.
\]
Using the monotonicity of $E_M$ in \eqref{eq:budget-rs-tail-upper} gives
\begin{align*}
    \DTV(P^R_{N,1/(n + 1)}, \pi)
    &\leq \frac{16(C^\pi - 1)}{N} + 2E_{m_N} + \exp\left(-\frac{N}{12m_N}\right) \\
    &= \frac{16(C^\pi - 1)}{N} + F(m_N, N/12).
\end{align*}
The last line rewrites the explicit bound in the notation of Theorem~\ref{thm:pilot-rs-fdom}.
This proves part~\textup{(i)}.

\paragraph{Proof of Theorem~\ref{thm:pilot-rs-fdom} (ii): bounded weights.}
Let $C_\infty^\pi < \infty$.
The variables $w(Y_i)/C_\infty^\pi$ lie in $[0,1]$ and have mean $1/C_\infty^\pi$.
The lower- and upper-tail multiplicative Chernoff
bounds with relative deviation $1/2$, followed by a union bound, yield
\begin{equation}
    \Pr(\cG^c) \leq \exp\left(-\frac{n}{8C_\infty^\pi}\right) + \exp\left(-\frac{n}{12C_\infty^\pi}\right)
    \leq 2\exp\left(-\frac{n}{12C_\infty^\pi}\right).
    \label{eq:bc-bounded-pilot}
\end{equation}
For any confidence parameter $\delta \in (0,1)$,
distinct from the calibration input $\delta_{\rm cal}$,~\eqref{eq:bc-bounded-pilot} implies
\begin{equation}
    n \geq 12C_\infty^\pi\log(2/\delta)
    \quad\Longrightarrow\quad
    \Pr(G^c) \leq \delta.
\end{equation}
If $M \geq 4C_\infty^\pi$, then on $G$,
$T \geq M/2 \geq 2C_\infty^\pi$, so there is no
clipping.
Consequently, whenever $n \geq 12C_\infty^\pi\log(2/\delta)$ and $M \geq 4C_\infty^\pi$,
\begin{equation}
    \DTV(P^R_{N, \delta_{\rm cal}}, \pi)
    \leq \delta + \left(\frac{\delta_{\rm cal}}{4}\right)^{8/3}.
\end{equation}
In particular, under $n \geq 16C_\infty^\pi\log(4/\delta_{\rm cal})$, we may choose $\delta = \delta_{\rm cal}/2$, which yields
\[
    \DTV(P^R_{N,\delta_{\rm cal}}, \pi)
    \leq \frac{\delta_{\rm cal}}{2} + \left(\frac{\delta_{\rm cal}}{4}\right)^{8/3}
    \leq \delta_{\rm cal}.
\]
This recovers the sufficient guarantee in Appendix~\ref{app:budget-rs-law} (\eqref{eq:budget-rs-source-condition}).
Note that the Chernoff confidence parameter $\delta$ is an analysis choice, whereas $\delta_{\rm cal}$ determines the algorithmic threshold.

Now assume $n > 16C_\infty^\pi\log 4$ and choose $\delta_{\rm cal} = 4\exp(-n/(16C_\infty^\pi))$, as in part~\textup{(ii)}.
This input lies in $(0,1)$ and gives $M_{N,\delta_{\rm cal}} = 4C_\infty^\pi$.
The preceding guarantee therefore yields
\begin{equation}
    \DTV\left(P^R_{N,4\exp(-n/(16C_\infty^\pi))}, \pi\right)
    \leq 4\exp\left(-\frac{n}{16C_\infty^\pi}\right)
    \leq 4\exp\left(-\frac{N - 2}{32C_\infty^\pi}\right).
    \label{eq:budget-rs-bounded-upper}
\end{equation}
The final inequality uses $n \geq (N - 2)/2$; for odd $N$, replace $N - 2$ by $N - 1$.
The condition $N > 32C_\infty^\pi\log 4 + 2$ is sufficient for the strict condition on $n$.
This proves part~\textup{(ii)}.

\paragraph{Refinement for Table~\ref{tab:comparison}.}
Return to $\delta_{\rm cal} = 1/(n + 1)$ and assume $C_\infty^\pi < \infty$.
Combining the Chebyshev bound with~\eqref{eq:bc-bounded-pilot} gives
\begin{equation}
    \begin{split}
        \DTV(P^R_{N,1/(n + 1)}, \pi)
        &\leq \min\left\{\frac{16(C^\pi - 1)}{N}, 2\exp\left(-\frac{N - 2}{24C_\infty^\pi}\right)\right\} \\
        &\quad + 2E_{m_N} + \exp\left(-\frac{N}{12m_N}\right).
    \end{split}
    \label{eq:bc-hybrid-pilot}
\end{equation}
If $N - 2 \geq 64C_\infty^\pi L_N$, the pilot term is at most $2\exp(-8L_N/3) = 2\exp(-N/(12m_N))$.
Hence it can be absorbed into the last exponential term.

The simpler condition $N \geq 512C_\infty^\pi\log(eC_\infty^\pi)$ suffices.
To verify this, set $u := \log(eC_\infty^\pi) \geq 1$ and $N_0 := 512C_\infty^\pi u$.
For $x \geq N_0$, the function $x - 2 - 64C_\infty^\pi\log(4(x + 1))$ is increasing.
Also, $N_0 + 1 \leq 513C_\infty^\pi u$, $\log u \leq u - 1$, and $\log 2052 < 31/4$ give
\[
    N_0 - 2 - 64C_\infty^\pi\log(4(N_0 + 1))
    \geq C_\infty^\pi(384u - 368) - 2
    > 0.
\]
Since $n + 1 \leq N + 1$, this implies $N - 2 \ge 64C_\infty^\pi L_N$.
For smaller budgets, retain the term w.r.t. $C^\pi$.
The two cases then yield:
\begin{equation}
    \DTV(P^R_{N,1/(n + 1)},\pi)
    \leq \frac{16(C^\pi - 1)}{N}\1\left\{N < 512C_\infty^\pi\log(eC_\infty^\pi)\right\} + 3F(m_N, N/12).
    \label{eq:bc-indicator-pilot}
\end{equation}
This recovers the calibrated-RS bound summarized in Table~\ref{tab:comparison}.
Here $C_\infty^\pi$ is used only in this analysis;
the input remains $\delta_{\mathrm{cal}} = 1/(n + 1)$.
The indicator records absorption of pilot failure into the conditional sampling bound, not a zero failure probability.
\end{proof}

\subsection{Proof of Proposition~\ref{prop:budget-rs-separation}: exponential separation on a fixed pair}
\label{app:budget-rs-binary}

We prove Proposition~\ref{prop:budget-rs-separation} for the fixed pair $\mu = (1/2, 1/2)$ and $\pi = (1/4, 3/4)$.
We then compare the budgets needed for a prescribed accuracy and record the asymptotic errors of the methods on this fixed pair.
Throughout the fixed pair proof, $N = 2n + 1$ with $n \geq 1$.
\begin{proof}
First, notice that the normalized weights are $w(0) = \tfrac{\pi(0)}{\mu(0)} = 1/2$ and $w(1) = \tfrac{\pi(1)}{\mu(1)} = 3/2$.
Fix a calibration and set $M := M_{N, \delta_{\rm cal}}$.
It suffices to show that every multiplier $M > 0$ leaves TV error at least $2^{-n}/12$.
We first record the output law of one accepted proposal at a fixed normalized threshold $\tau$.
The acceptance probabilities of states $0$ and $1$ are
\[
    a_0(\tau) = \min\left\{1, \frac{1}{2\tau}\right\}, \qquad
    a_1(\tau) = \min\left\{1, \frac{3}{2\tau}\right\}.
\]
Since $\mu$ is uniform, conditional on acceptance, the probability of state $1$ is
\[
    h(\tau) := Q_\tau^{\mathrm{RS}}(\{1\})
    =\frac{a_1(\tau)}{a_0(\tau)+a_1(\tau)}
    =
    \begin{cases}
        1/2, & 0 < \tau \leq 1/2, \\
        \tau/(\tau + 1/2), & 1/2 < \tau < 3/2, \\
        3/4, & \tau \geq 3/2.
    \end{cases}
\]
Hence,
\[
    \frac{1}{2} \leq h(\tau) \leq \frac{3}{4} = \pi(1), \quad
    \text{for every}\ \tau > 0.
\]
The fresh fallback is drawn from $\mu$, so it assigns state $1$ probability $1/2$ as well.
Hence, for every realized pilot threshold, both the accepted law and the fallback underweight state $1$ relative to its target probability $3/4$.
In particular, the errors arising from different pilot outcomes all have the same sign and therefore cannot cancel when we average over the pilot.

Let $r_T$ denote the rejection probability of one proposal at the random threshold $T$.
Conditional on $T$, at least one of the $n$ rejection stage proposals is accepted with probability $1 - r_T^n$, while all $n$ proposals are rejected with probability $r_T^n$, in which case the fresh fallback is used.
Therefore
\[
    \Pr(\widehat{Y}_R = 1 \mid T) = (1 - r_T^n)h(T) + \frac{1}{2} r_T^n.
\]
Since this quantity is always at most $3/4$, the binary TV error has no absolute value cancellation, and~\eqref{eq:budget-rs-output-law} gives
\begin{align}
    \DTV(P_{N, \delta_{\rm cal}}^R, \pi) &= \E\left[\frac{3}{4} - (1 - r_T^n)h(T) -\frac{1}{2} r_T^n\right] \notag\\
    &= \E\left[(1 - r_T^n)\left(\frac{3}{4} - h(T)\right) +\frac14 r_T^n\right].
    \label{eq:budget-rs-binary-error}
\end{align}
We now lower-bound this expression for every $M>0$.
If $M \leq 2$, consider the event that all $n$ pilot weights equal the lower weight $1/2$. This event has probability $2^{-n}$ and gives
\[
    \overline{W}_n = \frac{1}{2}, \qquad
    T = M\overline{W}_n = \frac{M}{2} \leq 1.
\]
Since $h(\tau)$ is nondecreasing,
\[
    h(T)\leq h(1) = \frac{1}{1+1/2} = \frac{2}{3}.
\]
On this event, both the accepted law and the fallback assign state $1$ probability at most $2/3$, so the conditional TV error is at least
\[
    \frac{3}{4} - \frac{2}{3} = \frac{1}{12}.
\]
Since all conditional errors have the same sign, we obtain
\[
    \DTV(P_{N, \delta_{\rm cal}}^R, \pi) \geq 2^{-n} \cdot \frac{1}{12}.
\]
Now consider the case where $M > 2$.
Since each pilot weight equals $1/2$ or $3/2$ with equal probability, $\overline{W}_n - 1$ is symmetric about zero.
Hence
\(
    \Pr(\overline{W}_n \geq 1) \geq \tfrac{1}{2}.
\)
On this event,
\[
    T = M\overline{W}_n \geq M > 2 > \frac{3}{2},
\]
so there is no clipping and hence $h(T) = 3/4$.
Moreover, since $\E_\mu[w(Y)]=1$, the acceptance probability of one proposal is $1/T$, and therefore
\[
    r_T = 1 - \frac{1}{T} \geq \frac{1}{2}, \qquad
    r_T^n \geq 2^{-n}.
\]
The clipping term in~\eqref{eq:budget-rs-binary-error} vanishes, while the fallback term contributes at least
\(
    \tfrac{1}{4} r_T^n \geq \tfrac{1}{4}2^{-n}.
\)
Since the event $\{\overline{W}_n \geq 1\}$ has probability at least $1/2$,
\[
    \DTV(P_{N, \delta_{\rm cal}}^R, \pi) \geq \frac{1}{2}\cdot \frac{1}{4}2^{-n}
    = \frac{1}{8}2^{-n}
    \geq \frac{1}{12}2^{-n}.
\]
Hence the same lower bound holds after taking the infimum over $\delta_{\rm cal} \in (0,1)$.

The UR formula in Proposition~\ref{prop:binary-sir-separation} gives
\(
    \DTV(P_k^U, \pi) = \frac{3}{4}3^{-k},
\)
which yields the claimed equality at $k = N = 2n + 1$.
This completes the proof.
\end{proof}

\paragraph{The gap persists with fewer UR proposals.}
Every baseline run uses at least $n+1$ proposals in total, including its pilot.
Even with only this many proposals, UR satisfies
\begin{equation}
    \frac{\inf_{\delta_{\rm cal} \in (0, 1)}\DTV(P_{2n + 1, \delta_{\rm cal}}^R, \pi)}
    {\DTV(P_{n + 1}^U, \pi)}
    \geq \frac{1}{3}\left(\frac{3}{2}\right)^n.
    \label{eq:budget-rs-smaller-ur-budget}
\end{equation}
With the same total budget $N = 2n + 1$ for both methods, the lower bound on the ratio is $\frac{1}{3}(9/2)^n$, as stated in Section~\ref{sec:rohatgi}.
The different bases correspond to different UR budgets, not different target--proposal pairs.

\paragraph{Error ratio versus sample complexity.}
An exponentially growing error ratio does not imply different orders of optimized sample complexity.
Let $\xi \in (0,1)$ be a TV error tolerance, with $\xi \downarrow 0$.
For UR, the exact formula gives a required budget of order $\log(1/\xi)$.
For calibrated RS, the lower bound requires $n$ to be at least of this order.
Conversely, $C_\infty^\pi = 3/2$, and choosing $\delta_{\rm cal} = \xi$ in~\eqref{eq:budget-rs-source-condition} gives TV error at most $\xi$ whenever the integer $n$ satisfies
\[
    n \geq 24\log(4/\xi).
\]
Hence both methods require $\Theta(\log(1/\xi))$ proposals when tuned for the prescribed accuracy.
This statement allows the calibration to vary with the requested accuracy;
it does not assert convergence for every fixed $\delta_{\rm cal}$.
Appendix~\ref{app:pilot-rs-local} addresses the separate question of matching the sharp guarantee uniformly over pairs at a fixed budget.

\paragraph{Further comparisons on the fixed pair.}
We note that UR has exact error $(3/4)3^{-N}$, and known-envelope RS has the same law by Proposition~\ref{prop:two-level-rs}, with a uniform observed-index fallback.
Substituting $\varepsilon = 1/4$ into \eqref{eq:binary-sir-numeric-rate} gives
\[
    \DTV(P_N^{\mathrm{SIR}}, \pi) = \frac{3}{16N} + O(N^{-3/2}).
\]
For calibrated RS, \eqref{eq:budget-rs-separation} gives
\[
    \inf_{\delta_{\rm cal} \in (0,1)}\DTV(P^R_{N,\delta_{\rm cal}}, \pi)
    \geq \frac{1}{12}2^{-(N-1)/2}
    = \frac{\sqrt{2}}{12}2^{-N/2}
\]
for every odd $N \geq 3$.
This proves the claimed $\Omega(2^{-N/2})$ lower bound; no matching $\Theta(2^{-N/2})$ rate is asserted.
Appendix~\ref{app:pilot-rs-local} addresses whether calibrated RS can match the universal bounded weight guarantee within a pair-independent factor at a fixed budget.

\subsection{Local separation under optimal calibration}
\label{app:pilot-rs-local}

The fixed pair comparison in Section~\ref{sec:rohatgi} distinguishes exponential rates as the budget grows.
Here we fix the budget and vary the pair to ask whether calibrated RS can match the sharp bounded weight guarantee within a factor independent of the pair.
Keep the unnormalized weights of states $0$ and $1$ equal to $1/2$ and $1$, respectively, and set
\begin{equation}
    \mu_\eta := (\eta,1-\eta), \qquad
    \pi_\eta := \left(\frac{\eta}{2-\eta}, \frac{2(1-\eta)}{2-\eta}\right), \qquad 
    0 < \eta < 1.
    \label{eq:pilot-rs-local-family}
\end{equation}
Hence $\eta$ is the proposal probability of the lower weight state.
The normalizer is $Z_\eta = 1 - \eta/2$, and $C_\infty^{\pi_\eta} = 1/Z_\eta$.
In particular, the bounded weight benchmark is $(1 - 1/C_\infty^{\pi_\eta})^N = (\eta/2)^N$.

\begin{proposition}[Local separation under optimal calibration]
\label{prop:pilot-rs-local}
Fix an odd $N \geq 13$.
For the family in \eqref{eq:pilot-rs-local-family}, as $\eta \downarrow 0$,
\begin{equation}
    \inf_{\delta_{\rm cal} \in (0,1)} \DTV(P^R_{N, \delta_{\rm cal}}, \pi_\eta) = \frac{N-1}{4(N-2)} \eta^2 + o_N(\eta^2), 
    \label{eq:pilot-rs-local-asymptotics}
\end{equation}
whereas
\begin{equation}
    \DTV(P^U_N, \pi_\eta) = \frac{1-\eta}{1-\eta/2} \left(\frac{\eta}{2}\right)^N.
    \label{eq:pilot-rs-local-ur}
\end{equation}
The infimum allows pair- and budget-dependent tuning before observing the pilot.
The UR identity holds for every $N \geq 1$ and $0 < \eta < 1$.
\end{proposition}

\begin{proof}
Fix $n \geq 6$ and $N = 2n + 1$, and let $\eta \downarrow 0$.
For the lower bound, allow any deterministic multiplier $M > 0$ chosen before observing the pilot.
Let $J$ count the state $0$ candidates among the $n$ pilot proposals.
Then $J \sim \Bin(n,\eta)$, and
\[
    \widehat{Z} = \frac{J/2 + n - J}{n}
    = 1 - \frac{J}{2n}, \qquad
    \tau_J := M\widehat{Z}
    = M\left(1 - \frac{J}{2n}\right).
\]
Here $\tau_J$ is the raw threshold; its normalized counterpart is $T = \tau_J/Z_\eta$.
Allowing all $M > 0$ enlarges the calibrated family.
The upper bound choice $M = 1$ corresponds to $\delta_{\rm cal} = 4\exp(-n/4)$.
This input lies in $(0,1)$ exactly when $n > 4\log 4$, or $n \geq 6$ for integer $n$.
For odd $N = 2n + 1$, this is precisely $N \geq 13$.

\paragraph{The error conditional on the pilot.}
The accepted probability of state $0$ is
\[
    q_\eta(\tau)
    =
    \begin{cases}
    \eta, & 0< \tau \leq 1/2,\\
    \frac{\eta/2}{\eta/2 + (1 - \eta)\tau}, & 1/2 < \tau < 1, \\
    \frac{\eta}{2 - \eta}, & \tau \geq 1.
    \end{cases}
\]
It is nonincreasing in $\tau$ and lies between $\pi_\eta(0)$ and $\mu_\eta(0) = \eta$.
Write
\[
    d_\eta(\tau) := q_\eta(\tau) - \pi_\eta(0), \qquad
    g_\eta := \eta - \pi_\eta(0)
    = \frac{\eta(1 - \eta)}{2 - \eta}.
\]
Let $r_\eta(\tau) := r_{\tau/Z_\eta}$ be the rejection probability of one fresh trial; the division by $Z_\eta$ converts the raw threshold to the normalized units of Appendix~\ref{app:budget-rs-law}.

Both acceptance and fallback assign state $0$ at least its target probability, so averaging over the pilot cannot cancel this excess.
Consequently,~\eqref{eq:budget-rs-output-law} gives the exact TV error of the relaxed rule:
\begin{equation}
    e_n(M, \eta) = \E\left[d_\eta(\tau_J) + r_\eta(\tau_J)^n\left(g_\eta-d_\eta(\tau_J)\right)\right]
    = \E\left[(1 - r_\eta(\tau_J)^n)d_\eta(\tau_J)
    + r_\eta(\tau_J)^n g_\eta\right].
    \label{eq:pilot-local-error}
\end{equation}
The expectation is over $J$.
The integrand lies in $[d_\eta(\tau_J), g_\eta]$, a fact used for both bounds below.


\paragraph{Step 1: Upper bound from $M = 1$.}
For $J = 0$, the threshold is one, the accepted law is exact, and $r_\eta(1) = \eta/2$.
The contribution is at most $g_\eta(\eta/2)^n = O_n(\eta^{n + 1})$.
For $J = 1$, put $v_n := 1 - 1/(2n) \in (1/2, 1)$, where we use $n \geq 2$.
Direct expansion gives
\[
    d_\eta(v_n)
    = \frac{1}{2}\left(\frac{1}{v_n} - 1\right)\eta + O_n(\eta^2)
    = \frac{\eta}{2(2n - 1)} + O_n(\eta^2).
\]
Here $r_\eta(v_n) = \eta(1 - 1/(2v_n)) = O_n(\eta)$ and $\Pr(J = 1) = n\eta + O_n(\eta^2)$.
The contribution is consequently $n\eta^2/[2(2n - 1)] + O_n(\eta^3)$.
Finally, $\Pr(J \geq 2) = O_n(\eta^2)$ and each conditional error is at most $g_\eta = O(\eta)$.
Hence
\begin{equation}
    e_n(1, \eta) = \frac{n}{2(2n - 1)}\eta^2 + O_n(\eta^3).
    \label{eq:pilot-local-upper}
\end{equation}

\paragraph{Step 2: Typical pilot forces $M_\eta\to1$.}
Suppose $e_n(M_\eta, \eta) = O_n(\eta^2)$, and fix $\gamma \in (0,1/2)$.
If $M_\eta \leq 1 - \gamma$ along a subsequence, then the event $J = 0$ and the monotonicity of $d_\eta$ give
\[
    e_n(M_\eta, \eta)
    \geq (1 - \eta)^n d_\eta(1 - \gamma), \qquad
    \frac{d_\eta(1 - \gamma)}{\eta} \rightarrow
    \frac{\gamma}{2(1 - \gamma)}>0.
\]
This contradicts the assumed second-order error.
If instead $M_\eta \geq 1 + \gamma$ along a subsequence, then on $J = 0$ there is no clipping, but $r_\eta(M_\eta) = 1 - Z_\eta/M_\eta\ge\gamma/(1 + \gamma)$.
Hence
\[
    e_n(M_\eta, \eta) \geq
    (1 - \eta)^n\left(\frac{\gamma}{1 + \gamma}\right)^n g_\eta.
\]
Since $g_\eta/\eta \to 1/2$, this again leaves first-order error.
Therefore $M_\eta \to 1$.

\paragraph{Step 3: The event $J=1$ forces the leading coefficient.}
For each $\eta$, choose $M_\eta>0$ such that
\[
    e_n(M_\eta, \eta) \leq
    \inf_{M>0}e_n(M, \eta) + \eta^3.
\]
Step 1 makes this sequence second-order accurate, so Step 2 gives $M_\eta \to 1$.
Hence $\tau_1 = M_\eta v_n \to v_n$, and
\[
    \frac{d_\eta(M_\eta v_n)}{\eta} \rightarrow
    \frac{1}{2(2n - 1)}.
\]
The nonnegative rejection term in~\eqref{eq:pilot-local-error} can only increase the error.
Keeping just the $J=1$ contribution therefore gives
\[
    \inf_{M>0}e_n(M, \eta) \geq \Pr(J = 1)d_\eta(M_\eta v_n) - \eta^3.
\]
Dividing by $\eta^2$ and using $\Pr(J=1)/\eta\to n$ proves
\[
    \liminf_{\eta \downarrow 0} \frac{\inf_{M>0}e_n(M, \eta)}{\eta^2} \geq \frac{n}{2(2n - 1)}.
\]
Together with the feasible upper bound in~\eqref{eq:pilot-local-upper}, this proves~\eqref{eq:pilot-rs-local-asymptotics}, since $n/[2(2n - 1)] = (N - 1)/[4(N - 2)]$.
In particular, the multiplier may depend on the pair and budget, but not on the realized pilot.

\paragraph{UR's exact error.}
The largest normalized weight is $C_\infty^{\pi_\eta} = 1/Z_\eta$.
A state $1$ candidate crosses this score level almost surely.
A state $0$ candidate stays below it exactly when its uniform exceeds $1/2$, so a single candidate fails to cross with probability $\eta/2$.
Independence gives $\Pr(H_{C_\infty^{\pi_\eta}}^c) = (\eta/2)^N$.
On this event all candidates are in state $0$, so UR returns $0$; on the complement its output law is $\pi_\eta$ by \eqref{eq:exact-threshold}.
The mixture identity in~\eqref{eq:bounded-mixture} therefore gives
\[
    \DTV(P_N^U, \pi_\eta) = \pi_\eta(1)\left(\frac{\eta}{2}\right)^N
    = \frac{1- \eta}{1- \eta/2}\left(\frac{\eta}{2}\right)^N,
\]
which proves~\eqref{eq:pilot-rs-local-ur}.
This identity holds for every $N \geq 1$ and $0 < \eta < 1$, not only in the local limit.
\end{proof}

For SIR, at fixed $N$, exactly one state $0$ candidate appears with probability $N\eta + O_N(\eta^2)$, and SIR selects it with probability $1/(2N - 1)$.
Batches containing at least two such candidates have total probability $O_N(\eta^2)$.
Hence $P_N^{\mathrm{SIR}}(\{0\})=N\eta/(2N-1)+O_N(\eta^2)$, and subtracting $\pi_\eta(0)=\eta/2+O(\eta^2)$ yields
\[
    \DTV(P_N^{\mathrm{SIR}}, \pi_\eta) = \frac{\eta}{2(2N-1)} + O_N(\eta^2).
\]
Known-envelope RS uses the exact raw envelope $Z_\eta C_\infty^{\pi_\eta} = 1$ and a uniform observed-index fallback.
Its output law equals UR's on this two-weight family by Proposition~\ref{prop:two-level-rs}.
Note that UR satisfies the universal bounded weight bound by Corollary~\ref{cor:rates}, and known-envelope RS does so by its exact accepted law and complete-rejection probability.
On this family the bound is $(1 - 1/C_\infty^{\pi_\eta})^N = (\eta/2)^N$, so the positive first- and second-order coefficients above rule it out for SIR and calibrated RS, even up to a pair-independent multiplicative constant at fixed odd $N \geq 13$.
Even UR with only $n + 1$ proposals has error $\Theta_n(\eta^{n + 1})$ on this family.
All local limits fix the budget; no remainder uniform in $N$ is asserted.

\section{Moment bounds and minimax rates}
\label{app:rates}
The bounded weight upper bound in Corollary~\ref{cor:rates} follows from~\eqref{eq:exact-threshold}.
We prove the second moment bound next, followed by the matching lower bound under a uniform weight bound.
The final subsection gives a lower bound under a second moment constraint.
These lower bounds use the same restriction: a sampler cannot return a state that never appears among its proposals.

\subsection{Direct second moment upper bound}
Assume $C^\pi = \E_\mu[w(Y)^2] < \infty$.
By Proposition~\ref{prop:score-law}, conditioning on $S_{(N)} = s$ excludes exactly the target mass $\pi(w \geq s)$.
We first average this conditional error, then bound the chance that the largest score is too small.
Let $X \sim \pi$ be independent of the entire race, used only for analysis.
Recall $G(t) := \Pr(w(Y)/U \leq t)$ for independent $Y \sim \mu$ and $U \sim \Unif(0, 1)$.
Then
\begin{align}
    \DTV(P_N^U, \pi)
    &\stackrel{\mathrm{(i)}}{\leq} \E[\pi(w \geq S_{(N)})] \notag \\
    &\stackrel{\mathrm{(ii)}}{=} \Pr(S_{(N)} \leq w(X)) \notag \\
    &\stackrel{\mathrm{(iii)}}{=} \E[G(w(X))^N],
    \label{eq:app-tail-mixture}
\end{align}
where (i) follows from convexity of TV and the conditional law;
(ii) follows from using the independent target point $X \sim \pi$;
and (iii) conditions on $X$ and uses independence of the $N$ scores.
Under the zero weight convention of Appendix~\ref{app:zero-weights}, the conditional error at score zero is one, equal to $\pi(w \geq 0)$, so the same calculation applies.
We now bound the chance that a single score exceeds $t > 0$:
\begin{align*}
    1 - G(t)
    &= \E_\mu\left[\min\left\{\frac{w(Y)}{t}, 1\right\}\right] \\
    &\stackrel{\mathrm{(i)}}{\geq} \E_\mu\left[\frac{w(Y)}{t + w(Y)}\right] \\
    &\stackrel{\mathrm{(ii)}}{=} \E\left[\frac{1}{t + w(X)}\right] \\
    &\stackrel{\mathrm{(iii)}}{\geq} \frac{1}{t + \E[w(X)]} = \frac{1}{t + C^\pi},
\end{align*}
where (i) follows from the fact that $\min\{\tfrac{a}{t}, 1\} \geq \tfrac{a}{t + a}$ for $a \geq 0$;
(ii) follows since $w \dd \mu = \dd \pi$;
and (iii) follows from Jensen's inequality for the convex function $a \mapsto \tfrac{1}{t + a}$.
This is where the second moment enters: $\E[w(X)] = \E_\mu[w(Y)^2] = C^\pi$.
Independence and $1 - a \leq e^{-a}$ give $G(t)^N \leq e^{-N/(t + C^\pi)}$.
Substituting this into~\eqref{eq:app-tail-mixture},
\begin{align}
    \DTV(P_N^U, \pi)
    & \leq \E\left[\exp\left(-\frac{N}{w(X) + C^\pi}\right)\right] \label{eq:app-moment-exponential} \\
    & \leq \frac{\E[w(X)] + C^\pi}{eN} = \frac{2C^\pi}{eN},
    \label{eq:app-moment-final}
\end{align}
where the last inequality follows from the fact that $e^{-x} \leq 1/(ex)$ for $x > 0$.
Together with the trivial bound $1$, this proves the second moment term in Corollary~\ref{cor:rates}.


\paragraph{Relation to the known SIR bound.}
The order $C^\pi/N$ is not specific to UR.
\citet[Theorem~2.1]{agapiou2017} bound the absolute bias of self-normalized importance sampling by $12C^\pi/N$ for bounded measurable functions (bounded by one in absolute value).
For an event $D$, apply that bound to $2\1_D - 1$.
The bias is then $2(P_N^{\mathrm{SIR}}(D) - \pi(D))$, giving $\DTV(P_N^{\mathrm{SIR}}, \pi) \leq \min\{1, 6C^\pi/N\}$.
Theorem~\ref{thm:SIR} also transfers that rate to UR.
The direct proof above gives the constant $2/e$ without using the SIR comparison, but does not establish optimality of that constant.

\subsection{Exact bounded ratio minimax error}
UR already supplies the upper bound in~\eqref{eq:minimax}.
For a lower bound, consider a binary pair with target mass one on state $1$ and proposal probability $1/C$ on that state, where $C > 1$.
Every observed-candidate sampler must return state $0$ when all $N$ proposals miss state $1$.
This is the same basic obstruction used in approximate sampling lower bounds~\citep{block2023}.
To obtain the bound using strictly positive distributions, fix $C > 1$ and $0 < \eta < 1$ and take
\[
    \pi_\eta(1) = 1 - \eta, \quad \mu_\eta(1) = \frac{1 - \eta}{C}.
\]
Then $w_\eta(1) = C$ and $w_\eta(0) = \eta/[1 - (1 - \eta)/C] < C$.
For any sampler $\mathsf{A}$ returning an observed candidate,
\begin{align*}
    \DTV(P_N^{\mathsf{A}}, \pi_\eta)
    &\stackrel{\mathrm{(i)}}{\geq} P_N^{\mathsf{A}}(\{0\}) - \eta \\
    &\stackrel{\mathrm{(ii)}}{\geq} \left(1 - \frac{1 - \eta}{C}\right)^N - \eta.
\end{align*}
Here (i) evaluates the TV supremum on $\{0\}$; and (ii) uses the event that every proposal is state $0$, on which the output must also be state $0$.
Letting $\eta \downarrow 0$ gives $(1 - 1/C)^N$, uniformly over the choice of sampler, and proves the minimax lower bound.
The argument applies even to samplers with additional information about the pair.
For $C = 1$, $w \leq 1$ and $\E_\mu w = 1$ imply $w = 1$ almost surely.
Thus $\pi = \mu$ and returning the first proposal is exact.

\subsection{Second moment minimax lower bound}
A second moment constraint still allows a state to be much rarer under the proposal than under the target.
We choose its proposal probability so that the chance of observing it is at most half its target probability.
\begin{proposition}[A lower bound under a second moment constraint]
\label{prop:second-minimax}
Fix $R > 1$ and $N \geq 1$.
For every procedure returning one of the $N$ observed proposals, there is a finite strictly positive pair with $C^\pi \leq R$ and TV error at least
\[
    \frac{1}{16} \min\left\{1, \frac{R - 1}{N}\right\}.
\]
For $R \geq 2$, this is at least $\tfrac{1}{32} \min\{1, R/N\}$.
\end{proposition}
\begin{proof}
Set
\[
    p := \min\left\{\frac{R - 1}{8N}, \frac{1}{4}\right\}, \quad
    q := \frac{p}{2N}, \quad \pi(1) := p, \quad
    \mu(1) := q.
\]
Both binary distributions are strictly positive.
Since the output must be observed,
\[
    \Pr(\widehat{Y} = 1) \leq \Pr(\exists i : Y_i = 1) \leq Nq = p/2,
\]
where the second inequality is the union bound.
The TV error is therefore at least
\[
    p - \Pr(\widehat{Y} = 1) \geq p/2 \geq \frac{1}{16} \min\left\{1, \frac{R - 1}{N}\right\}.
\]
It remains to check the moment constraint.
For a binary pair, direct substitution gives
\begin{align*}
    C^\pi - 1 = \frac{(p - q)^2}{q(1 - q)}
    \stackrel{\mathrm{(i)}}{\leq} \frac{2Np}{1 - q}
    \stackrel{\mathrm{(ii)}}{\leq} \frac{16}{7} Np \stackrel{\mathrm{(iii)}}{\leq} \frac{2}{7} (R - 1) \leq R - 1.
\end{align*}
Here (i) uses $(p - q)^2 \leq p^2$ and $q = p/(2N)$;
(ii) and (iii) use $q \leq 1/8$ and $Np \leq (R - 1)/8$.
Hence $C^\pi \leq R$.
For $R \geq 2$, $R - 1 \geq R/2$ gives the second lower bound.
\end{proof}
For $R \geq 2$, the upper and lower bounds have the same dependence on $R$ and $N$, up to constants.
This is a minimax statement over a moment-constrained class, not a lower bound for every fixed pair.
\section{Proofs for the comparison to SIR}
\label{app:comparison}
The proof of Theorem~\ref{thm:SIR} uses two facts.
UR returns weights above any threshold at least as often as SIR (Lemma~\ref{lem:selected-weights}), and UR's density relative to $\pi$ is nonincreasing in the weight (Lemma~\ref{lem:ur-density}).
We first combine these facts to prove the theorem, then establish them directly from the selection rule.

\subsection{Proof for convex divergences}
Let $h_N^U$ and $h_N^{\mathrm{SIR}}$ be the scalar functions in Lemma~\ref{lem:ur-density}, so that
\[
    \frac{\dd P_N^U}{\dd \pi}(y) = h_N^U(w(y)), \quad \frac{\dd P_N^{\mathrm{SIR}}}{\dd \pi}(y) = h_N^{\mathrm{SIR}}(w(y)).
\]
For $X \sim \pi$, both $h_N^U(w(X))$ and $h_N^{\mathrm{SIR}}(w(X))$ are positive, finite, and have mean one.
We compare the amount by which each exceeds a level $c > 0$.
We note that in order to compare all convex $f$-divergences, it suffices to compare positive part expectations of the two density ratios at every level $c > 0$, together with their (equal) means.
These positive part functions measure how much of the density ratio lies above a chosen level; their comparison extends from piecewise linear convex functions to general convex $f$.

To this end, define $D_c := \{y: h_N^U(w(y)) > c\}$.
Since $h_N^U$ is nonincreasing, $D_c$ has the form $\{w < t\}$ or $\{w \leq t\}$ for some threshold $t$, or is an empty or full set.
The threshold comparison in Lemma~\ref{lem:selected-weights}, with complements when necessary, gives $P_N^U(D_c) \leq P_N^{\mathrm{SIR}}(D_c)$.
Consequently,
\begin{align}
    \E[(h_N^U(w(X)) - c)_+]
    &\stackrel{\mathrm{(i)}}{=} P_N^U(D_c) - c \pi(D_c) \label{eq:app-hinge-identity} \\
    &\stackrel{\mathrm{(ii)}}{\leq} P_N^{\mathrm{SIR}}(D_c) - c \pi(D_c) \label{eq:app-hinge-comparison} \\
    &\stackrel{\mathrm{(iii)}}{\leq} \int (h_N^{\mathrm{SIR}}(w(y)) - c)_+ \pi(\dd y) \notag \\
    &= \E[(h_N^{\mathrm{SIR}}(w(X)) - c)_+].
    \label{eq:app-hinge-order}
\end{align}
Here (i) integrates the positive part over its positive set $D_c$.
(ii) is the probability comparison on $D_c$.
For (iii), the integral of $h_N^{\mathrm{SIR}}(w(y)) - c$ over $D_c$ is at most the integral of its positive part over the whole space.
At $c = 1$, the two expectations are exactly the TV errors, proving the desired result in TV distance.
To pass to a convex $f$-divergence, first consider a convex piecewise linear function with finitely many breakpoints.
It can be written as
\[
    g(t) = a + bt + \sum_{k = 1}^m c_k (t - t_k)_+, \quad c_k \geq 0, \quad t_k > 0.
\]
Each $c_k$ is the increase in slope at a breakpoint.
Because both density ratios have mean one, the affine terms have equal expectations.
Applying~\eqref{eq:app-hinge-order} to each remaining term gives
\[
    \E[g(h_N^U(w(X)))] \leq \E[g(h_N^{\mathrm{SIR}}(w(X)))].
\]
Now let $f: (0, \infty) \to \R$ be convex with $f(1) = 0$.
Choose the slope $b$ of a supporting line at $1$ and put $\widetilde{f}(t) := f(t) - b(t - 1) \geq 0$.
Choose a countable dense set of points in $(0, \infty)$ and a supporting affine function of $\widetilde{f}$ at each point.
Let $g_m$ be the maximum of zero and the first $m$ such functions.
Then $g_m$ is nonnegative and convex piecewise linear, and $g_m \uparrow \widetilde{f}$ pointwise.
For completeness, supporting slopes are bounded on every compact subinterval of $(0, \infty)$ by secant slopes to nearby endpoints.
At points of the dense set approaching $t$, continuity and this bound make the supporting lines approach $\widetilde{f}(t)$, proving the claimed limit.
Monotone convergence yields
\[
    \E[\widetilde{f}(h_N^U(w(X)))] \leq \E[\widetilde{f}(h_N^{\mathrm{SIR}}(w(X)))],
\]
including infinite values.
Adding back the affine part changes neither side because both means are one.
This proves~\eqref{eq:fdom}.
The choices $f(t) = |t - 1|/2$, $t \log t$, $-\log t$, and $(t - 1)^2$ give TV, forward KL, reverse KL, and $\chi^2$ divergence, respectively.

\paragraph{The target also bounds the selected weight probabilities.}
We additionally obtain the interpretation used in Section~\ref{sec:sir}:
\begin{equation}
    \Pr(w(\widehat{Y}_{\mathrm{SIR}}) > t) \leq \Pr(w(\widehat{Y}_U) > t) \leq \pi(w > t), \quad t > 0.
    \label{eq:bracket}
\end{equation}
Lemma~\ref{lem:selected-weights} gives the first inequality.
For the second, take $A := \{y: w(y) > t\}$ and use that the density integrates to one:
\begin{align*}
    P_N^U(A) - \pi(A)
    &= P_N^U(A) \pi(A^c) - \pi(A) P_N^U(A^c) \\
    &= \int_A \int_{A^c} [h_N^U(w(y)) - h_N^U(w(z))] \pi(\dd z) \pi(\dd y) \\
    & \leq 0.
\end{align*}
The last inequality holds because $w(y) > w(z)$ on $A \times A^c$ and $h_N^U$ is nonincreasing.
The same proof works for $\{w \geq t\}$.
No moment assumption beyond $\E_\mu w = 1$ is used.

\subsection{Two supporting lemmas}
\begin{lemma}[Probability of returning a weight above a threshold]
\label{lem:selected-weights}
For every $t > 0$,
\begin{equation}
    \Pr(w(\widehat{Y}_{\mathrm{SIR}}) > t) \leq \Pr(w(\widehat{Y}_U) > t).
    \label{eq:stochastic-lower}
\end{equation}
The same inequality holds for weights at least $t$.
Taking complements reverses the inequalities for weights below a threshold.
\end{lemma}
\begin{proof}
For $N = 1$, the claim is immediate.
For $N \geq 2$, fix a positive vector $w_{1:N}$ and write $z_i := w_i/w_{\max}$, where $w_{\max} := \max_j w_j$.
Let $p_i^U$ and $p_i^{\mathrm{SIR}}$ be the selection probabilities on this fixed vector.
We first compare candidates, then sum over those above $t$.
A maximal weight candidate has score strictly above $w_{\max}$.
Candidate $i$ can therefore be selected only if $U_i < z_i$.
Write $U_i = z_i u$.
For $u \in (0, 1)$, its score is larger than candidate $j$'s exactly when $U_j > z_j u$.
Independence gives
\begin{equation}
    p_i^U = z_i \int_0^1 \prod_{j \neq i} (1 - z_j u) \dd u.
    \label{eq:batch-law}
\end{equation}
For distinct $i, k$ with $w_i \geq w_k$, put $A_{ik}(u) := \prod_{j \notin \{i, k\}} (1 - z_j u)$.
Subtracting the two formulas gives
\[
    \frac{p_i^U}{z_i} - \frac{p_k^U}{z_k} = (z_i - z_k) \int_0^1 u A_{ik}(u) \dd u \geq 0.
\]
Since $p_i^{\mathrm{SIR}} = z_i/\sum_j z_j$, this is equivalent to
\begin{equation}
    p_i^U p_k^{\mathrm{SIR}} - p_i^{\mathrm{SIR}} p_k^U \geq 0 \quad \text{whenever } w_i \geq w_k.
    \label{eq:sharpening}
\end{equation}
Thus the change from SIR to UR favors larger weights in pairwise comparisons.
Let $I_t := \{i: w_i > t\}$.
Using that both probability vectors sum to one,
\begin{align*}
    \sum_{i \in I_t} (p_i^U - p_i^{\mathrm{SIR}})
    &\stackrel{\mathrm{(i)}}{=} \sum_{i \in I_t} \sum_{k \notin I_t} (p_i^U p_k^{\mathrm{SIR}} - p_i^{\mathrm{SIR}} p_k^U) \\
    &\stackrel{\mathrm{(ii)}}{\geq} 0.
\end{align*}
In (i), the terms with both indices in $I_t$ cancel.
(ii) uses~\eqref{eq:sharpening}, since every weight in $I_t$ is larger than every weight outside it.
This includes empty and full sets $I_t$.
Averaging over the proposals proves~\eqref{eq:stochastic-lower}.
Replacing $I_t$ by $\{i: w_i \geq t\}$ proves the other endpoint convention.
\end{proof}

\begin{lemma}[Density of the returned sample]
\label{lem:ur-density}
Let $G(s) := \Pr(w(Y)/U \leq s)$ for independent $Y \sim \mu$ and $U \sim \Unif(0, 1)$.
Then
\begin{equation}
    \frac{\dd P_N^U}{\dd \pi}(y) = h_N^U(w(y)), \quad h_N^U(x) := N \int_x^\infty \frac{G(s)^{N - 1}}{s^2} \dd s, \quad x > 0.
    \label{eq:marginal-density}
\end{equation}
The function $h_N^U$ is nonincreasing.
\end{lemma}
\begin{proof}
The density formula is the integral in Lemma~\ref{lem:joint-score}.
Its integrand is nonnegative, so increasing the lower limit $x$ cannot increase the integral.
For comparison, the SIR density is
\begin{equation}
    \frac{\dd P_N^{\mathrm{SIR}}}{\dd \pi}(y) = h_N^{\mathrm{SIR}}(w(y)), \quad h_N^{\mathrm{SIR}}(x) := N \E\left[\frac{1}{x + \sum_{j = 2}^N W_j}\right],
    \label{eq:app-sir-density}
\end{equation}
where $x > 0$ and $W_j = w(Y_j)$ for independent proposals.
In particular, a specified candidate at $y$ contributes its SIR selection probability $w(y)/(w(y) + \sum_{j = 2}^N W_j)$ times $\mu(\dd y)$.
There are $N$ exchangeable candidate positions, and $w(y) \mu(\dd y) = \pi(\dd y)$, giving the formula.
Both scalar functions are positive and at most $N/x$ for $x > 0$, and their compositions with $w$ integrate to one under $\pi$.
At $N = 1$, both functions are $1/x$.
\end{proof}

\subsection{Equivalent Bernoulli selection rule}
\label{app:bernoulli}
UR has an equivalent rule on each fixed positive weight vector: retain candidates independently and select uniformly among those retained.
Write $z_j := w_j/\max_k w_k$, draw independent $B_j \sim \operatorname{Bern}(z_j)$, and retain $j$ when $B_j = 1$.
If there are $k$ retained indices, each is selected with probability $1/k$.
At least one $z_j = 1$, so the retained set is nonempty.
The denominator in $z_j$ is the largest observed weight, not a bound on unobserved weights.
Put $K_{-i} := \sum_{j \neq i} B_j$.
By independence, the probability of selecting $i$ is $z_i \E[(1 + K_{-i})^{-1}]$.
Using $1/(k + 1) = \int_0^1 x^k \dd x$,
\begin{align*}
    \E\left[\frac{1}{1 + K_{-i}}\right]
    &\stackrel{\mathrm{(i)}}{=} \int_0^1 \E[x^{K_{-i}}] \dd x \\
    &\stackrel{\mathrm{(ii)}}{=} \int_0^1 \prod_{j \neq i} (1 - z_j + z_j x) \dd x \\
    &\stackrel{\mathrm{(iii)}}{=} \int_0^1 \prod_{j \neq i} (1 - z_j t) \dd t.
\end{align*}
Here (i) interchanges a nonnegative integral and expectation.
(ii) uses independence of the Bernoulli variables, and (iii) substitutes $t = 1 - x$.
After multiplication by $z_i$, this is exactly~\eqref{eq:batch-law}.
The two implementations therefore have the same selection probabilities on every fixed batch, though they need not select the same index when coupled using the same uniforms.
\section{Binary examples and bias calculations}
\label{app:bias}

Section~\ref{sec:sir} gives the direct comparison for two candidates.
We first extend that binary calculation to every budget and identify the source of SIR's remaining error.
We then prove the equality with RS with $M = C_\infty^\pi$ for two weight values, derive the general SIR bias and UR's corresponding adjustment, and conclude with a continuous example.
Throughout the binary calculation,
\[
    \mu = \left(\frac{1}{2}, \frac{1}{2}\right), \qquad
    \pi = (\eps, 1 - \eps), \qquad
    0 < \eps < \frac{1}{2},
\]
with $w_L = 2\eps$ and $w_H = 2(1 - \eps)$.

\subsection{Proof of Proposition~\ref{prop:binary-sir-separation}: Binary example for every budget}
\label{app:binary-general}
Let $K \sim \Bin\left(N, \frac{1}{2}\right)$ count the candidates in state $1$, and put $X := K/N$.
Thus $X$ is the observed fraction of candidates in state $1$.
Define, for $x \in [0, 1]$,
\[
    d(x) := \eps + (1 - 2\eps)x, \qquad
    g(x) := \frac{(1 - \eps)x}{d(x)}.
\]
When $X = x$, the total observed weight is $2N d(x)$, and SIR returns state $1$ with probability $g(x)$.
The exact output probabilities are
\begin{align}
    \Pr(\widehat{Y}_{\mathrm{SIR}} = 1) &= \E[g(X)],
    \label{eq:binary-sir} \\
    \Pr(\widehat{Y}_U = 1) &= (1 - \eps)\left[1 - \left(\frac{1 - 2\eps}{2(1 - \eps)}\right)^N\right].
    \label{eq:binary-ur}
\end{align}
Although $g(\E[X]) = 1 - \eps$, averaging $g(X)$ does not give this target probability.
More precisely,
\begin{equation}
    \pi(1) - \Pr(\widehat{Y}_{\mathrm{SIR}} = 1) = 4\eps(1 - \eps)(1 - 2\eps)\E\left[\frac{\left(X - \frac{1}{2}\right)^2}{d(X)}\right].
    \label{eq:binary-exact-gap}
\end{equation}
On a binary state space, TV is the absolute error in the probability of returning state $1$.
Thus these formulae will give the claimed TV rates once we verify that both probabilities are below $1 - \eps$.
We first establish the output formulae by conditioning on $K$, then obtain the sign and size of SIR's error from the gap between $g$ and its tangent at $\E[X] = 1/2$.

\paragraph{The two output distributions.}
It suffices to compute each selection probability conditional on $K$, since averaging over this binomial count gives the marginal output probability.
For SIR, given $K = k$, the probability is
\[
    \frac{(1 - \eps)k}{(1 - \eps)k + \eps(N - k)}
    = g\left(\frac{k}{N}\right).
\]
Averaging over $K$ proves~\eqref{eq:binary-sir}.
For UR and $k \geq 1$, fix a state $1$ candidate and condition on its uniform variable being $u$.
It beats another state $1$ candidate when that candidate's uniform exceeds $u$, and a state $0$ candidate when its uniform exceeds $\eps u/(1 - \eps)$.
The uniforms are independent, and ties have probability zero.
Summing over the $k$ possible selected state $1$ candidates gives
\[
    \Pr(\widehat{Y}_U = 1 \mid K = k) = k\int_0^1 (1 - u)^{k - 1}\left(1 - \frac{\eps}{1 - \eps}u\right)^{N - k}\dd u.
\]
The probability is zero when $k = 0$.
Therefore,
\begin{align}
    \Pr(\widehat{Y}_U = 1) &= \frac{1}{2^N}\int_0^1\sum_{k = 1}^N \binom{N}{k}k(1 - u)^{k - 1}\left(1 - \frac{\eps}{1 - \eps}u\right)^{N - k}\dd u \notag \\
    &\stackrel{\mathrm{(i)}}{=} \frac{N}{2}\int_0^1
    \left(1 - \frac{u}{2(1 - \eps)}\right)^{N - 1}\dd u \\
    &\stackrel{\mathrm{(ii)}}{=} (1 - \eps)
    \left[1 - \left(\frac{1 - 2\eps}{2(1 - \eps)}\right)^N\right],
\end{align}
where (i) uses $k\binom{N}{k} = N\binom{N - 1}{k - 1}$ and the binomial theorem;
and (ii) evaluates the integral, proving~\eqref{eq:binary-ur} directly from the selection rule.

\paragraph{The gap and its rate.}
To prove~\eqref{eq:binary-exact-gap}, it is enough to express $g(m) - g(X)$ as a centered linear term plus a nonnegative remainder, where $m := 1/2$.
The linear term will disappear on taking expectations because $\E[X - m] = 0$.
Since $g(m) = 1 - \eps = \pi(1)$ and
\[
    g''(x) = - \frac{2\eps(1 - \eps)(1 - 2\eps)}{d(x)^3} < 0,
\]
$g$ lies below its tangent at $m$.
The exact distance from that tangent is
\[
    g(m) + g'(m)(x - m) - g(x) = \frac{4\eps(1 - \eps)(1 - 2\eps)(x - m)^2}{d(x)}.
\]
Taking expectations therefore proves~\eqref{eq:binary-exact-gap}.
The formula is positive for every finite $N$, and attributes the gap to fluctuations in the candidate counts, not only to batches that miss a state.
The bounds $\eps \leq d(X) \leq 1 - \eps$ and $\E[(X - m)^2] = 1/(4N)$ give
\[
    \frac{\eps(1 - 2\eps)}{N}
    \leq \pi(1) - \Pr(\widehat{Y}_{\mathrm{SIR}} = 1)
    \leq \frac{(1 - \eps)(1 - 2\eps)}{N}.
\]
For the leading coefficient, it suffices to keep the quadratic term in Taylor's formula and bound its expected remainder by $O_{\eps}(N^{-3/2})$.
Here $g''(m) = - 16\eps(1 - \eps)(1 - 2\eps)$.
For fixed $\eps \in (0, 1/2)$, the third derivative of $g$ is bounded on $[0, 1]$, because $d(x) \geq \eps$.
Taylor's formula gives
\begin{align*}
        \E[g(X)] &= g(m) + g'(m)\E[X - m] + \frac{g''(m)}{2}\E[(X - m)^2] + O_{\eps}\left(\E[|X - m|^3]\right) \\
    &= g(m) + \frac{g''(m)}{8N}
    + O_{\eps}(N^{-3/2}).
\end{align*}
Here the remainder follows from
\[
    \E[|X - m|^3] \leq \left(\E[(X - m)^4]\right)^{3/4} = O(N^{-3/2}), \qquad
    \E[(X - m)^4] = \frac{3}{16N^2} - \frac{1}{8N^3}.
\]
Consequently,
\[
    \pi(1) - \Pr(\widehat{Y}_{\mathrm{SIR}} = 1)
    = \frac{2\eps(1 - \eps)(1 - 2\eps)}{N} + O_{\eps}(N^{-3/2}).
\]
Both samplers assign state $1$ less than its target probability.
On a binary space, each TV error equals this gap, proving Proposition~\ref{prop:binary-sir-separation}.
The leading SIR coefficient is strictly positive for every fixed $0 < \eps < 1/2$.
The remainder constants may depend on $\eps$; no uniform limit as $\eps$ approaches an endpoint is claimed.

\subsection{Equality with RS at \texorpdfstring{$M = C_\infty^\pi$}{M=Cinfpi} for two weight values}
\label{app:two-level-rs}
The equality with RS at $M=C_\infty^\pi$ extends beyond two states: the relevant restriction is that the weights take at most two values.
We retain a specific fallback rule for RS: if all candidates are rejected, each of the $N$ observed candidate indices is selected with probability $1/N$.
This is uniform selection among indices, not among distinct states.

\begin{proposition}[Two possible weight values]
\label{prop:two-level-rs}
Suppose $C_\infty^\pi < \infty$ and $w(Y) \in \{a, C_\infty^\pi\}$ almost surely under $\mu$, where $0 \leq a < C_\infty^\pi$.
Let RS use threshold $M = C_\infty^\pi$ and the uniform fallback above.
If all sampled weights vanish, let UR also select each candidate index with probability $1/N$.
Then, for every fixed $N \geq 1$,
\[
    P_N^U = P_{N, C_\infty^\pi}^{\mathrm{RS}}.
\]
\end{proposition}
\begin{proof}
It suffices to compare the probability of selecting each candidate on a fixed list, after uniformly randomizing the order in which RS processes it.
The randomization preserves the RS marginal law because the proposals are independent and identically distributed; UR already treats candidate order symmetrically.
There are two cases to check: whether the list contains a weight equal to the global maximum $C_\infty^\pi$.

First express the order-averaged RS rule in a useful form.
Conditional on the list, assign independent acceptance indicators with probabilities $w_i/C_\infty^\pi$, independently of the random order.
Given a nonempty accepted set, its first index in that order is uniform over the set.
Thus this RS rule selects uniformly among accepted indices, or uses the specified fallback if none is accepted.

If a weight $C_\infty^\pi$ is present, it is also the largest observed weight.
The acceptance probabilities are exactly those in UR's Bernoulli representation from Appendix~\ref{app:bernoulli}.
At least one index is accepted with probability one.
The retained set has the same distribution in the two Bernoulli representations, and both select uniformly from it, so their conditional selection probabilities agree.
Zero weight indices, if present, are never accepted or selected and can be removed when applying that representation.

If no weight $C_\infty^\pi$ is present, every observed weight is $a$.
UR is uniform by symmetry, including the specified convention when $a = 0$.
Order-averaged RS is also symmetric over these indices, both on acceptance and on its fallback, so every index has probability $1/N$.
These cases cover every list under the two-value assumption.
Averaging the equal conditional probabilities over the candidates proves the claimed equality of marginal laws.
\end{proof}
The two-value restriction makes a batch without a maximal-weight candidate have equal weights throughout.
With more weight values, such a batch can still contain unequal weights, so this second case of the proof no longer applies.

\subsection{Leading SIR bias}
\label{app:sir-bias}
The binary calculation attributes SIR's error to random counts and normalization.
To examine the same effect beyond state probabilities, let $\psi$ be a bounded measurable function of the returned sample, and write $P(\psi) := \int \psi\dd P$ for its expectation under $P$.
For example, choosing $\psi = \1_D$ makes $P(\psi)$ the probability of an event $D$; taking all bounded $\psi$ will later recover the TV error.
Throughout this subsection and the next subsection, assume
$0 < w(Y) \leq C < \infty$ $\mu$-almost surely for a fixed $C$.
No positive lower bound on $w$ is required.
Put $a := \pi(\psi)$, $b := \|\psi\|_\infty$, and
\[
    A_\psi := \E_\pi[(w - C^\pi)\psi].
\]
No monotonicity of $\psi$ is assumed.
We prove
\begin{equation}
    P_N^{\mathrm{SIR}}(\psi) - \pi(\psi) = - \frac{A_\psi}{N} + O_C(bN^{-3/2}).
    \label{eq:sir-bias}
\end{equation}
The coefficient is signed and may vanish for a particular $\psi$.
The TV expansion below instead considers all bounded measurable functions.

\paragraph{Reduction to two estimates.}
We will separate the bias into an exactly computable leading term and a remainder uniform over bounded $\psi$.
Write $W_i := w(Y_i)$ and define
\[
    \overline{W}_N := \frac{1}{N}\sum_{i = 1}^N W_i, \qquad
    Z_N := \frac{1}{N}\sum_{i = 1}^N W_i(\psi(Y_i) - a), \qquad
    D_N := \overline{W}_N - 1.
\]
Here $Z_N$ is a centered weighted fluctuation, not an estimate of the normalizing constant $Z$.
Conditional on the candidates, SIR's expected value of $\psi$ is their weighted average.
Subtracting $a$ gives
\[
    \frac{\sum_{i = 1}^N W_i\psi(Y_i)}{\sum_{i = 1}^N W_i} - a = \frac{Z_N}{\overline{W}_N}.
\]
Although $\E[Z_N] = 0$, dividing by the dependent random denominator need not preserve that zero mean.
The identity
\[
    \frac{1}{1 + D_N} = 1 - D_N + \frac{D_N^2}{1 + D_N}
\]
separates the leading interaction with the denominator from a remainder:
\begin{equation}
    P_N^{\mathrm{SIR}}(\psi) - a
    = - \E[Z_ND_N] + \E\left[\frac{Z_ND_N^2}{\overline{W}_N}\right].
    \label{eq:app-ratio-bias}
\end{equation}
Thus it suffices to establish
\[
    \E[Z_ND_N] = \frac{A_\psi}{N}, \qquad
    \E\left[\left|\frac{Z_ND_N^2}{\overline{W}_N}\right|\right] = O_C(bN^{-3/2}).
\]
The second estimate controls the signed remainder as well, and its uniform dependence on $b$ will justify taking the TV supremum.

\paragraph{The leading coefficient.}
Let $Y \sim \mu$ and $W := w(Y)$.
Both $W(\psi(Y) - a)$ and $W - 1$ have mean zero, so cross terms from distinct samples vanish.
Consequently,
\begin{align}
    \E[Z_ND_N] &\stackrel{\mathrm{(i)}}{=} \frac{1}{N}\E_\mu[W(\psi(Y) - a)(W - 1)] \\
    &\stackrel{\mathrm{(ii)}}{=} \frac{1}{N}\E_\mu[W^2(\psi(Y) - a)] \notag \\
    &\stackrel{\mathrm{(iii)}}{=} \frac{1}{N}\E_\pi[(w - C^\pi)\psi] \notag \\
    &= \frac{A_\psi}{N}.
\end{align}
Here (i) uses independence and centering;
(ii) uses $\E_\mu[W(\psi(Y) - a)] = 0$;
and (iii) uses $w\dd\mu = \dd\pi$, $\E_\pi[w] = C^\pi$, and $a = \pi(\psi)$.
Consequently, the leading bias in~\eqref{eq:app-ratio-bias} comes from the dependence between the centered weighted sum and the total observed weight.

\paragraph{The remainder without a lower weight bound.}
We cannot bound $1/\overline{W}_N$ uniformly, because weights may approach zero.
The required estimate will instead follow from two centered moments and a lower-tail probability.
Set $B_N := \{\overline{W}_N \geq 1/2\}$.
We first reduce to those quantities:
\begin{align*}
    \E\left[\left|\frac{Z_ND_N^2}{\overline{W}_N}\right|\right]
    &\stackrel{\mathrm{(i)}}{\leq} 2\E[|Z_N|D_N^2] + 2b\Pr(B_N^c) \\
    &\stackrel{\mathrm{(ii)}}{\leq} 2\left(\E[Z_N^2]\right)^{1/2}\left(\E[D_N^4]\right)^{1/2} + 2b\Pr(B_N^c).
\end{align*}
Here (i) follows since $1/\overline{W}_N \leq 2$ on $B_N$.
On $B_N^c$, use the whole ratio instead: $Z_N/\overline{W}_N$ is a weighted average of $\psi$ minus $a$, so its absolute value is at most $2b$; also $|D_N| \leq 1$ there.
(ii) follows from the Cauchy--Schwarz inequality.

It remains to bound the three quantities on the right.
Independence and centering give
\[
    \E[Z_N^2] = \frac{1}{N}\E_\mu[W^2(\psi(Y) - a)^2]
    \leq \frac{4Cb^2}{N}.
\]
In the fourth moment of the centered sum, only four equal indices or two equal pairs survive, so
\begin{align*}
    \E[D_N^4] &= \frac{N\E_\mu[(W - 1)^4] + 3N(N - 1)\left(\E_\mu[(W - 1)^2]\right)^2}{N^4} \\
    &= O_C(N^{-2}).
\end{align*}
Finally, since $0 < W \leq C$ and $\E_\mu[W] = 1$, Hoeffding's inequality yields
\[
    \Pr(B_N^c) \leq \exp\left(- \frac{N}{2C^2}\right).
\]
Substituting these estimates into the displayed bound gives $O_C(bN^{-3/2})$.
Together with the leading coefficient, this proves~\eqref{eq:sir-bias} without an inverse moment assumption.

\paragraph{From expectation bias to TV error.}
The remainder is uniform over $\|\psi\|_\infty \leq 1$, which is essential when taking the supremum that defines TV.
Hence
\begin{align}
    \DTV(P_N^{\mathrm{SIR}}, \pi) &= \frac{1}{2}\sup_{\|\psi\|_\infty \leq 1}
    |P_N^{\mathrm{SIR}}(\psi) - \pi(\psi)| \notag \\
    &\stackrel{\mathrm{(i)}}{=} \frac{1}{2N}
    \sup_{\|\psi\|_\infty \leq 1}|A_\psi| + O_C(N^{-3/2}) \notag \\
    &\stackrel{\mathrm{(ii)}}{=} \frac{1}{2N}\E_\pi[|w - C^\pi|] + O_C(N^{-3/2}).
    \label{eq:app-sir-tv-coefficient}
\end{align}
Here (i) uses the uniform remainder in~\eqref{eq:sir-bias}.
For (ii), choose $\psi(y) = \operatorname{sign}(w(y)-C^\pi)$ to attain the supremum.
The coefficient vanishes exactly when $w = C^\pi$ $\pi$-almost surely.
Because the weights are positive, $\pi$ and $\mu$ have the same null sets; hence $w$ is then constant $\mu$-almost surely.
The identity $\E_\mu[w] = 1$ forces that constant to be one, so the coefficient vanishes exactly when $\pi = \mu$.
The leading term in~\eqref{eq:sir-bias} is the classical asymptotic bias term for self-normalized importance sampling; see, for example, \citet[Theorem~2.1]{deligiannidis2026}.
The proof here additionally supplies the displayed remainder uniformly over bounded measurable functions $\psi$ under a finite upper weight bound, without requiring a positive lower weight bound or an inverse moment assumption.

Combining~\eqref{eq:app-sir-tv-coefficient} with Corollary~\ref{cor:rates}, every fixed nontrivial pair with bounded positive weights therefore has SIR error of order $1/N$, whereas UR's TV error decays exponentially.

\subsection{How UR cancels the leading SIR bias}
With the definitions and assumptions above, we prove
\begin{equation}
    P_N^U(\psi) - P_N^{\mathrm{SIR}}(\psi) = \frac{A_\psi}{N} + o(N^{-1}).
    \label{eq:correction}
\end{equation}
The marginal coefficient also follows from~\eqref{eq:sir-bias} and the exponential UR error bound.
Here we derive it from the choices within a fixed batch, to explain how UR changes the SIR probabilities before averaging over proposals.
Let $p_i^U$ and $p_i^{\mathrm{SIR}}$ denote the selection probabilities conditional on the realized candidates, and define
\[
    \Delta_N(\psi) := \sum_{i = 1}^N (p_i^U - p_i^{\mathrm{SIR}})\psi(Y_i).
\]
Its expectation is the left side of~\eqref{eq:correction}.
Write $W_\Sigma := \sum_{i = 1}^N W_i$ and $Q_\Sigma := \sum_{i = 1}^N W_i^2$.
The event $B_N$ from the preceding subsection is equivalently $\{W_\Sigma \geq N/2\}$.
We first show how a batchwise expansion implies~\eqref{eq:correction}, then reduce that expansion to a candidate-wise estimate and prove the remaining integral approximation.

\paragraph{Proof from a batchwise expansion.}
We will establish, on $B_N$ and for sufficiently large $N$ depending only on $C$,
\begin{equation}
    \Delta_N(\psi) = \frac{\sum_i W_i^2\psi(Y_i)}{W_\Sigma^2} - \frac{Q_\Sigma\sum_i W_i\psi(Y_i)}{W_\Sigma^3} + O_C(bN^{-2}).
    \label{eq:app-aggregated-correction}
\end{equation}
To see why this is sufficient, split the expected correction according to $B_N$:
\[
    N\E[\Delta_N(\psi)] = \E[N\Delta_N(\psi)\1_{B_N}] + \E[N\Delta_N(\psi)\1_{B_N^c}].
\]
It is enough for the first term to converge to $A_\psi$ and the second to vanish.
Assuming~\eqref{eq:app-aggregated-correction}, the strong law gives
\begin{align*}
    N\Delta_N(\psi)\1_{B_N} &= \left[
    \frac{N^{-1}\sum_i W_i^2\psi(Y_i)}{(W_\Sigma/N)^2} - \frac{(Q_\Sigma/N)(N^{-1}\sum_i W_i\psi(Y_i))}{(W_\Sigma/N)^3}
    \right]\1_{B_N} + O_C(b/N) \\
    & \rightarrow \E_\pi[w\psi] - C^\pi\pi(\psi) \\
    &= A_\psi \qquad \text{almost surely}.
\end{align*}
Here $W_\Sigma/N \to 1$, $Q_\Sigma/N \to C^\pi$, and the two weighted sample averages converge to $\E_\pi[w\psi]$ and $\pi(\psi)$; in particular, $B_N$ occurs eventually almost surely.
Each of the two leading terms, restricted to $B_N$, is at most $2Cb$ in absolute value.
In particular, $Q_\Sigma \leq CW_\Sigma$, $|\sum_i W_i\psi(Y_i)| \leq bW_\Sigma$, and $W_\Sigma \geq N/2$.
The remainder is uniform over batches in $B_N$, so dominated convergence gives $\E[N\Delta_N(\psi)\1_{B_N}] \to A_\psi$.
On the complement, $|\Delta_N(\psi)| \leq 2b$ and the preceding Hoeffding bound give
\[
    N\E[|\Delta_N(\psi)|\1_{B_N^c}] \leq 2bN\exp\left(- \frac{N}{2C^2}\right)
    \rightarrow 0.
\]
Hence the batchwise expansion implies~\eqref{eq:correction}.
We now prove that expansion.

\paragraph{A sufficient candidate-wise estimate.}
Because $\Delta_N(\psi)$ is a sum over $N$ candidates, it suffices to prove the following expansion uniformly in $i$ on $B_N$:
\begin{equation}
    p_i^U - p_i^{\mathrm{SIR}} = \frac{W_i}{W_\Sigma^2}\left(W_i - \frac{Q_\Sigma}{W_\Sigma}\right) + O_C(N^{-3}).
    \label{eq:app-original-weights}
\end{equation}
Multiplying by $\psi(Y_i)$ and summing gives the two leading terms in~\eqref{eq:app-aggregated-correction}, with total remainder $O_C(bN^{-2})$.
Here the centering term $Q_\Sigma/W_\Sigma$ is the
SIR-weighted average importance weight in the realized batch.
Hence, at leading order, UR increases the selection probability of candidates with weights above this average and decreases it for those below.
Note that the correction sums to zero across candidates, as required for a change of selection probabilities.
\begin{align*}
    \sum_i \frac{W_i}{W_\Sigma^2}\left(W_i - \frac{Q_\Sigma}{W_\Sigma}\right)
    &= \frac{Q_\Sigma}{W_\Sigma^2} - \frac{Q_\Sigma}{W_\Sigma^2}
    = 0.
\end{align*}
We next derive this estimate from the exact selection probabilities.

\paragraph{Reduction to the selection integral.}
Fix a positive batch and set
\[
    z_j := \frac{W_j}{\max_k W_k}, \qquad
    \lambda_i := \sum_{j \ne i} z_j, \qquad
    s_{2, i} := \sum_{j \ne i} z_j^2.
\]
The exact selection law~\eqref{eq:batch-law} gives
\[
    \frac{p_i^U}{z_i} = \int_0^1 \prod_{j \ne i}(1 - z_jt)\dd t.
\]
We will prove, uniformly for $\lambda_i \geq 1$,
\begin{equation}
    \frac{p_i^U}{z_i} = \frac{1}{\lambda_i} - \frac{s_{2, i}}{\lambda_i^3} + O(\lambda_i^{-3}).
    \label{eq:app-dense-integral}
\end{equation}
The remainder constant is independent of the batch.
The second term cannot be absorbed into the remainder: its numerator $s_{2, i}$ can grow with $\lambda_i$, although $s_{2, i} \leq \lambda_i$.

First use this integral approximation to obtain~\eqref{eq:app-original-weights}.
Since $0 < z_i \leq 1$,
\begin{align*}
    p_i^{\mathrm{SIR}} = \frac{z_i}{\lambda_i + z_i}
    = \frac{z_i}{\lambda_i} - \frac{z_i^2}{\lambda_i^2} + O\left(\frac{z_i}{\lambda_i^3}\right).
\end{align*}
Multiplying~\eqref{eq:app-dense-integral} by $z_i$ and subtracting this expansion yields
\begin{equation}
    p_i^U - p_i^{\mathrm{SIR}} = \frac{z_i}{\lambda_i^2}\left(z_i - \frac{s_{2, i}}{\lambda_i}\right) + O\left(\frac{z_i}{\lambda_i^3}\right).
    \label{eq:app-candidate-correction}
\end{equation}
On $B_N$, $W_i \leq C$ and $W_\Sigma \geq N/2$ imply
\[
    \lambda_i \geq \frac{N/2 - C}{C}
    \geq \frac{N}{4C} \qquad \text{when } N \geq 4C.
\]
Thus $\lambda_i \geq 1$ and the remainder is uniformly $O_C(N^{-3})$.
In the original weights, the leading expression is
\[
    \frac{W_i^2}{(W_\Sigma - W_i)^2} - \frac{W_i(Q_\Sigma - W_i^2)}{(W_\Sigma - W_i)^3}.
\]
Because $W_i \leq C$, $Q_\Sigma \leq CW_\Sigma$, and $W_\Sigma - W_i \geq N/4$, replacing $W_\Sigma - W_i$ by $W_\Sigma$ in the denominators changes this expression by $O_C(N^{-3})$, uniformly in $i$.
The term with numerator $W_i^3$ is also $O_C(N^{-3})$.
This proves~\eqref{eq:app-original-weights}, and hence~\eqref{eq:correction}, once the integral approximation is verified.

\paragraph{Proof of the integral approximation.}
The product in the selection integral is bounded by $e^{-\lambda_i t}$, so its mass concentrates near $t = 0$ as $\lambda_i$ grows.
The approximation in~\eqref{eq:app-dense-integral} will follow by replacing that product with $e^{-\lambda_i t}(1 - s_{2, i}t^2/2)$, since
\[
    \int_0^\infty e^{-\lambda_i t}
    \left(1 - \frac{s_{2, i}t^2}{2}\right)\dd t
    = \frac{1}{\lambda_i} - \frac{s_{2, i}}{\lambda_i^3}.
\]
To justify the replacement, it suffices to bound the difference on $[0, 1/2]$ and both tails starting at $1/2$ by $O(\lambda_i^{-3})$.

For the local difference, we claim that, uniformly on $0 \leq t \leq 1/2$,
\[
    \prod_{j \ne i}(1 - z_jt)
    = e^{-\lambda_i t}
    \left[1 - \frac{s_{2, i}t^2}{2}
    + O(\lambda_i t^3 + \lambda_i^2t^4)\right].
\]
Its integrated error is at most a constant times
\begin{align*}
    \int_0^{1/2} e^{-\lambda_i t}
    (\lambda_i t^3 + \lambda_i^2t^4)\dd t
    &\leq \lambda_i\frac{3!}{\lambda_i^4}
    + \lambda_i^2\frac{4!}{\lambda_i^5}
    = O(\lambda_i^{-3}),
\end{align*}
using $\int_0^\infty t^k e^{-\lambda_i t}\dd t = k!/\lambda_i^{k + 1}$.
The original product's tail is at most $e^{-\lambda_i/2}/\lambda_i$.
The absolute tail of the approximating function is at most
\[
    \int_{1/2}^\infty e^{-\lambda_i t}
    \left(1 + \frac{\lambda_i t^2}{2}\right)\dd t,
\]
since $s_{2, i} \leq \lambda_i$.
Both tails are uniformly $O(\lambda_i^{-3})$ for $\lambda_i \geq 1$.
Thus only the claimed local approximation remains to be checked.

The logarithmic series gives
\[
    \prod_{j \ne i}(1 - z_jt) = e^{-\lambda_i t}\exp\left(- \frac{s_{2, i}t^2}{2} - R_i(t)\right), \qquad
    0 \leq R_i(t) \leq \frac{2}{3}\lambda_i t^3.
\]
Specifically, $\sum_{k \geq 3} u^k/k \leq 2u^3/3$ for $0 \leq u \leq 1/2$, and $\sum_{j \ne i} z_j^3 \leq \lambda_i$.
For $u, v \geq 0$,
\begin{align*}
    |e^{-u - v} - 1 + u| \leq |e^{-u} - 1 + u| + 1 - e^{-v}
    \leq \frac{u^2}{2} + v.
\end{align*}
Apply this with $u = s_{2,i}t^2/2$ and $v = R_i(t)$.
Since $s_{2,i} \leq \lambda_i$, this gives the claimed local remainder.
This verifies~\eqref{eq:app-dense-integral} and completes the candidate-wise, batchwise, and marginal arguments above.
The exponential UR target error follows from~\eqref{eq:exact-threshold}, not merely from cancellation of the displayed $N^{-1}$ coefficients.

\subsection{Continuous example}
Consider $\mu = \Unif(0, 1)$ and $\pi(\dd y) = 2y\dd y$ on $(0, 1)$, so $w(y) = 2y$ and the target mean is $2/3$.
The two output means are
\begin{align}
    \E[\widehat{Y}_U] = \frac{2}{3} - \frac{2^{1 - N}}{3(N + 1)}, \qquad
    \E[\widehat{Y}_{\mathrm{SIR}}] = \frac{2}{3} - \frac{1}{9N} + O(N^{-3/2}).
    \label{eq:continuous}
\end{align}
These formulae concern means, not TV distances.
For UR, we first express the mean error through the largest score; this shows which part of its distribution must be computed.
For SIR, it suffices to evaluate $A_\psi$ from Appendix~\ref{app:sir-bias} at $\psi(y) = y$.

\paragraph{UR: Reduce the mean error to scores below $2$.}
The conditional law~\eqref{eq:score-target-restriction} has mean $s/3$ for $0 < s < 2$ and $2/3$ for $s \geq 2$.
The first value is the mean of the target restricted to $(0, s/2)$; in the second case, no target mass is excluded.
Consequently, the law of total expectation gives
\[
    \frac{2}{3} - \E[\widehat{Y}_U] = \E\left[\left(\frac{2}{3} - \frac{S_{(N)}}{3}\right)\1\{S_{(N)} < 2\}\right].
\]
It remains to compute the largest-score density on $(0, 2)$ and evaluate this expectation.
For independent $Y \sim \mu$ and $U \sim \Unif(0, 1)$, set $S := 2Y/U$.
Since $\Pr(S \leq s \mid Y = y) = (1 - 2y/s)_{+}$ for $s > 0$,
\[
    G(s) = \int_0^{\min\{1, s/2\}}(1 - 2y/s)\dd y
    =
    \begin{cases}
        s/4, & 0 < s < 2, \\
        1 - 1/s, & s \geq 2.
    \end{cases}
\]
Independence makes the largest-score CDF equal to $G(s)^N$.
Its derivative on $(0, 2)$ is $Ns^{N - 1}/4^N$.
Substituting into the preceding expectation gives
\begin{align}
    \frac{2}{3} - \E[\widehat{Y}_U] &\stackrel{\mathrm{(i)}}{=} \int_0^2\left(\frac{2}{3} - \frac{s}{3}\right)\frac{Ns^{N - 1}}{4^N}\dd s \notag \\
    &\stackrel{\mathrm{(ii)}}{=} \frac{2^{1 - N}}{3(N + 1)}.
\end{align}
Step (i) uses the score density in the preceding reduction, and (ii) evaluates the polynomial integral.

\paragraph{SIR: Evaluate the general bias coefficient.}
For $\psi(y) = y$,
\[
    C^\pi = \frac{4}{3}, \qquad
    A_\psi = \int_0^1 (2y - 4/3)y(2y)\dd y
    = 1 - \frac{8}{9}
    = \frac{1}{9}.
\]
Substituting into~\eqref{eq:sir-bias} proves the second formula in~\eqref{eq:continuous}.
Although $w(y)$ approaches zero near the origin, it is positive almost surely and bounded above, so the assumptions of that bias calculation apply.
\section{Proof of Theorem~\ref{thm:uniqueness}}
\label{app:uniqueness}
Fix $N \geq 2$ and maintain the structural restrictions on $\mathsf{A}$ mentioned in Section~\ref{sec:uniqueness}.
The idea is to keep the weights presented to the sampler fixed while varying how frequently those weights appear.
A rule that sees only relative weights cannot change its conditional choice probabilities as these proposal frequencies vary.
Its marginal output probabilities are therefore polynomials in the frequencies.
The required error rate forces their low degree coefficients to agree with UR's, thereby determining the selection probabilities on each batch.
We first illustrate the argument for two and three candidates, then provide the general proof.

\paragraph{Why accuracy determines a choice.}
Take $N = 2$ and keep the relative weights of states $0$ and $1$ fixed at $1/2$ and $1$.
We change only how often state $0$ is proposed:
let $\mu_\eta(0) = \eta$ and $\mu_\eta(1) = 1 - \eta$, where $0 < \eta < 1$.
Since $\pi(y) \propto w_u(y)\mu(y)$, 
\[
    \pi_\eta(0) = \frac{\eta/2}{\eta/2 + 1 - \eta} = \frac{\eta}{2 - \eta}, \qquad
    C_\infty^{\pi_\eta} = \frac{1}{1 - \eta/2} = \frac{2}{2 - \eta}.
\]
Let $q$ denote the probability that $\mathsf{A}$ selects state $0$ when the batch contains both states.
The sampler sees the same relative weights $(1/2, 1)$ on every such batch, regardless of $\eta$, and symmetry makes this probability independent of the order of the two candidates.
Hence the same $q$ must be used for every $\eta$.
Decomposing over the three batch types,
\[
    P_2^{\mathsf{A}}(\{0\}) = \Pr(00) + \Pr(\{01,10\})\Pr(\text{return } 0 \mid \{01,10\})
    = \eta^2 + 2\eta(1 - \eta)q.
\]
For small $\eta$, one can readily see that $P_2^{\mathsf{A}}(\{0\}) = 2q\eta + O(\eta^2)$ and $\pi_\eta(0) = \frac{\eta}{2} + O(\eta^2)$.
Hence any mismatch between $2q$ and $1/2$ creates a first-order error in $\eta$.
However, the assumed guarantee allows only second-order error, since $1 - \tfrac{1}{C_\infty^{\pi_\eta}} = \tfrac{\eta}{2}$.
Since the state space is binary, $\lvert P_2^{\mathsf{A}}(\{0\})-\pi_\eta(0)\rvert = \DTV(P_2^{\mathsf{A}},\pi_\eta)$ and~\eqref{eq:global-characterization} becomes:
\[
    \left\lvert \eta^2 + 2\eta(1 - \eta)q - \frac{\eta}{2 - \eta}\right\rvert
    = \lvert P_2^{\mathsf{A}}(\{0\})-\pi_\eta(0)\rvert
    \leq \frac{L_2\eta^2}{4}
    \Longleftrightarrow
    \left\lvert \eta + 2(1 - \eta)q - \frac{1}{2-\eta}\right\rvert
    \leq \frac{L_2\eta}{4}.
\]
Letting  $\eta \downarrow 0$, this forces $2q = 1/2$, so $q = 1/4$.
This is exactly UR's probability, since $\Pr(U_L < U_H/2) = 1/4$ for independent uniforms on $(0,1)$.
The general proof below applies a similar argument with several rare states, determining the probabilities on every possible batch.

\subsection{Three candidates show how accuracy determines a batchwise choice}
\label{app:three-candidate}
The preceding binary example determines a choice from a first-order term.
With three candidates, a batch containing two rare states first appears in a second-order term.
This is why one must examine more than the probability of seeing a single rare candidate.

Take distinct relative weights $a, b \in (0, 1)$ and a reference weight $1$.
Give these three states proposal probabilities $\eta_1, \eta_2, 1 - \eta_1 - \eta_2$, respectively, where $\eta_1, \eta_2 > 0$ and $\eta_1 + \eta_2 < 1$.
Write $\boldsymbol{\eta} := (\eta_1, \eta_2)$.
The common normalizer is
\(
    Z_{\boldsymbol{\eta}} := 1 - (1-a)\eta_1 - (1-b)\eta_2.
\)
The largest normalized importance weight is therefore
\[
    C_\infty^{\pi_{\boldsymbol{\eta}}} = \frac{1}{Z_{\boldsymbol{\eta}}},
\]
and hence
\[
    1 - \frac{1}{C_\infty^{\pi_{\boldsymbol{\eta}}}}
    = (1-a)\eta_1 + (1-b)\eta_2.
\]
The assumed bound~\eqref{eq:global-characterization}, specialized to $N = 3$, gives
\[
    \DTV(P_3^{\mathsf{A}}, \pi_{\boldsymbol{\eta}})
    \leq L_3\left((1 - a)\eta_1 + (1 - b)\eta_2\right)^3.
\]
Therefore, any first- or second-order mismatch between the sampler and the target would violate the cubic error bound.

First, consider the separate binary family containing only relative weights $a$ and $1$.
Let $\ell$ be the probability that $\mathsf{A}$ selects the sole $a$ candidate from $(a, 1, 1)$.
The sampler returns the lower weight state with probability $3\eta_1(1-\eta_1)^2\ell + O(\eta_1^2) = 3\eta_1\ell + O(\eta_1^2)$, while the corresponding target probability is
\[
    \frac{a\eta_1}{1-(1-a)\eta_1} = a\eta_1 + O(\eta_1^2).
\]
The assumed error is $O(\eta_1^3)$, so their linear coefficients agree and $\ell = a/3$.
The first-order coefficient determines the rule on a batch with one rare candidate.
To identify the probability of choosing state $a$ when both rare states appear, we inspect the coefficient
of $\eta_1\eta_2$:
it is the first coefficient that depends on a batch containing one $a$ candidate and one $b$ candidate.
Its lower order contributions are already determined by the single-rare-state calculation.

Return to the ternary state family and let $x$ be the probability of selecting the $a$ candidate from $(a,b,1)$.
The target probability of the state with relative weight $a$ expands as:
\begin{align*}
    \frac{a\eta_1}{Z_{\boldsymbol{\eta}}} &= a\eta_1 + a(1-a)\eta_1^2 + a(1-b)\eta_1\eta_2 + O((\eta_1+\eta_2)^3).
\end{align*}
Hence its mixed $\eta_1\eta_2$ coefficient is $a(1-b)$.
We now compute the same coefficient for the sampler.
The contribution of the batch composition $(a,1,1)$ is
\begin{align*}
    3\eta_1(1-\eta_1-\eta_2)^2\frac{a}{3} &= a\eta_1 -2a\eta_1^2 -2a\eta_1\eta_2 +O((\eta_1+\eta_2)^3),
\end{align*}
so its mixed coefficient is $-2a$.
The contribution of $(a,b,1)$ is
\[
    6\eta_1\eta_2(1-\eta_1-\eta_2)x
    =
    6x\eta_1\eta_2
    +O((\eta_1+\eta_2)^3),
\]
so its mixed coefficient is $6x$.
Batches without an $a$ candidate contribute zero.
Every remaining composition with an $a$ candidate contains either at least two $a$ candidates or three lower weight candidates, so none contributes to the $\eta_1\eta_2$ coefficient.
Therefore the sampler's mixed coefficient is $-2a + 6x$.

The cubic error bound forces this coefficient to match the target's, giving
\[
    -2a + 6x = a(1 - b).
\]
Solving for $x$ gives
\begin{align}
    x = \frac{a(3 - b)}{6} = a\int_0^1 (1 - bu)(1 - u)\dd u,
    \label{eq:three-candidate}
\end{align}
which is precisely UR's probability from~\eqref{eq:batch-law}.
The contribution from $(a, 1, 1)$ was determined first, leaving the mixed coefficient to determine the choice on $(a, b, 1)$.
The general proof repeats this procedure in increasing order of the number of lower weight candidates.

\subsection{General proof in three steps}

\paragraph{Step 1. Construct a family with fixed relative weights.}
We construct a family of target--proposal pairs that keeps the relative weights fixed while making the lower weight states rare.

To this end, we choose $1 \leq d \leq N-1$ distinct relative weights
$a_1, \ldots, a_d \in (0,1)$ and let $a_0 := 1$.
State $0$ is the reference state with the largest relative weight, while states $1, \ldots, d$ have smaller fixed weights.
For $j = 1, \ldots, d$, choose $\eta_j > 0$ with $\sum_{j = 1}^d \eta_j < 1$, and set
\[
    \eta_0 := 1 - \sum_{j = 1}^d \eta_j.
\]
We consider the setting where the $a_j$'s remain fixed throughout the argument, whereas the $\eta_j$'s tend to zero.

Now write $\boldsymbol{\eta} := (\eta_1, \ldots, \eta_d)$ and define the normalizing constant
\[
    Z_{\boldsymbol{\eta}} := \sum_{j = 0}^d a_j\eta_j = 1 - \sum_{j = 1}^d (1 - a_j)\eta_j.
\]
The proposal, target, and normalized weights are, for $j = 0, \ldots, d$,
\begin{equation}
    \mu_{\boldsymbol{\eta}}(j) := \eta_j, \qquad
    \pi_{\boldsymbol{\eta}}(j) := \frac{a_j\eta_j}{Z_{\boldsymbol{\eta}}}, \qquad
    w_{\boldsymbol{\eta}}(j) := \frac{a_j}{Z_{\boldsymbol{\eta}}}.
    \label{eq:rare-types}
\end{equation}
Hence the fixed values $a_j$ are unnormalized importance weights, in the same sense as $w_u = Zw$ in Section~\ref{sec:intro}.
Note that varying $\boldsymbol{\eta}$ changes their frequencies and their common normalization, but not their ratios.
The largest normalized weight is $C_\infty^{\pi_{\boldsymbol{\eta}}} = 1/Z_{\boldsymbol{\eta}}$, so
\begin{equation}
    1 - \frac{1}{C_\infty^{\pi_{\boldsymbol{\eta}}}} = \sum_{j = 1}^d (1 - a_j)\eta_j.
    \label{eq:app-near-ratio}
\end{equation}
As the smaller weights become rare, this quantity tends to zero and the required error becomes small.

\paragraph{Step 2: Express marginal output differences as a polynomial.}
We compare $\mathsf{A}$ with UR rather than directly with the target.
The target probabilities involve the changing normalizer $Z_{\boldsymbol{\eta}}^{-1}$, whereas both samplers' output probabilities are polynomials in $\boldsymbol{\eta}$.
To this end, let $\boldsymbol{n} := (n_1, \ldots, n_d)$ be a count vector which records how many candidates from each lower weight state appear in the batch.
Here $|\boldsymbol{n}| := \sum_{j = 1}^d n_j \leq N$,
so the remaining $N - |\boldsymbol{n}|$ candidates are in state $0$.
For example, if $d = 2$, then $\boldsymbol n = (1,2)$ means one candidate from state $1$, two from state $2$, and the remaining $N-3$ candidates from state $0$.
Let $q_j^{\mathsf{A}}(\boldsymbol{n})$ be the probability of returning state $j$ from this batch, and define $q_j^U(\boldsymbol{n})$ similarly.
Unlike $p_i^{\mathsf{A}}$, which refers to one candidate, $q_j^{\mathsf{A}}$ sums the probabilities of all candidates in state $j$.
Neither the ordering nor the common normalization affects these probabilities, so they are independent of $\boldsymbol{\eta}$.
Write $\Delta_j(\boldsymbol{n}) := q_j^{\mathsf{A}}(\boldsymbol{n}) - q_j^U(\boldsymbol{n})$.

\begin{lemma}[Output probability differences are polynomials]
\label{lem:polynomial-batch}
For the pair in~\eqref{eq:rare-types} and $j = 1, \ldots, d$, let $D_j(\boldsymbol{\eta}) := P_N^{\mathsf{A}}(\{j\}) - P_N^U(\{j\})$.
Then
\begin{equation}
    D_j(\boldsymbol{\eta}) = \sum_{|\boldsymbol{n}| \leq N} \frac{N!}{(N - |\boldsymbol{n}|)!\prod_{h = 1}^d n_h!} \eta_0^{N - |\boldsymbol{n}|}\prod_{h = 1}^d \eta_h^{n_h}\Delta_j(\boldsymbol{n}).
    \label{eq:app-polynomial}
\end{equation}
This is a polynomial in $\eta_1, \ldots, \eta_d$ of total degree at most $N$.
\end{lemma}
\begin{proof}
For each count vector, the multinomial factor times the powers of $\eta_h$ is the probability of that batch composition.
Multiplying by its conditional output difference and summing is the law of total probability.
Since $\eta_0 = 1 - \sum_{h = 1}^d \eta_h$, each term has degree at most $N$.
No continuity of the selection functions is needed: their values on each fixed relative weight vector remain unchanged while the probabilities of those vectors vary.
\end{proof}

\paragraph{Step 3: Use the error guarantee to recover the batchwise probabilities.}
We now use the assumed bound~\eqref{eq:global-characterization} to determine the choices within each batch.
Fix $\boldsymbol{v} \in (0, \infty)^d$ and let $\boldsymbol{\eta} = \theta\boldsymbol{v}$ for sufficiently small $\theta > 0$.
By~\eqref{eq:app-near-ratio}, the triangle inequality and the two error bounds give
\begin{align*}
    |D_j(\theta\boldsymbol{v})| &\stackrel{\mathrm{(i)}}{\leq} \DTV(P_N^{\mathsf{A}}, \pi_{\theta\boldsymbol{v}}) + \DTV(P_N^U, \pi_{\theta\boldsymbol{v}}) \\
    &\stackrel{\mathrm{(ii)}}{\leq} (L_N + 1)\left(\theta\sum_{h = 1}^d (1 - a_h)v_h\right)^N \\
    &= O(\theta^N) \\
    &= o(\theta^{N - 1}),
\end{align*}
where (i) compares both samplers with the same target on the event $\{j\}$;
and (ii) uses~\eqref{eq:global-characterization} and Corollary~\ref{cor:rates}.
The constant $L_N$ cannot change with $\theta$, since it is independent of the target and proposal.
The following lemma explains why this small error removes every lower degree term.

\begin{lemma}[Small marginal differences remove the lower degree terms]
\label{lem:vanishing-coefficients}
Suppose that, for every fixed $\boldsymbol{v} \in (0, \infty)^d$,
\[
    D_j(\theta\boldsymbol{v}) = o(\theta^{N - 1}) \qquad (\theta \downarrow 0).
\]
Then every coefficient of $D_j$ of total degree below $N$ is zero.
\end{lemma}
\begin{proof}
Setting $\boldsymbol{\eta} = \theta\boldsymbol{v}$ scales all rare-state probabilities down together while keeping their proportions fixed.
For this fixed $\boldsymbol{v}$, write
\[
    D_j(\theta\boldsymbol{v}) = c_0(\boldsymbol{v}) + c_1(\boldsymbol{v})\theta + \cdots + c_N(\boldsymbol{v})\theta^N.
\]
If the first nonzero term had degree $m < N$, dividing by $\theta^m$ would give a nonzero limit.
The assumed bound instead gives a limit of zero, including when $m = N - 1$.
Hence $c_m(\boldsymbol{v}) = 0$ for every $m < N$.

This conclusion holds for every positive $\boldsymbol{v}$, not just one choice of proportions.
For each $m$, $c_m(\boldsymbol{v})$ is exactly the degree-$m$ part of $D_j$ evaluated at $\boldsymbol{v}$.
It is a polynomial that vanishes throughout the positive orthant, and therefore all its coefficients vanish.
In particular, fix all but one coordinate in an open box and use that a univariate polynomial vanishing on an interval is zero; repeating this for the other coordinates proves the claim.
Lastly, notice that varying the proportions is essential:
a polynomial such as $\eta_1 - \eta_2$ vanishes when $\eta_1 = \eta_2$ without having zero coefficients.
\end{proof}

\begin{lemma}[The coefficients determine the selection probabilities]
\label{lem:recover-batches}
If all coefficients of $D_j$ of total degree below $N$ vanish for $j = 1, \ldots, d$, then $\Delta_j(\boldsymbol{n}) = 0$ whenever $|\boldsymbol{n}| < N$.
Allowing the relative weights to vary identifies each candidate's selection probability on every positive input vector.
\end{lemma}
\begin{proof}
Fix $j \geq 1$ and proceed in increasing order of $m := |\boldsymbol{n}|$.
For $m = 0$, state $j$ is absent, so both rules assign it probability zero.
Now fix $1 \leq m \leq N - 1$ and suppose all compositions with fewer than $m$ lower weight candidates already have $\Delta_j = 0$.
Their terms in~\eqref{eq:app-polynomial} vanish entirely.
Compositions with more than $m$ lower weight candidates cannot contribute to degree $m$, since they contain more than $m$ factors of the rare-state probabilities.
For a composition with exactly $m$ such candidates,
\[
    \eta_0^{N - m}
    = \left(1 - \sum_{h = 1}^d \eta_h\right)^{N - m}
\]
has constant term one, and all its other terms raise the total degree above $m$.
Therefore the degree-$m$ part is exactly
\[
    \sum_{|\boldsymbol{n}| = m}\frac{N!}{(N - m)!\prod_{h = 1}^d n_h!}\Delta_j(\boldsymbol{n})\prod_{h = 1}^d \eta_h^{n_h}.
\]
Each composition produces a distinct monomial with a positive multinomial coefficient.
Since the displayed polynomial is identically zero, every $\Delta_j(\boldsymbol{n})$ with $|\boldsymbol{n}| = m$ is zero.
Induction therefore proves the assertion for $j \geq 1$; the probabilities of state $0$ then agree since each rule's probabilities sum to one.

Finally, divide an arbitrary positive weight vector by its maximum.
If all entries become one, symmetry gives probability $1/N$ to each candidate.
Otherwise, its distinct smaller values are some $a_1, \ldots, a_d$, with multiplicities $n_1, \ldots, n_d$.
At least one maximal entry remains, so $|\boldsymbol{n}| \leq N - 1$.
The preceding induction identifies the total selection probability of each equal-weight group.
Within a group, exchanging two candidates leaves the weight vector unchanged, so symmetry assigns them equal probabilities.
In conclusion, the group probability determines each individual probability.
\end{proof}

Lemma~\ref{lem:vanishing-coefficients} removes the lower degree coefficients, and Lemma~\ref{lem:recover-batches} determines every batchwise probability.
Conversely, Corollary~\ref{cor:rates} supplies the UR bound with $L_N = 1$.
This completes the proof of Theorem~\ref{thm:uniqueness}.

\paragraph{A weaker local condition.}
For fixed $N$, the same conclusion holds under the weaker requirement
\[
    \sup_{(\pi, \mu): C_\infty^\pi \leq 1 + \delta}
    \DTV(P_N^{\mathsf{A}}, \pi) = o(\delta^{N - 1}) \qquad (\delta \downarrow 0),
\]
where the supremum ranges over finite target--proposal pairs with strictly positive probabilities.
For each fixed $\boldsymbol{v} \in (0, \infty)^d$, \eqref{eq:app-near-ratio} shows that $C_\infty^{\pi_{\theta\boldsymbol{v}}} - 1$ is proportional to $\theta$ to first order.
The local requirement and UR's bound therefore give $D_j(\theta\boldsymbol{v}) = o(\theta^{N - 1})$.
Lemmas~\ref{lem:vanishing-coefficients} and~\ref{lem:recover-batches} then apply unchanged.

\section{Relation to independent Metropolis--Hastings}
\label{app:imh}

Theorem~\ref{thm:uniqueness} characterizes the batchwise selection probabilities implied by the bounded weight guarantee.
We next ask whether the resulting marginal output distribution also admits a sequential realization.
It does: when independent Metropolis--Hastings (IMH) is initialized from the
proposal and run using $N$ i.i.d.\ proposals in total, its terminal state has exactly the same marginal distribution as UR.

\begin{theorem}[IMH and UR have the same final output distribution]
\label{thm:ur-imh}
Fix $N \geq 1$ under the positive weight assumptions of Section~\ref{sec:intro}.
Draw $Y_1, \ldots, Y_N$ independently from $\mu$ and initialize $X_1 := Y_1$.
For $t = 2, \ldots, N$, set $X_t := Y_t$ with probability
\(
    \min\{1, \tfrac{w_u(Y_t)}{w_u(X_{t - 1})}\},
\)
and otherwise retain $X_t := X_{t - 1}$.
Then $\mathcal{L}(X_N) = P_N^U$.
Moreover, for any fixed batch with positive weights, if the candidates are presented to IMH in a uniformly random order, then each candidate is selected with the same probability as under UR.
\end{theorem}

\paragraph{Consequences.}
Every distributional guarantee established above for UR transfers verbatim to proposal-initialized IMH at the same fixed proposal budget.
In particular, combining Theorems~\ref{thm:SIR} and~\ref{thm:ur-imh}, for every
convex $f$-divergence covered by Theorem~\ref{thm:SIR},
\[
    D_f(\mathcal{L}(X_N) \Vert \pi) = D_f(P_N^{U} \Vert \pi)
    \leq D_f(P_N^{\mathrm{SIR}} \Vert \pi).
\]
Likewise, the best fixed threshold RS guarantee and the sharp bounded weight guarantee established for UR apply directly to the IMH terminal law.

\begin{proof}
For $N = 1$, both methods return the only proposal.
For $N \geq 2$, the proof of Theorem~\ref{thm:ur-imh} consists of two parts.
We first establish the bounded weight error guarantee for IMH initialized from the proposal, without using any equality with UR.
Specifically, Lemma~\ref{lem:IMH_bound} stated below proves that this IMH procedure has TV error at most $(1 - 1/C_\infty^\pi)^N$ for every finite positive $(\pi, \mu)$.
For a fixed batch, we then randomize its order before running IMH.
By doing this, the resulting selection rule depends only on the relative observed weights and is symmetric in the candidate labels.
Since the proposals are i.i.d., randomizing their order does not change the marginal output distribution, so the symmetrized rule inherits the same TV guarantee.
Theorem~\ref{thm:uniqueness} therefore makes these selection probabilities equal to UR's for every positive weight vector.
Averaging this identity over the proposals yields the desired equality of output distributions, including for pairs with unbounded positive weights.
Appendix~\ref{app:subsec:IMH_bound} and Appendix~\ref{app:subsec:random_order} prove the IMH bound and formalize the order averaging argument, respectively.
\end{proof}
\begin{remark}
    We highlight that the order averaging is essential for comparing probabilities on a fixed batch.
    In the binary example of Section~\ref{sec:sir},  consider a mixed batch containing states $0$ and $1$.
    IMH with $N = 2$ returns state $1$ with probability $1$ on $01$: the higher weight candidate is always accepted.
    On $10$, it returns state $1$ with probability $1 - \tfrac{\eps}{1 - \eps}$, since it keeps the first candidate only if the second is rejected.
    Averaging the two orders gives $1 - \tfrac{\eps}{2(1 - \eps)}$, which is exactly the UR probability on either mixed batch.
    Consequently, equality of output distributions does not mean that the procedures select the same candidate on each ordered batch.
\end{remark}
The equality is distributional at each fixed budget; it is not a pathwise identity.
We highlight that conditional on a fixed ordered batch, direct IMH and UR may select different candidates.
Randomizing the order of the batch makes the IMH
selection probabilities coincide with those of UR. The equivalence also need not extend to adaptive stopping rules or to different initializations.



\subsection{The IMH bound with proposal initialization}
\label{app:subsec:IMH_bound}
\begin{lemma}
\label{lem:IMH_bound}
For every target and proposal on a finite state space with strictly positive probabilities, the IMH procedure of Theorem~\ref{thm:ur-imh} satisfies
\[
    \DTV(\mathcal{L}(X_N),\pi) \leq \left(1 - \frac{1}{C_\infty^\pi}\right)^N
\]
for every $N \geq 1$.
\end{lemma}
\begin{proof}
If $C_\infty^\pi = 1$, then $w \leq 1$ and $\E_\mu w = 1$ imply $\pi = \mu$.
Every update accepts the new proposal, so the error is zero.
Suppose henceforth that $C_\infty^\pi > 1$.

As the state space is finite, one IMH update can be described by a transition matrix $K$.
Its entry $K(x,y)$ refers to the probability that the retained state becomes $y$ when the current state is $x$.

For an arbitrary distribution $\nu$ of the current state, define $(\nu K)(y) := \sum_x \nu(x)K(x,y)$, the distribution after one IMH update.
We will verify the following three facts:
\[
    \pi K = \pi, \qquad
    K(x,y) \geq \frac{\pi(y)}{C_\infty^\pi}, \qquad
    \mu(y) \geq \frac{\pi(y)}{C_\infty^\pi}.
\]
The first says that an update preserves the target distribution.
The other two say that every transition distribution, as well as the initial distribution, contains at least a $1/C_\infty^\pi$ share of the target.

\paragraph{How the three facts yield the desired bound.}
Subtracting the common component $\pi/C_\infty^\pi$ and renormalizing the remainder, let us define a matrix $R$:
\[
    R(x,y) := \frac{K(x,y) - \pi(y)/C_\infty^\pi}{1 - 1/C_\infty^\pi}.
\]
By the two inequalities above, $R(x,y) \geq 0$, and since both $K(x,\cdot)$ and $\pi$ have unit mass, each row of $R$ sums to one.
Hence $R$ is a transition matrix, and the following decomposition holds:
\[
    K(x,\cdot) = \frac{1}{C_\infty^\pi}\pi + \left(1 - \frac{1}{C_\infty^\pi}\right)R(x,\cdot).
\]
Now for any current state distribution $\nu$, the decomposition of $K$ yields:
\[
    \nu K = \frac{1}{C_\infty^\pi}\pi + \left(1 - \frac{1}{C_\infty^\pi}\right)\nu R.
\]
Since $\pi K = \pi$ and $\pi R = \pi$, one can readily see that the same decomposition holds with $\nu$ replaced by $\pi$.
The common target component therefore cancels when we compare $\nu K$ with $\pi K$, and therefore
\[
    \DTV(\nu K,\pi) = \left(1 - \frac{1}{C_\infty^\pi}\right)\DTV(\nu R, \pi R) \leq \left(1 - \frac{1}{C_\infty^\pi}\right)\DTV(\nu,\pi),
\]
where the last inequality follows from the contraction of TV distance under a common transition matrix.
Using the last fact, one can readily see that the initialization inequality likewise yields:
\[
    \DTV(\mu,\pi) = \sum_y(\pi(y) - \mu(y))_+ \leq \left(1 - \frac{1}{C_\infty^\pi}\right)\sum_y\pi(y) = 1 - \frac{1}{C_\infty^\pi}.
\]
Starting from $X_1 \sim \mu$ and applying $N - 1$ updates yields
\begin{align}
    \DTV(\mathcal{L}(X_N),\pi) &\leq \left(1 - \frac{1}{C_\infty^\pi}\right)^{N - 1}\DTV(\mu,\pi)
    \leq \left(1 - \frac{1}{C_\infty^\pi}\right)^N.
    \label{eq:app-IMH_bound}
\end{align}
For $N = 1$, the first line simply uses the initial distribution and the second line is its bound.
It remains to verify the three facts.

\paragraph{The update preserves the target.}
Let $\alpha(x,y) := \min\{1,w(y)/w(x)\}$ be the IMH acceptance probability from state $x$ to a proposal $y$.
Note that the unknown normalizing constant cancels from the ratio.
When $y \neq x$, moving from $x$ to $y$ requires proposing $y$ and accepting it, so $K(x,y) = \mu(y)\alpha(x,y)$.
Using $\pi(x) = w(x)\mu(x)$, observe that
\[
\begin{aligned}
    \pi(x)\mu(y)\alpha(x,y) &= w(x)\mu(x)\mu(y)\min\left\{1, \frac{w(y)}{w(x)}\right\} \\
    &= \mu(x)\mu(y)\min\{w(x), w(y)\}.
\end{aligned}
\]
Notice that the right-hand side is symmetric in $x$ and $y$.
Thus $\pi(x)K(x,y) = \pi(y)K(y,x)$ for $x \neq y$;
for $x = y$, the equality is immediate, including the probability mass from rejection.
Hence $K$ satisfies detailed balance with respect to $\pi$.
Summing over $x$ therefore yields the desired result of $\pi K = \pi$:
\[
    (\pi K)(y) = \sum_x \pi(x)K(x,y)
    = \pi(y)\sum_x K(y,x)
    = \pi(y).
\]

\paragraph{The transition kernel and initialization contain a common target component.}
Since $w(x) \leq C_\infty^\pi$ and $w(y)/C_\infty^\pi \leq 1$,
\[
    \alpha(x,y) = \min\left\{1,\frac{w(y)}{w(x)}\right\} \geq \frac{w(y)}{C_\infty^\pi}.
\]
Proposing $y$ and accepting it is one way for the next state to be $y$, so
\[
    K(x,y) \geq \mu(y)\alpha(x,y)
    \geq \frac{\mu(y)w(y)}{C_\infty^\pi}
    = \frac{\pi(y)}{C_\infty^\pi}.
\]
Hence, every transition distribution $K(x,\cdot)$ contains the common target component $\pi/C_\infty^\pi$.
Finally, $\pi(y) = w(y)\mu(y) \leq C_\infty^\pi\mu(y)$ gives $\mu(y) \geq \pi(y)/C_\infty^\pi$.

This completes the proof.
Finally we note that the stationary transition and contraction arguments are standard in IMH analysis~\citep{mengersen1996,wang2022}.
\end{proof}

\subsection{From the error bound to equality of output distributions}
\label{app:subsec:random_order}

Fix $N \geq 2$ and a deterministic positive weight vector $w_{1:N} := (w_1, \ldots, w_N)$.
Let $p_i^{\mathrm{IMH}}(w_{1:N})$ denote the probability that ordered IMH ultimately retains candidate $i$.
Since IMH processes candidates sequentially, these probabilities need not be symmetric in the candidate labels.

We therefore symmetrize the rule as follows:
First, draw a uniform random permutation $\sigma$, process the same candidates in the permuted order, and map the selected candidate back to its original label.
Under the convention $(\sigma w_{1:N})_{\sigma(i)}=w_i$, the resulting probability of selecting
candidate $i$ is
\[
    \overline{p}_i^{\mathrm{IMH}}(w_{1:N}) := \frac{1}{N!}\sum_\sigma p_{\sigma(i)}^{\mathrm{IMH}}(\sigma w_{1:N}).
\]
From this, one can readily see that the symmetrized rule is scale invariant and symmetric in the candidate labels.
Moreover, since the proposals are i.i.d., this randomization of their order does not change the marginal output distribution of IMH.
Hence, by Lemma~\ref{lem:IMH_bound}, the symmetrized rule satisfies the bounded weight guarantee required by Theorem~\ref{thm:uniqueness}.
Therefore, for every positive weight vector $w_{1:N}$,
\[
    \overline{p}_i^{\mathrm{IMH}}(w_{1:N}) = p_i^U(w_{1:N}), \qquad i \in [N].
\]
Since this equality holds for every positive weight vector, it also applies to each realized candidate batch when the underlying importance ratio is unbounded.

For any pair satisfying the positive weight assumptions of Section~\ref{sec:intro}, let $W_{u,1:N} := (w_u(Y_1), \ldots, w_u(Y_N))$.
Conditioning on the candidate batch and using the equality above, for every measurable set $D$,
\begin{align*}
    \Pr(X_N \in D) &\stackrel{\mathrm{(i)}}{=} \E\left[\sum_{i = 1}^N \overline{p}_i^{\mathrm{IMH}}(W_{u,1:N})\1\{Y_i \in D\}\right] \\
    &\stackrel{\mathrm{(ii)}}{=} \E\left[\sum_{i = 1}^N p_i^U(W_{u,1:N})\1\{Y_i \in D\}\right] \\
    &= P_N^U(D),
\end{align*}
where (i) uses the unchanged marginal output distribution under order averaging;
and (ii) uses $\overline{p}_i^{\mathrm{IMH}}(w_{1:N}) = p_i^U(w_{1:N}),\ i \in [N]$ shown above.
For $N = 1$, both procedures return the only proposal.
This proves Theorem~\ref{thm:ur-imh}.

\begin{remark}
The theorem gives equality of the output distributions separately at each fixed budget; it does not imply equality of the coupled outputs as the budget varies.
For example, let $\pi = \mu = \Unif(0,1)$ and reuse the same proposal sequence.
Then IMH accepts every new proposal and hence $\Pr(X_2 = X_1) = 0$.
By contrast, UR retains its first candidate at budget two with probability $1/2$.
Therefore, equality at each fixed budget does not imply equality of the output sequences across budgets, and in particular does not imply equivalence under adaptive stopping.
\end{remark}

\section{Nonuniform divisors can also converge exponentially}
\label{app:noise}

Uniformity of the auxiliary variables is not necessary for exponential convergence.
Consider the more general score
\[
    S_i^F := \frac{w(Y_i)}{T_i},
\]
where the $T_i$ are i.i.d. positive random variables independent of the proposals.
As a simple example, let
\[
    T \sim \frac{1}{2} \Unif(0,1) + \frac{1}{2} \Unif(2,3).
\]
Its CDF satisfies
\[
    F(t) = \frac{t}{2}, \qquad 0 \leq t \leq 1.
\]
Suppose $C_\infty^\pi<\infty$.
A candidate with state $y$ has score at least $C_\infty^\pi$ exactly when
\[
    T_i \leq \frac{w(y)}{C_\infty^\pi}.
\]
Since $w(y)/C_\infty^\pi\le1$,
\[
    \Pr\left(T_i \leq \frac{w(y)}{C_\infty^\pi} \middle| Y_i = y\right)
    = \frac{w(y)}{2C_\infty^\pi}.
\]

Let $\widehat{I}_F := \arg\max_{i\in[N]} S_i^F$ and $P_N^F := \mathcal L(Y_{\widehat{I}_F})$.
Suppose $C_\infty^\pi < \infty$ and define
\[
    H_i := \left\{T_i \leq \frac{w(Y_i)}{C_\infty^\pi}\right\}, \qquad
    V_i := \frac{C_\infty^\pi T_i}{w(Y_i)} \quad \text{on }H_i .
\]
For every measurable set $D$ and $0\le v\le1$,
\[
\begin{aligned}
    \Pr(Y_i\in D,H_i,V_i\le v) &= \int_D F\left(\frac{v w(y)}{C_\infty^\pi}\right)\mu(dy)
    = \frac{v}{2C_\infty^\pi}\pi(D).
\end{aligned}
\]
Hence, one can observe that conditional on $H_i$, $Y_i \sim \pi$ and $V_i \sim \Unif(0,1)$ are independent.
Conditioning on the set of crossing indices preserves independence across candidates, since each crossing event depends only on its own candidate and divisor.
Among the candidates satisfying $H_i$, the rule selects the smallest $V_i$, since $S_i^F = C_\infty^\pi/V_i$ on $H_i$. Therefore, conditional on at least one threshold crossing, the selected state has law $\pi$.

Each candidate crosses independently with probability $1/(2C_\infty^\pi)$.
It follows that
\[
    \DTV(P_N^F, \pi) \leq \left(1 - \frac{1}{2C_\infty^\pi}\right)^N.
\]
In conclusion, we observe that uniform divisors are not necessary for exponential convergence.
This does not contradict Theorem~\ref{thm:uniqueness}, which assumes a specific universal bounded weight benchmark rather than merely an exponential rate.
The bound above does not establish that stronger guarantee.

\section{Experimental Details}
\label{app:exp_details}

This section provides additional details on the experimental setup, evaluation procedure, and results.

\paragraph{Datasets.}
Throughout the experiments, we evaluate on two mathematical reasoning benchmarks.

\textbf{GSM8K}~\citep{cobbe2021training} is a dataset created by OpenAI, consisting of 8.5K grade-level math word problems.
For each evaluated prompt, we pre-generated a response pool $\mathcal{A}_x^\star$ of $N^\star :=  \text{4,096}$ responses from the reference policy, with a maximum generation length of 500 tokens.

\textbf{MATH500}~\citep{hendrycks2021measuring, lightman2024let} is a 500-question evaluation subset of the MATH benchmark, composed of competition-level mathematics problems.
Compared to GSM8K, MATH500 requires more complex multi-step mathematical reasoning, providing a setting where we can evaluate our method with more challenging problems.
For each problem (i.e., evaluated prompt), we pre-generated a response pool $\mathcal{A}_x^\star$ of 4,096 responses, with a maximum generation length of 1,024 tokens.

\paragraph{Reference policy and response generation.}
We use \path{meta-llama/Llama-3.2-3B-Instruct}~\citep{grattafiori2024llama} as the reference policy.
Specifically, for each evaluated prompt $x$, we pre-generate a fixed response multiset (pool)
\[
    \mathcal{A}_x^\star := \{a_1, \ldots, a_{N^\star}\}, \qquad
    N^\star = 4{,}096.
\]
We use a generation temperature of $0.3$ and maximum generation lengths of 500 tokens for GSM8K and 1,024 tokens for MATH500, respectively.
The response generation and reward model settings are summarized in Table~\ref{tab:settings}.
All response generation and reward model evaluation were run on two NVIDIA RTX PRO 6000 Blackwell Server Edition GPUs.
\begin{table}[!htbp]
  \centering
  \small
  \renewcommand{\arraystretch}{1.25}
  \begin{tabularx}{\linewidth}{@{}l l >{\raggedright\arraybackslash}X@{}}
    \toprule
    \textbf{Component} & \textbf{Setting} & \textbf{Value} \\
    \midrule
    \multirow{5}{*}{\shortstack[l]{Response generation}}
      & Model
      & \path{meta-llama/Llama-3.2-3B-Instruct} \\
      & Max tokens
      & 500 (GSM8K); 1,024 (MATH500) \\
      & Generation temperature
      & 0.3 \\
      & Pre-generated responses/prompt
      & 4,096 \\
    \midrule
    \multirow{2}{*}{\shortstack[l]{Reward model}}
      & Model
      & \path{OpenAssistant/reward-model-deberta-v3-large-v2} \\
    \bottomrule
  \end{tabularx}
  \caption{Response generation and reward model settings.}
  \label{tab:settings}
\end{table}

\paragraph{Finite empirical proposal.}
The fixed response pool $\mathcal{A}_x^\star$ defines the empirical distribution:
\[
    \widehat{\mu}(y \mid x) := \tfrac{1}{N^\star} \sum_{i=1}^{N^\star}\1\{a_i = y\}.
\]
Since $\mathcal{A}_x^\star$ may contain repeated responses, let $\supp(\widehat{\mu}(\cdot \mid x))$ denote the set of distinct responses represented in $\mathcal{A}_x^\star$.
For brevity, we write $\supp(\widehat{\mu})$ when the prompt $x$ is fixed.
This fixed empirical proposal allows us to draw repeated independent candidate sets \emph{without additional LLM generations.}
The pool size $N^\star$ is fixed throughout the experiment, whereas $N$ denotes the sampling budget available to each sampling procedure and is varied to produce the curves.

\paragraph{Reward model and target policy.}
Each distinct response $y \in \supp(\widehat{\mu})$ is scored by \path{OpenAssistant/reward-model-deberta-v3-large-v2}~\citep{kopf2023openassistant}, yielding a reward $\widehat{r}(x,y)$.
Motivated by the standard KL-regularized RLHF objective~\citep{jaques2017sequence, jaques2020human, rafailov2023direct},
we define the reward-tilted target policy
\[
    \pi_\beta(y \mid x) = \frac{\widehat{\mu}(y \mid x)\exp(\widehat{r}(x,y)/\beta)}{\sum_{y' \in \supp(\widehat{\mu})}\widehat{\mu}(y'\mid x)\exp(\widehat{r}(x,y')/\beta)}, \quad y \in \supp(\widehat{\mu}).
\]
Since $\supp(\widehat{\mu})$ is finite, the normalizing constant can be computed exactly for evaluation.
The sampling procedures themselves receive the unnormalized importance weights
\[
    w_u(x,y) = \exp\left(\widehat{r}(x,y)/\beta\right).
\]
The envelope baseline additionally receives the exact empirical envelope, and the budget-calibrated RS baseline receives $\delta_{\rm cal}$;
no sampler is supplied the normalizing constant.
Hence, $(\pi_\beta, \widehat{\mu})$ forms a finite target--proposal pair to which the theoretical results in Sections~\ref{sec:oracle}--\ref{sec:uniqueness} apply directly.


\subsection{Sampling Algorithms}
\label{app:exp_sampling_alg}

We consider the four sampling algorithms mentioned in Section~\ref{sec:experiments}.
In particular, all four samplers receive the same $N$ i.i.d.\ candidates
\(
    Y_1, \ldots, Y_N \stackrel{\mathrm{iid}}{\sim} \widehat{\mu}(\cdot \mid x)
\)
and their unnormalized weights $w_u(x, Y_i)$.
Each sampler returns exactly one of the observed candidates.
If $N = 1$, all four methods return the sole candidate and hence coincide with $\widehat{\mu}$.
We detail the two RS baselines below; UR and SIR are defined in Sections~\ref{sec:intro} and~\ref{sec:sir}

\paragraph{Envelope RS.}
This baseline implements the approximate rejection sampler of Section~\ref{sec:intro}.
For convenience, we parameterize its threshold directly on the unnormalized weight scale.
Given a threshold $\tau > 0$, candidates are examined in order and candidate $i$ is accepted with probability
\[
    \min\left\{\frac{w_u(x, Y_i)}{\tau}, 1\right\}.
\]
The first accepted candidate is returned; if all $N$ candidates are rejected, one of the observed candidates is returned uniformly at random.
Equivalently, $\tau = ZM$ for the normalized threshold $M$ used in the theoretical formulation.

For each $(x, \beta)$, we set the threshold using the fixed response pool $\mathcal{A}_x^\star$ and use the same threshold for every $N$.
Specifically, we set
\[
    \tau := \max_{y \in \supp(\widehat{\mu})} w_u(x,y).
\]
Hence, envelope RS is supplied with the exact envelope of the finite empirical target--proposal pair $(\pi_\beta, \widehat{\mu})$, and serves as an empirical proxy for RS with threshold $M = C_\infty^\pi$.

\paragraph{Budget-calibrated RS.}
This baseline is adapted from~\citet[Algorithm~5]{rohatgi2025} as described in Section~\ref{sec:rohatgi}.
In particular, given a sampling budget $N$, set
\(
    n = \left\lfloor\frac{N-1}{2}\right\rfloor.
\)
For $N \in \{1, 2\}$, for which $n = 0$, we define the baseline to return a single fresh candidate $Y_1 \sim \widehat{\mu}(\cdot \mid x)$.
For $N \geq 3$, the algorithm first uses $n$ candidates as pilot samples to estimate the normalizing constant,
\(
    \widehat{Z} := \tfrac{1}{n}\sum_{i=1}^n w_u(x, Y_i),
\)
and the threshold multiplier is set to
\(
    M_{N, \delta_{\rm cal}} := \tfrac{n}{4\log(4/\delta_{\rm cal})},
\)
$\delta_{\rm cal} \in (0,1)$.
The algorithm then runs standard RS on the next $n$ candidates with acceptance probability
\[
    \min\left\{\frac{w_u(x, Y_i)}{M_{N, \delta_{\rm cal}}\widehat{Z}}, 1\right\},
\]
and the fresh fallback candidate $Y_{2n+1}$ is returned if all are rejected.
Hence, at most $2n + 1 \leq N$ candidates are used.
In addition, we select $\delta_{\rm cal}$ from $\{0.001, 0.01, 0.05, 0.1, 0.25, 0.5, 0.9\}$ separately for each $(x, \beta, N)$.
For the TV plots, we choose the value minimizing the estimated TV error;
for the accuracy plots, we choose the value maximizing ground-truth accuracy.
Hence, we note that these results should be interpreted as favorably tuned benchmark comparisons, not as the performance of a deployable calibration rule that selects $\delta_{\rm cal}$ without access to the
evaluation metric. 
All candidate values of $\delta_{\rm cal}$ are evaluated on the same Monte Carlo runs.

\subsection{Evaluation Metrics and Estimation}

We report two complementary metrics:
(1) the TV error $\DTV(P_N, \pi_\beta)$ and
(2) ground-truth accuracy.

\paragraph{TV error.}
Let $P_N(\cdot \mid x)$ denote the marginal distribution of the response returned by a sampler with budget $N$.
Since all four algorithms are designed to approximate the same target distribution $\pi_\beta$, we measure sampling fidelity by $\DTV(P_N, \pi_\beta)$.
We note that this is also the quantity controlled directly by our theoretical results.

\paragraph{Ground-truth accuracy.}
Ground-truth accuracy is not the objective optimized by the samplers.
Rather, it measures whether the sampled response gives the correct final answer to the underlying math problem.
For a fixed prompt $x$, define the target accuracy
\[
    \operatorname{Acc}(\pi_\beta; x) := \sum_y \pi_\beta(y \mid x)\1\{y\text{ has the correct final answer}\}.
\]
Since correctness corresponds to an event,
\[
    \left|\operatorname{Acc}(P_N; x) - \operatorname{Acc}(\pi_\beta; x)\right|
    \leq \DTV(P_N, \pi_\beta).
\]
Hence, as the output distribution approaches $\pi_\beta$, its ground-truth accuracy approaches the corresponding target accuracy.
We report accuracy to assess whether improved sampling fidelity also appears in downstream task performance.
Importantly, the target accuracy may be either higher or lower than the reference policy accuracy, so accuracy need not increase monotonically with $N$.

\subsubsection{Estimating the TV Error on a Finite Pool}
\label{app:empirical_setup}

Computing the TV error requires knowledge of the target distribution $\pi_\beta$ and the output law $P_N$.
Note that the target distribution $\pi_\beta$ can be computed exactly on the finite empirical support $\supp(\widehat{\mu})$, 
but the output law $P_N$ of a sampling procedure should be estimated through repeated trials.

To estimate $P_N$, for each prompt $x$, temperature parameter $\beta$, and sampling budget $N$, each trial independently draws
\(
    Y_1, \ldots, Y_N \stackrel{\mathrm{iid}}{\sim}\widehat{\mu}(\cdot \mid x)
\)
and applies the corresponding sampling procedure.
Using the trial counts specified below, we repeat this experiment $m$ times and record the returned responses
\(
    \widehat{Y}^{(1)}, \ldots, \widehat{Y}^{(m)}.
\)
A fresh candidate set is drawn in every trial because $P_N$ is the marginal law of the returned response, averaging over both candidate generation and the sampler's internal randomness.
Reusing a single candidate set would instead estimate the output law conditional on that set.
With this, we estimate $P_N$ by its empirical frequencies,
\[
    \widehat{P}_N(y \mid x) := \frac{1}{m}\sum_{t=1}^m \1\{\widehat{Y}^{(t)}=y\},\quad y \in \supp(\widehat{\mu}),
\]
and report the plug-in TV estimate
\[
    \widehat{D}_{\mathrm{TV}}(P_N, \pi_\beta)
    := \frac{1}{2}\sum_{y \in \supp(\widehat{\mu})}\left|\widehat{P}_N(y \mid x) - \pi_\beta(y \mid x)\right|.
\]

\paragraph{Monte Carlo resolution scale.}
\label{app:monte_carlo_reference}

We highlight that even when a sampler is exact, the plug-in TV estimate above is generally positive because $\widehat P_N$ is computed from finitely many Monte Carlo
trials.
Consequently, once the true sampling error becomes sufficiently small, the observed TV error is limited by the statistical resolution of the estimator.

To quantify this effect without relying on an asymptotic approximation, we compute an exact sampler baseline.
For each setting, we draw $m$ independent samples directly from $\pi_\beta$, form their empirical distribution using the same estimator, and compute its TV distance from $\pi_\beta$.
This exact sampler experiment provides a Monte Carlo reference for the resolution scale induced by finite $m$.
We refer to this empirical reference level as the \emph{Monte Carlo floor} in the figures and discussion below, rather than a mathematical lower bound on the plug-in TV estimate.
Once a sampler reaches this level, differences of comparable magnitude cannot be reliably distinguished from Monte Carlo estimation error.

\subsection{Further experimental results}
\label{app:exp_results_analysis}

This section provides additional empirical results supporting
Section~\ref{sec:experiments}.
We first summarize the theoretical comparisons relevant to the experiments and then examine their empirical behavior across prompts and values of $\beta$.

\paragraph{Theoretical comparisons and empirical reference points.}
Since $(\pi_\beta, \widehat{\mu})$ forms a valid finite target--proposal pair, the theoretical results apply directly to the corresponding population output distributions:
\begin{itemize}[leftmargin=1em, nosep]
    \item \textbf{UR vs. envelope RS (Theorem~\ref{thm:main}):}
    UR attains the best fixed threshold RS guarantee in hindsight without requiring threshold information.
    We use RS supplied with the exact empirical envelope as an empirical proxy for RS with threshold $M = C_\infty^\pi$.
    We expect the two methods to exhibit comparable empirical performance.
    \item \textbf{UR vs. budget-calibrated RS (Theorem~\ref{thm:pilot-rs-fdom}):}
    for every prompt, $\beta$, sampling budget $N$, and $\delta_{\rm cal}$,
    \(
        \DTV(P_N^U, \pi_\beta) \leq \DTV(P_{N, \delta_{\rm cal}}^{R}, \pi_\beta).
    \)
    \item \textbf{UR vs. SIR (Theorem~\ref{thm:SIR}):}
    for every prompt, $\beta$, and sampling budget $N$,
    \(
        \DTV(P_N^U, \pi_\beta) \leq \DTV(P_N^{\mathrm{SIR}}, \pi_\beta).
    \)
    \item \textbf{Monte Carlo resolution:}
    Once the true TV error becomes sufficiently small, the estimated curves approach the scale of the exact sampler Monte Carlo baseline.
    Differences at this scale cannot be reliably distinguished from Monte Carlo estimation error.
\end{itemize}

\paragraph{Number of trials.}
All appendix results in this section use $m=10^5$ independent runs.
For Figure~\ref{fig:tv_error_and_acc}, we use $m=10^6$ to reduce Monte Carlo estimation error.

\paragraph{Confidence intervals for TV error.} 
Since $\widehat{D}_{\mathrm{TV}}$ is computed from the empirical output distribution $\widehat{P}_N$ based on $m$ independent runs, we estimate its Monte Carlo variability using a bootstrap.
For each $b = 1, \ldots, B$ with $B = 200$, we draw $m$ responses independently from $\widehat{P}_N$, then form the corresponding empirical distribution $\widehat{P}_N^{\star(b)}$, and compute
\[
    \widehat{D}_b^\star := \frac{1}{2}\sum_{y \in \supp(\widehat{\mu})}
    \left|\widehat{P}_N^{\star(b)}(y \mid x) - \pi_\beta(y \mid x)\right|.
\]
We take $\widehat{\mathrm{se}}$ to be the sample standard deviation of $\{\widehat{D}_b^\star\}_{b = 1}^B$ and report the approximate 95\% interval
\(
    \widehat{D}_{\mathrm{TV}}\pm z_{0.975}\widehat{\mathrm{se}}.
\)
The interval reflects Monte Carlo variability of the estimator.
The exact sampler separately shows the nonzero plug-in error induced by finite $m$ when the true TV error is zero.

\paragraph{Confidence intervals for accuracy.}
For accuracy, each of the $m$ independent runs yields a Bernoulli outcome.
We report
\(
    \widehat{a}\pm t_{m-1,0.975}\widehat{\mathrm{se}},
\)
and
\(
    \widehat{\mathrm{se}} = \tfrac{s}{\sqrt m},
\)
where $\widehat{a}$ and $s$ denote the sample mean and sample standard deviation of the outcomes, respectively.
At $m = 10^5$, $t_{m-1, 0.975} \approx 1.96$.
In particular, for budget-calibrated RS, we tune $\delta_{\rm cal}$ over the candidate grid using the same Monte Carlo runs used for evaluation for each $(x, \beta, N)$.
This choice favors the baseline in the reported point estimates.
For both metrics, the reported intervals quantify Monte Carlo variability while treating the selected $\delta_{\rm cal}$ as fixed and do not account for calibration selection effects.

\paragraph{GSM8K: variation across prompts.}
Figure~\ref{fig:gsm8k_tv_beta_0p1} reports the TV error at $\beta = 0.1$ for eight randomly selected GSM8K problems.
The difficulty of approximating the target varies substantially across prompts.
For some prompts, the methods reach the Monte Carlo floor within a small sampling budget, leaving little resolution for distinguishing them at larger $N$.
For others, substantial separation remains throughout the tested range.
Across the displayed prompts, UR achieves estimated TV error below or comparable to SIR and budget-calibrated RS, up to Monte Carlo resolution, consistent with Theorems~\ref{thm:pilot-rs-fdom} and~\ref{thm:SIR}.
UR remains competitive with envelope RS without requiring threshold information.

\paragraph{GSM8K: variation across $\beta$.}
Figure~\ref{fig:gsm8k_tv_question_201} fixes Problem~\#201 and varies $\beta \in \{0.01, 0.05, 0.1, 0.125, 0.25, 0.5, 1, 2\}$.
Changing $\beta$ changes the concentration of the reward-tilted target and therefore the finite budget difficulty of the sampling problem.
At smaller $\beta$, the target is more strongly tilted toward high-reward responses, whereas larger $\beta$ makes $\pi_\beta$ closer to the reference proposal $\widehat{\mu}$.
Accordingly, the sampling difficulty varies substantially with $\beta$.
The same qualitative comparison persists across the displayed values: UR remains competitive with other baselines without requiring threshold information.

\paragraph{GSM8K: Ground-truth accuracy.}
Figure~\ref{fig:gsm8k_acc} reports ground-truth accuracy across eight prompts
at $\beta = 0.1$ and, for Problem~\#258, across $\beta \in \{0.01, 0.05, 0.1, 0.125, 0.25, 0.5, 1, 2\}$.
Distributional closeness to $\pi_\beta$ translates directly into closeness in ground-truth accuracy, since correctness corresponds to an event.
However, the target accuracy may be either above or below the reference policy accuracy.
Accordingly, as a sampler approaches $\pi_\beta$, its accuracy may either increase or decrease toward the corresponding target accuracy.
For this reason, we use TV error as the primary metric for comparing the sampling procedures and accuracy as a complementary metric.

\paragraph{MATH500.}
Figures~\ref{fig:math500_tv} and~\ref{fig:math500_acc} report the MATH500 results across eight problems at $\beta = 0.1$ and across values of $\beta$ for selected problems.
Across the displayed settings, the qualitative behavior observed on GSM8K also appears on MATH500, indicating that the empirical comparisons carry over to these more challenging mathematical reasoning problems.

\begin{figure}[!htbp]
    \begin{subfigure}[b]{1\textwidth}
        \centering
        \includegraphics[width=1\linewidth]{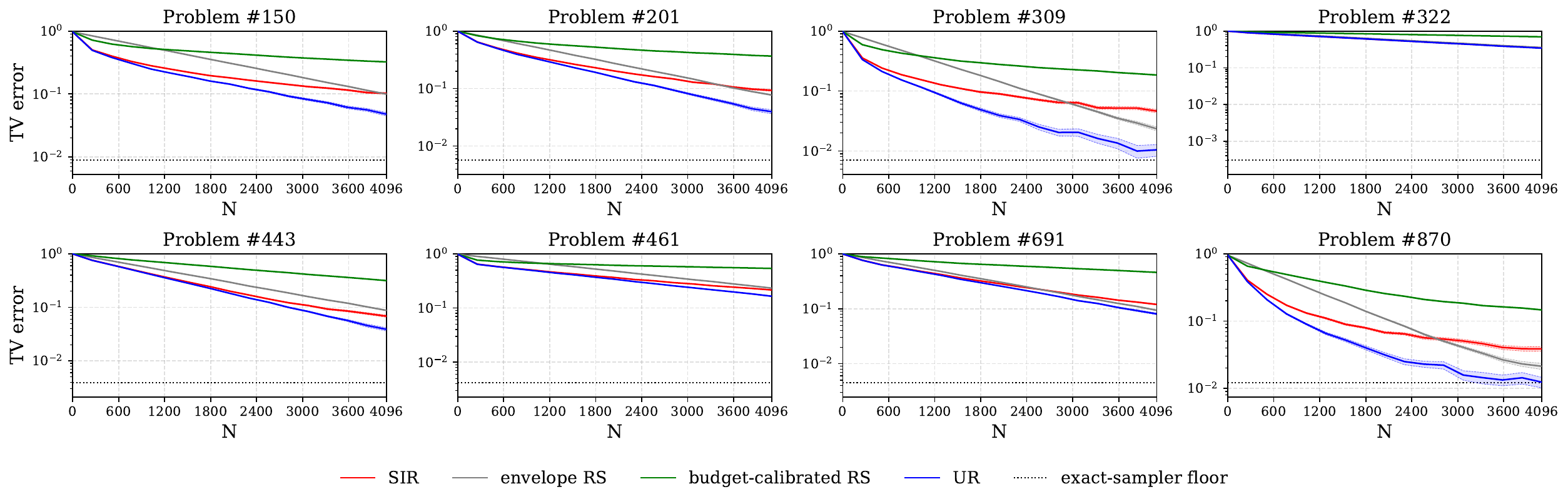}
        \caption{$\beta = 0.1$ for eight randomly selected problems.}
        \label{fig:gsm8k_tv_beta_0p1}
    \end{subfigure}
    \begin{subfigure}[b]{1\textwidth}
        \centering
        \includegraphics[width=1\linewidth]{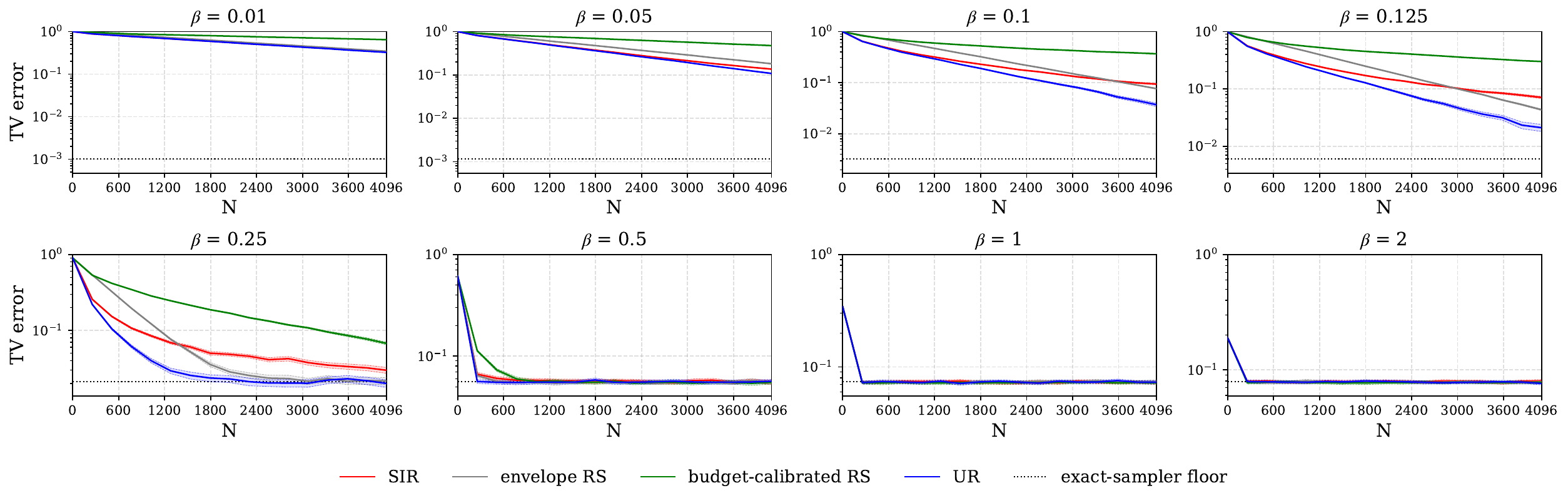}
        \caption{Problem~\#201 across eight values of $\beta$ from $0.01$ to $2$.}
        \label{fig:gsm8k_tv_question_201}
    \end{subfigure}
    \caption{TV error $\widehat{D}_{\mathrm{TV}}(P_N, \pi_\beta)$ w.r.t.\ sampling budget $N$ on GSM8K.
    The dotted line indicates the exact sampler Monte Carlo floor.}
    \label{fig:gsm8k_tv}
    \begin{subfigure}[b]{1\textwidth}
        \centering
        \includegraphics[width=1\linewidth]{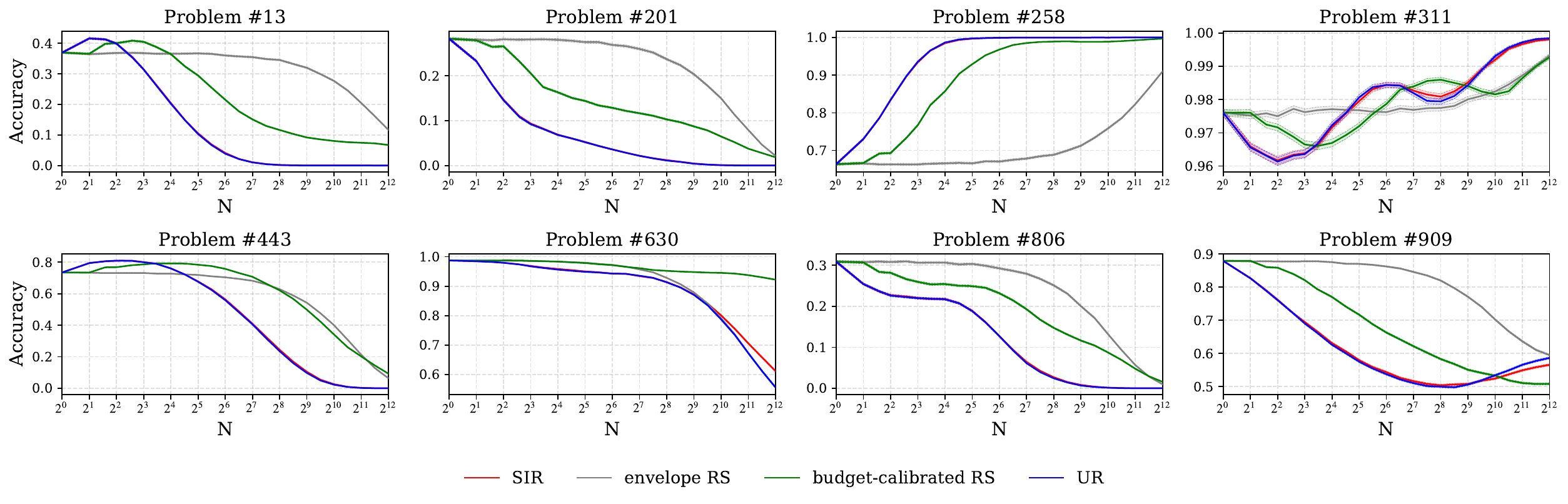}
        \caption{$\beta = 0.1$ for eight randomly selected problems.}
        \label{fig:gsm8k_acc_beta_0p1}
    \end{subfigure}
    \begin{subfigure}[b]{1\textwidth}
        \centering
        \includegraphics[width=1\linewidth]{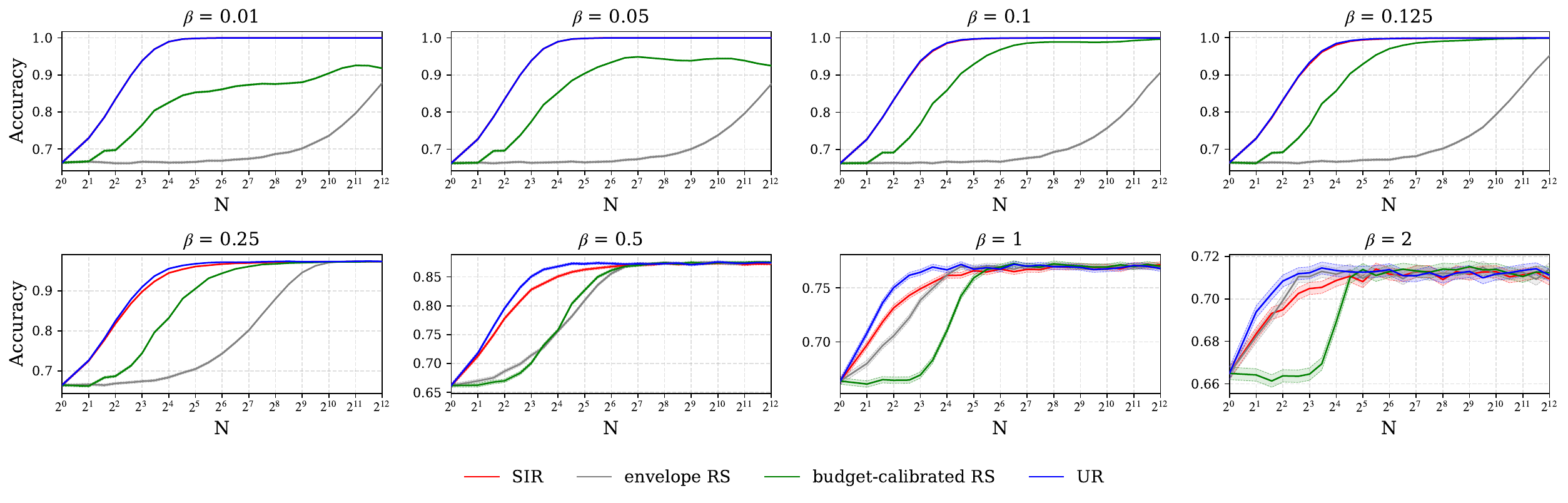}
        \caption{Problem~\#258 for $\beta$ from $0.01$ to $2$.}
        \label{fig:gsm8k_acc_question_258}
    \end{subfigure}
    \caption{Ground-truth accuracy w.r.t.\ sampling budget $N$ on GSM8K.}
    \label{fig:gsm8k_acc}
\end{figure}

\begin{figure}[!htbp]
    \begin{subfigure}[b]{1\textwidth}
        \centering
        \includegraphics[width=1\linewidth]{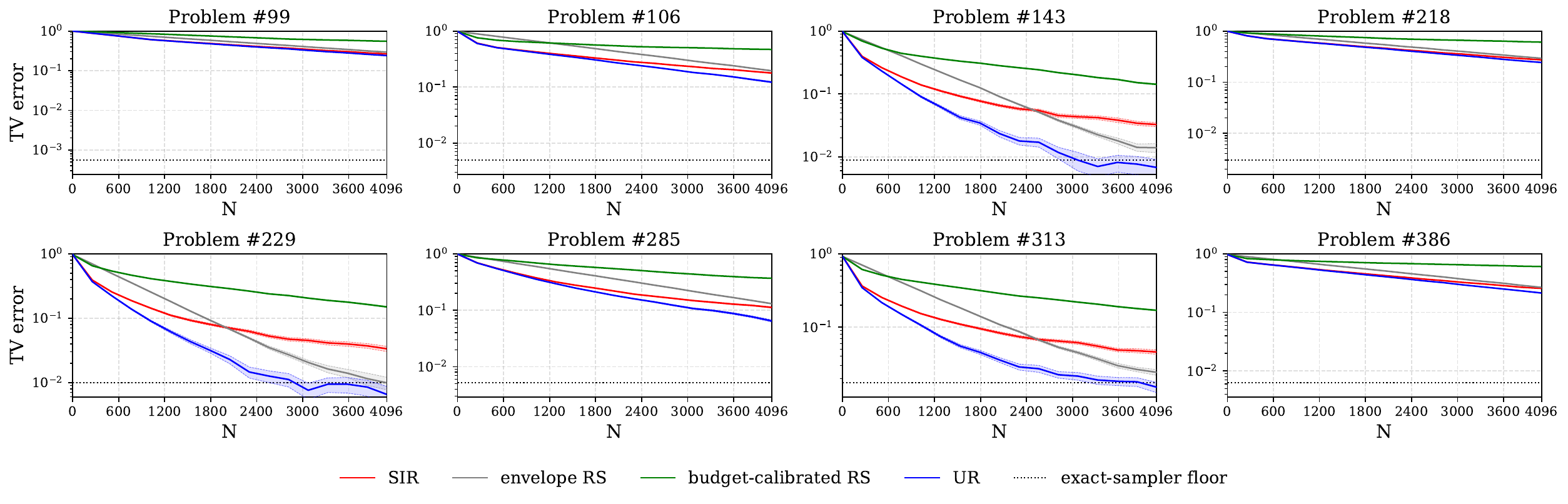}
        \caption{$\beta = 0.1$ for eight randomly selected problems.}
        \label{fig:math500_tv_beta_0p1}
    \end{subfigure}
    \begin{subfigure}[b]{1\textwidth}
        \centering
        \includegraphics[width=1\linewidth]{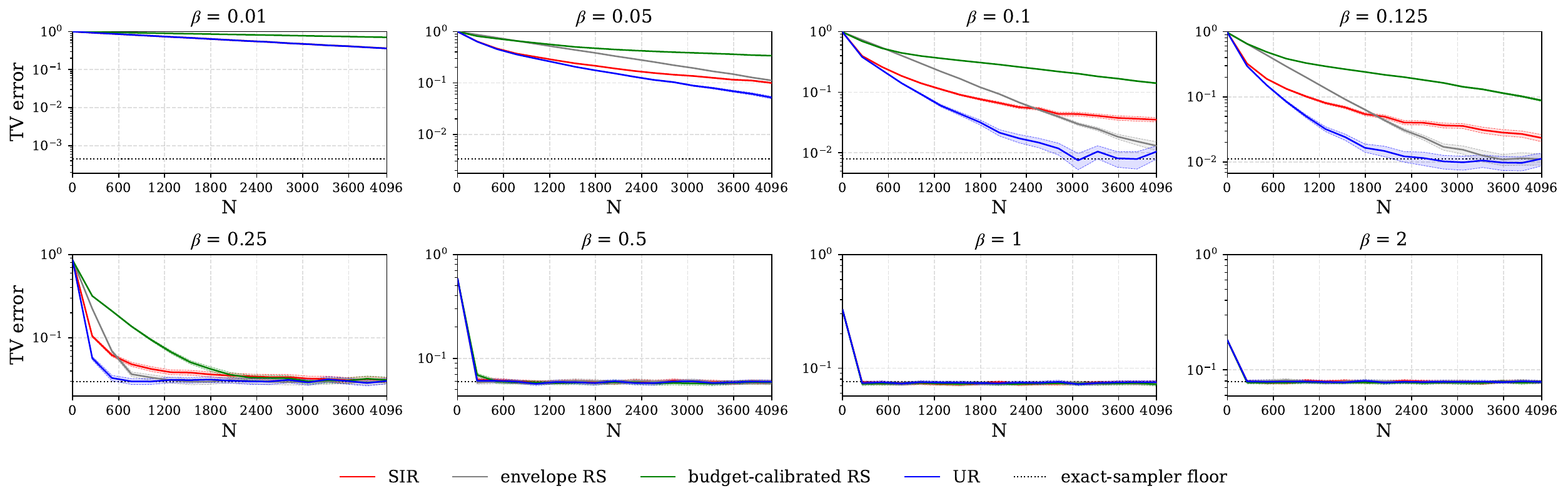}
        \caption{Problem~\#143 for $\beta$ from $0.01$ to $2$.}
        \label{fig:math500_tv_question_143}
    \end{subfigure}
    \caption{TV error $\widehat{D}_{\mathrm{TV}}(P_N, \pi_\beta)$ w.r.t.\ sampling budget $N$ on MATH500.
    The dotted line indicates the exact sampler Monte Carlo resolution scale.}
    \label{fig:math500_tv}
    \begin{subfigure}[b]{1\textwidth}
        \centering
        \includegraphics[width=1\linewidth]{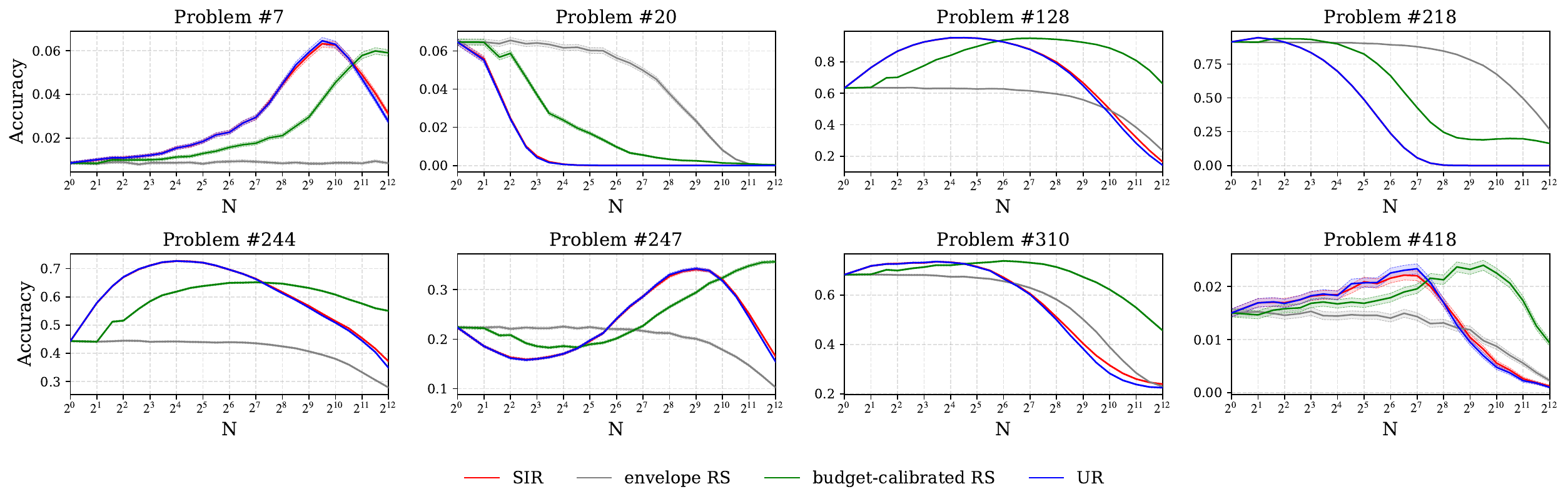}
        \caption{$\beta = 0.1$ for eight randomly selected problems.}
        \label{fig:math500_acc_beta_0p1}
    \end{subfigure}
    \begin{subfigure}[b]{1\textwidth}
        \centering
        \includegraphics[width=1\linewidth]{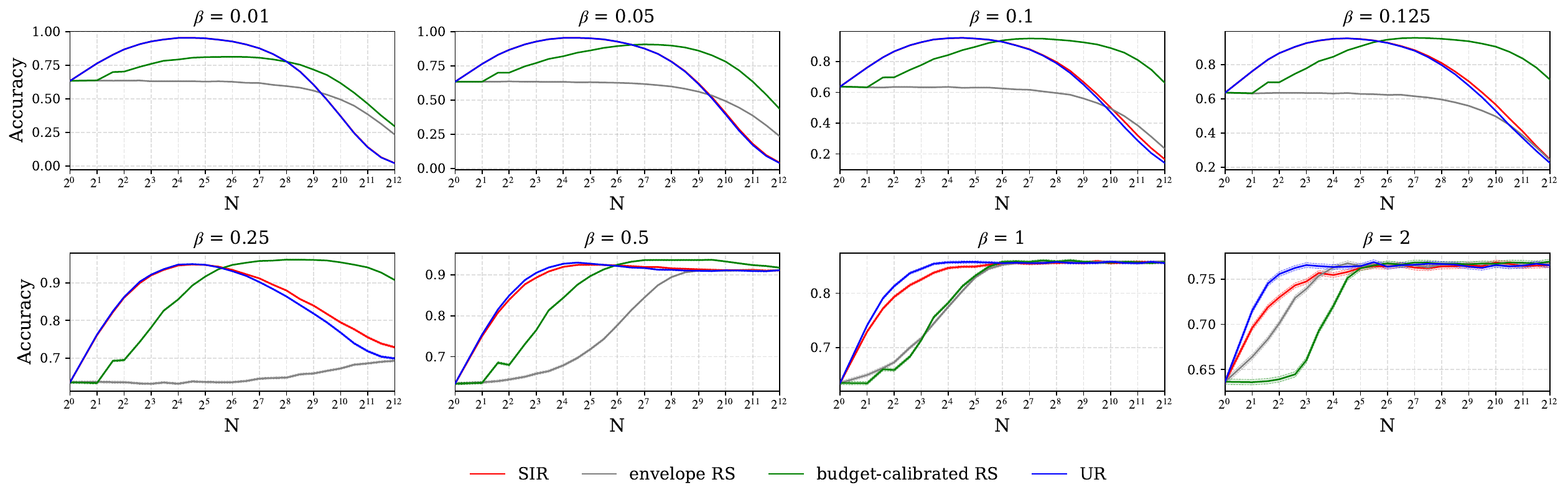}
        \caption{Problem~\#128 for $\beta$ from $0.01$ to $2$.}
        \label{fig:math500_acc_question_128}
    \end{subfigure}
    \caption{Accuracy w.r.t.\ sampling budget $N$ on MATH500.}
    \label{fig:math500_acc}
\end{figure}

\end{document}